\documentclass[11pt]{article}
\pdfoutput=1   %
\usepackage[letterpaper, margin=1in]{geometry}
\usepackage{parskip}
\usepackage[sort]{natbib}
\usepackage{times}

\PassOptionsToPackage{hyphens}{url}
\usepackage{hyperref}          %
\usepackage{url}
\usepackage{enumitem}
\newcommand{\PIG}{\textsc{pig}}
\newcommand{\PIGint}{\textsc{pig.int}}
\newcommand{\PIGext}{\textsc{pig.ext}}
\usepackage{graphicx}

\usepackage{bm}
\usepackage{multirow}
\usepackage{booktabs}
\usepackage{algorithm}
\usepackage{algpseudocode}
\usepackage{wrapfig}
\usepackage{axlib}

\usepackage{extarrows}

\usepackage{placeins}
\usepackage{float}          %
\newcommand{\detailfigscale}{0.9}   %
\newcommand{\codefootnote}{}
\newcommand{\hflink}{}
\newcommand{\tbteaserscale}{1.18}   %
\newsavebox{\squishbox}
\newlength{\squishlen}
\newcommand{\squishfig}[2]{%
  \sbox{\squishbox}{\resizebox{\textwidth}{!}{#1}}%
  \setlength{\squishlen}{\dimexpr\ht\squishbox+\dp\squishbox-0.5cm\relax}%
  \resizebox*{\textwidth}{\squishlen}{\begingroup\tikzset{every picture/.append style={yscale=#2}}#1\endgroup}}
\usepackage{comment}        %
\usepackage{titletoc}       %

\usepackage{authblk}

\makeatletter
\def\@maketitle{%
  \newpage\null\vskip 2.0em%
  \begin{center}%
    {\LARGE \@title \par}\vskip 2.2em%
    {\normalsize\lineskip .5em\begin{tabular}[t]{c}\@author\end{tabular}\par}%
  \end{center}\par\vskip 2.4em}
\makeatother

\usepackage{fancyhdr}
\renewenvironment{abstract}{%
  \begin{center}\normalsize\bfseries Abstract\end{center}\vspace{-0.3\baselineskip}
  \begin{list}{}{\setlength{\leftmargin}{0.3in}\setlength{\rightmargin}{\leftmargin}\setlength{\listparindent}{0pt}%
    \setlength{\itemindent}{0pt}\setlength{\parsep}{0pt}}\item\relax\small}
 {\end{list}\vspace{0.3\baselineskip}}
\definecolor{linkblue}{RGB}{23, 64, 140}
\definecolor{citeteal}{RGB}{15, 118, 110}
\hypersetup{colorlinks=true, linkcolor=linkblue, citecolor=citeteal, urlcolor=linkblue,
            pdftitle={PIG: Parallel Inconsistency Guidance for Accelerating Diffusion Language Models},
            pdfauthor={Quentin Fruytier, Anton Xue, Sanjay Shakkottai, Sujay Sanghavi}}

\title{\texorpdfstring{%
  $\vcenter{\hbox{\includegraphics[height=2.15em]{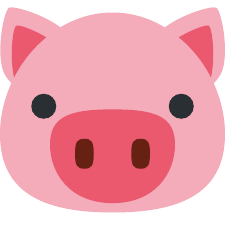}}}\mkern9mu
   \vcenter{\hbox{\bfseries PIG: Parallel Inconsistency Guidance for}\vskip2pt\hbox{\bfseries Accelerating Diffusion Language Models}}$%
  }{PIG: Parallel Inconsistency Guidance for Accelerating Diffusion Language Models}}

\author{Quentin Fruytier}
\author{Anton Xue}
\author{Sanjay Shakkottai}
\author{Sujay Sanghavi}
\affil{The University of Texas at Austin}
\affil{\footnotesize qdf76@my.utexas.edu, anton.xue@utexas.edu, sanjay.shakkottai@utexas.edu, sanghavi@mail.utexas.edu}
\date{}                          %

\begin{document}

\maketitle

\begin{abstract}

Diffusion language models accelerate generation by decoding multiple masked positions per forward pass. However, accuracy is known to drop when parallel decoding trajectories deviate from their sequential counterparts, which are slow and expensive ``oracles'' that adhere to the joint distribution of language. In this paper, we propose \textit{Parallel Inconsistency Guidance} (PIG), a new paradigm that uses a trained auxiliary model to directly estimate the inconsistency between parallel and sequential generation from a single model forward pass. During inference, masked positions are traversed by decreasing confidence, and tokens are added to the parallel decode set one-at-a-time until the total estimated inconsistency between the parallel and sequential generation exceeds a threshold. Surprisingly, it suffices to train a lightweight $\sim2$M parameter auxiliary model on a small mixture --- drawn from OpenCoder, GSM8K-train, MATH-train, and MBPP-train. For LLaDA-8B Instruct and NemotronDiffusion-3B Instruct, PIG extends the accuracy-TPF frontier on GSM8K, MATH500, MBPP, HumanEval, and MMLU, when compared to the industry standard Fast-dLLM and measures favorably on LLaDA-8B Instruct against DAPD, KLASS, EB-sampler, and other learned unmasking policies.

\end{abstract}

\noindent{\centering\footnotesize\textbf{Code:} \url{https://github.com/Quentin-Fruytier/PIG-DLM-release}\hflink\par}
\vspace{1.3\baselineskip}

\providecolor{tzcommit}{HTML}{0F766E}
\providecolor{tzover}{HTML}{B42318}
\newcommand{\tbeightpt}{\fontsize{8}{9.6}\selectfont}
\newcommand{\tbsevenpt}{\fontsize{7}{8.4}\selectfont}
\tikzset{
  tbtok/.style={font=\small, inner sep=0pt, execute at begin node=\strut},
  tbctx/.style={tbtok, text=black!55},
  tbslot/.style={draw=black!30, fill=black!7, rounded corners=2pt, minimum width=0.6cm, minimum height=0.36cm,
                 inner sep=0pt, font=\tbeightpt, text=black!50},
  tbconf/.style={font=\tbsevenpt, text=black!35, inner sep=0pt},
  tbmodel/.style={draw=tzcommit, fill=white, rounded corners=2pt, font=\tbeightpt\scshape, text=tzcommit,
                  inner xsep=4pt, inner ysep=2pt},
  tbout/.style={font=\small\bfseries, inner sep=1pt, execute at begin node=\strut},
  tbnote/.style={font=\tbsevenpt, inner sep=0pt},
  tbwide/.style={-{Triangle[length=7pt,width=11pt]}, line width=5pt, draw=tzcommit!22},
  tbarrow/.style={-{Stealth[length=4pt,width=3.2pt]}, line width=0.6pt},
  tbfaded/.style={draw=black!35, dash pattern=on 2pt off 1.5pt, rounded corners=4pt, line width=0.5pt, fill=black!3},
  tbmodelfaded/.style={draw=black!35, fill=white, rounded corners=2pt, font=\tbsevenpt\scshape, text=black!45,
                       inner xsep=3pt, inner ysep=1.5pt},
}

\newcommand{\tbsixpt}{\fontsize{6}{7.2}\selectfont}
\newcommand{\tbfivept}{\fontsize{5.5}{6.6}\selectfont}
\tikzset{
  tbpred/.style={tbslot, minimum width=0.62cm, inner xsep=2.5pt, text=black!60},
  tbstep/.style={draw=tzover!55, fill=tzover!4, rounded corners=4pt, line width=0.6pt},
  tbref/.style={font=\tbsixpt, text=black!50, inner sep=0pt},
  tbgroup/.style={fill=black!5, draw=black!15, rounded corners=3pt},
  tbouttok/.style={tbslot, inner xsep=2.5pt, font=\tbeightpt\bfseries},
  tbnfe/.style={draw=tzcommit, fill=white, rounded corners=2pt, font=\tbsixpt\scshape, text=tzcommit, inner xsep=2.5pt, inner ysep=1.5pt},
  tbfat/.style={-{Triangle[length=5pt,width=11pt]}, line width=7.5pt, draw=black!18},
  tbfatlbl/.style={font=\tbfivept\scshape, text=black!60, inner sep=0pt},
  tbdash/.style={-{Stealth[length=5pt,width=4.5pt]}, line width=1.1pt, draw=black!50, dash pattern=on 2.5pt off 1.8pt},
}
\newcommand{\tbteaserleft}{%
\begin{tikzpicture}[x=1cm,y=1cm]
  \def\cw{8.0}
  \node[tbctx] at (0.55,4.7) {Roses};
  \node[tbslot] at (1.4,4.7) {M1};
  \node[tbctx] at (2.2,4.7) {red,};
  \node[tbslot] at (2.95,4.7) {M2};
  \node[tbctx] at (3.85,4.7) {violets};
  \node[tbslot] at (4.75,4.7) {M3};
  \node[tbctx] at (5.5,4.7) {blue};
  \draw[rounded corners=5pt, draw=tzcommit, line width=1pt, fill=white] (0,1.75) rectangle (\cw,4.35);
  \node[fill=white, inner xsep=5pt, inner ysep=1.5pt, font=\normalsize\scshape, text=tzcommit, anchor=west] at (0.25,4.35) {pig};
  \draw[tbgroup] (0.12,2.15) rectangle (0.88,3.95);
  \draw[tbgroup] (1.8,2.15) rectangle (2.6,3.95);
  \draw[tbgroup] (7.0,2.15) rectangle (7.9,3.95);
  \node[tbslot] (m1) at (0.5,3.65) {M1};
  \node[tbslot] (m2) at (0.5,3.05) {M2};
  \node[tbslot] (m3) at (0.5,2.45) {M3};
  \draw[-{Triangle[length=7pt,width=22pt]}, line width=13pt, draw=tzcommit!18] (0.95,3.05) -- (1.72,3.05);
  \node[tbnfe] at (1.3,3.05) {nfe};
  \node[tbpred] (p1) at (2.2,3.65) {are};
  \node[tbpred] (p2) at (2.2,3.05) {and};
  \node[tbpred] (p3) at (2.2,2.45) {look};
  \draw[tbstep] (2.8,2.05) rectangle (6.75,4.1);
  \node[tbnote, text=tzover, font=\tbeightpt\scshape\bfseries] at (4.775,3.86) {parallel inconsistency};
  \node[tbnote, text=tzover, font=\tbeightpt] at (4.775,3.55) {$\hat{P}\,(\textsc{parallel}\neq\textsc{oracle})$};
  \node[tbnote, text=tzover, font=\tbeightpt] at (4.775,3.24) {(too high for M3)};
  \draw[-{Triangle[length=9pt,width=19pt]}, line width=12pt, draw=tzcommit!28] (2.95,2.55) -- (6.6,2.55);
  \node[font=\tbeightpt\scshape\bfseries, text=tzcommit, inner sep=0pt] at (4.7,2.55) {pig adapter};
  \node[tbouttok, text=tzcommit] at (7.45,3.65) {are};
  \node[tbouttok, text=tzcommit] at (7.45,3.05) {and};
  \node[tbouttok, text=tzover] (look) at (7.45,2.45) {look};
  \draw[tzover, line width=0.8pt] (7.05,2.2) -- (7.85,2.7);
  \begin{scope}[opacity=0.65]
    \draw[tbfaded] (0.0,0.3) rectangle (8.0,1.3);
    \draw[tbdash] (m3.south) -- (0.5,1.35);
    \node[tbslot] (s1) at (0.6,0.72) {M1};
    \draw[tbfat] (0.9,0.72) -- (1.8,0.72);
    \node[tbfatlbl] at (1.33,0.72) {nfe};
    \node[tbpred, text=black!45] (a1) at (2.1,0.72) {are};
    \node[tbslot] (s2) at (3.25,0.72) {M2};
    \draw[tbarrow, draw=black!45] (a1.east) -- (s2.west);
    \draw[tbfat] (3.55,0.72) -- (4.45,0.72);
    \node[tbfatlbl] at (3.98,0.72) {nfe};
    \node[tbpred, text=black!45] (a2) at (4.75,0.72) {and};
    \node[tbslot] (s3) at (5.9,0.72) {M3};
    \draw[tbarrow, draw=black!45] (a2.east) -- (s3.west);
    \draw[tbfat] (6.2,0.72) -- (7.1,0.72);
    \node[tbfatlbl] at (6.63,0.72) {nfe};
    \node[tbpred, text=black!45] (a3) at (7.4,0.72) {are};
    \draw[tbdash] (7.4,1.3) -- (7.4,2.22);
  \end{scope}
  \node[fill=white, inner xsep=5pt, inner ysep=1.5pt, font=\normalsize\scshape, text=black!60] at (4.0,1.3) {unseen sequential oracle};
\end{tikzpicture}}

\providecommand{\tbteaserscale}{1}
\newlength{\tbpanelh}\setlength{\tbpanelh}{\tbteaserscale\dimexpr122.8pt\relax}   %
\begin{figure}[H]
  \noindent\resizebox{!}{\tbpanelh}{\tbteaserleft}\hfill\includegraphics[height=\tbpanelh]{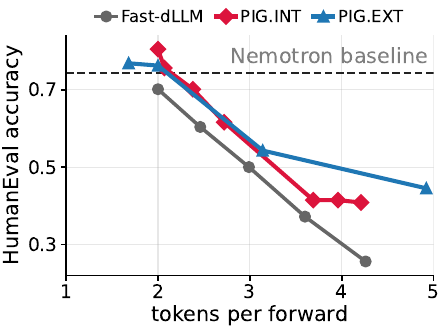}
  \caption{\textbf{Parallel Inconsistency Guidance (\PIG{}).} \emph{Left:} from one forward pass, \PIG{}
  estimates how likely each parallel commit is to disagree with the unseen sequential decode and defers the
  risky ones (M3). \emph{Right:} on Nemotron-3B HumanEval (64-token blocks), \PIGint{} and \PIGext{} extend
  the accuracy--TPF frontier of the model's own confidence decoder and baseline reported result.}
  \label{fig:teaser}
\end{figure}

\tikzset{
  pwtitle/.style={font=\tbeightpt\bfseries, inner sep=0pt},
  pwsub/.style={font=\tbsevenpt, text=black!55, inner sep=0pt},
  pwtxt/.style={font=\tbeightpt, inner sep=0pt},
  pwbox/.style={draw=black!30, fill=black!7, rounded corners=2pt, minimum width=0.62cm, minimum height=0.4cm,
                inner xsep=2pt, inner ysep=0pt, font=\tbeightpt, text=black!55},
  pwcommit/.style={pwbox, fill=tzcommit, draw=tzcommit, text=white, font=\tbeightpt\bfseries},
  pwreject/.style={pwbox, fill=white, draw=tzover, text=tzover, font=\tbeightpt\bfseries},
  pwframe/.style={draw=black!35, fill=black!4, line width=0.4pt},
  pwok/.style={font=\tbsevenpt\bfseries, text=tzcommit, inner sep=0pt},
  pwbad/.style={font=\tbsevenpt\bfseries, text=tzover, inner sep=0pt},
  pwdiv/.style={draw=black!22, line width=0.4pt},
}
\newcommand{\tbpigintwalk}{%
\begin{tikzpicture}[x=1cm,y=1cm]
  \draw[rounded corners=5pt, draw=tzover, line width=1pt, fill=white] (1.75,1.65) rectangle (12.95,6.2);
  \node[fill=white, inner xsep=5pt, inner ysep=1.5pt, font=\normalsize\scshape, text=tzover, anchor=west] at (2.0,6.2) {pig.int};
  \foreach \y/\w in {5.6/Roses, 4.36/{red,}, 3.74/and, 3.12/violets}{ \node[pwtxt] at (0.55,\y) {\w}; }
  \foreach \y/\w in {4.98/M1, 2.5/M2, 1.88/M3}{ \node[pwbox] at (0.55,\y) {\w}; }
  \node[pwsub] at (0.55,6.05) {masked};
  \draw[-{Triangle[length=8pt,width=24pt]}, line width=14pt, draw=tzcommit!22] (1.06,3.74) -- (1.72,3.74);
  \node[tbnfe] at (1.39,3.74) {nfe};
  \draw[pwdiv] (5.05,1.9) -- (5.05,5.55);
  \draw[pwdiv] (8.65,1.9) -- (8.65,5.55);
  \draw[black!30, dash pattern=on 2pt off 1.5pt, line width=0.4pt] (5.15,3.98) -- (12.85,3.98);
  \draw[black!30, dash pattern=on 2pt off 1.5pt, line width=0.4pt] (5.15,2.87) -- (12.85,2.87);
  \node[pwtitle] at (3.45,5.8)  {I.\ Traversal Order};
  \node[pwsub]   at (3.45,5.5)  {decreasing confidence order};
  \node[pwtitle] at (6.85,5.8)  {II.\ Set Selection};
  \node[pwsub]   at (6.85,5.5)  {$b_i=\hat p_\nu\big(\nu_i,\ \hat\mu_i(S)\big)$};
  \node[pwtitle] at (10.8,5.8)  {III.\ Thresholding};
  \node[pwsub]   at (10.8,5.5)  {admit while $\hat R(S)\le\varphi$};
  \draw[pwframe] (2.12,1.82) rectangle (2.38,4.8);
  \fill[black!18] (2.12,1.82) rectangle (2.38,4.8);
  \draw[tbdash] (1.96,4.8) -- (1.96,1.82);   %
  \node[pwsub] at (2.25,5.0) {$\nu$};
  \foreach \y/\slot/\k/\cv/\tok in {4.53/M1/1/0.92/are, 3.42/M3/3/0.71/blue, 2.32/M2/2/0.55/look}{
    \draw[black!60, line width=0.7pt] (2.06,\y) -- (2.44,\y);
    \node[pwbox] at (2.95,\y) {\slot};
    \node[pwtxt, anchor=west] at (3.4,\y+0.12) {$\nu_{\k}=\cv$};
    \node[pwsub, anchor=west] at (3.4,\y-0.16) {$\hat x_{\k}=$ \tok};
  }
  \node[pwtxt, anchor=west] (m1) at (5.25,4.53) {$S=\{$};
  \node[pwcommit, anchor=west] (m1a) at ([xshift=2pt]m1.east) {are};
  \node[pwtxt, anchor=west] at ([xshift=2pt]m1a.east) {$\}$};
  \node[pwtxt, anchor=west] (m2) at (5.25,3.42) {$S=\{$};
  \node[pwcommit, anchor=west] (m2a) at ([xshift=2pt]m2.east) {are};
  \node[pwcommit, anchor=west] (m2b) at ([xshift=3pt]m2a.east) {blue};
  \node[pwtxt, anchor=west] at ([xshift=2pt]m2b.east) {$\}$};
  \node[pwtxt, anchor=west] (m3) at (5.25,2.32) {$S=\{$};
  \node[pwcommit, anchor=west] (m3a) at ([xshift=2pt]m3.east) {are};
  \node[pwcommit, anchor=west] (m3b) at ([xshift=3pt]m3a.east) {blue};
  \node[pwreject, anchor=west] (m3c) at ([xshift=3pt]m3b.east) {look};
  \node[pwtxt, anchor=west] at ([xshift=2pt]m3c.east) {$\}$};
  \draw[tzover, line width=0.8pt] (m3c.south west) -- (m3c.north east);
  \foreach \y in {4.53, 3.42, 2.32}{
    \draw[pwframe] (9.3,\y-0.12) rectangle (11.3,\y+0.12);
  }
  \fill[tzcommit]    (9.3,3.42-0.12) rectangle (9.8,3.42+0.12);
  \fill[tzcommit!45] (9.3,2.32-0.12)  rectangle (9.8,2.32+0.12);
  \fill[tzover]      (9.8,2.32-0.12)  rectangle (11.09,2.32+0.12);
  \foreach \y in {4.53, 3.42, 2.32}{
    \draw[black, line width=0.7pt] (10.55,\y-0.2) -- (10.55,\y+0.2);
  }
  \node[pwtxt] at (10.55,4.89) {$\varphi$};
  \node[text=tzcommit, inner sep=0pt] at (9.0,4.53) {\ding{51}};
  \node[text=tzcommit, inner sep=0pt] at (9.0,3.42) {\ding{51}};
  \node[text=tzover,   inner sep=0pt] at (9.0,2.32) {\ding{55}};
  \node[pwok,  anchor=west] at (11.45,4.53)  {$0.00\le\varphi$};
  \node[pwok,  anchor=west] at (11.45,3.42) {$0.12\le\varphi$};
  \node[pwbad, anchor=west] at (11.45,2.32)  {$0.43>\varphi$};
  \draw[-{Triangle[length=8pt,width=24pt]}, line width=14pt, draw=black!15] (13.0,3.74) -- (13.62,3.74);
  \foreach \y/\w in {5.6/Roses, 4.36/{red,}, 3.74/and, 3.12/violets}{ \node[pwtxt] at (14.1,\y) {\w}; }
  \node[pwcommit] at (14.1,4.98) {are};
  \node[pwbox]    at (14.1,2.5)  {M2};
  \node[pwcommit] at (14.1,1.88) {blue};
  \node[pwsub] at (14.1,6.05) {committed};
\end{tikzpicture}}

\providecolor{tzext}{HTML}{1F77B4}   %
\tikzset{
  pwnext/.style={draw=tzext, dash pattern=on 2pt off 1.2pt, line width=0.7pt, fill=white, rounded corners=2pt,
                 minimum width=0.62cm, minimum height=0.4cm, inner xsep=2pt, inner ysep=0pt, font=\tbeightpt\bfseries, text=tzext},
  pwsmall/.style={draw=tzcommit!55, fill=tzcommit!55, rounded corners=1.5pt, minimum width=0.46cm, minimum height=0.3cm,
                  inner xsep=1.5pt, inner ysep=0pt, font=\tbsevenpt\bfseries, text=white},
  pwsmalltxt/.style={font=\tbsevenpt, text=black!45, inner sep=0pt},
}
\newcommand{\tbpigextwalk}{%
\begin{tikzpicture}[x=1cm,y=1cm]
  \draw[rounded corners=5pt, draw=tzext, line width=1pt, fill=white] (1.75,1.65) rectangle (12.95,6.2);
  \node[fill=white, inner xsep=5pt, inner ysep=1.5pt, font=\normalsize\scshape, text=tzext, anchor=west] at (2.0,6.2) {pig.ext};
  \foreach \y/\w in {5.6/Roses, 4.36/{red,}, 3.74/and, 3.12/violets}{ \node[pwtxt] at (0.55,\y) {\w}; }
  \foreach \y/\w in {4.98/M1, 2.5/M2, 1.88/M3}{ \node[pwbox] at (0.55,\y) {\w}; }
  \node[pwsub] at (0.55,6.05) {masked};
  \draw[-{Triangle[length=8pt,width=24pt]}, line width=14pt, draw=tzcommit!22] (1.06,3.74) -- (1.72,3.74);
  \node[tbnfe] at (1.39,3.74) {nfe};
  \draw[pwdiv] (5.05,1.9) -- (5.05,5.55);
  \draw[pwdiv] (8.65,1.9) -- (8.65,5.55);
  \node[pwtitle] at (3.45,5.8)  {I.\ Traversal Order};
  \node[pwsub]   at (3.45,5.5)  {decreasing confidence order};
  \node[pwtitle] at (6.85,5.8)  {II.\ Set Selection};
  \node[pwsub]   at (6.85,5.5)  {$N(S)$ = the next candidate};
  \node[pwsub]   at (6.85,5.22) {$c_i=\Delta\hat p\big(H_i,\ \widehat H_i^{\,\mathrm{next}}\big)$};
  \node[pwtitle] at (10.8,5.8)  {III.\ Thresholding};
  \node[pwsub]   at (10.8,5.5)  {admit token, continue if};
  \node[pwsub]   at (10.8,5.22) {$\hat R\big(N(S)\big)-\hat R\big(N(S)\mid S\big)\le\theta$};
  \draw[pwframe] (2.12,1.82) rectangle (2.38,4.8);
  \fill[black!18] (2.12,1.82) rectangle (2.38,4.8);
  \draw[tbdash] (1.96,4.8) -- (1.96,1.82);   %
  \node[pwsub] at (2.25,5.0) {$\nu$};
  \foreach \y/\slot/\k/\cv/\tok in {4.33/M1/1/0.92/are, 3.08/M3/3/0.71/blue, 2.12/M2/2/0.55/look}{
    \draw[black!60, line width=0.7pt] (2.06,\y) -- (2.44,\y);
    \node[pwbox] at (2.95,\y) {\slot};
    \node[pwtxt, anchor=west] at (3.4,\y+0.12) {$\nu_{\k}=\cv$};
    \node[pwsub, anchor=west] at (3.4,\y-0.16) {$\hat x_{\k}=$ \tok};
  }
  \node[pwsmalltxt, anchor=west] (s1) at (5.25,4.62) {$S=\{$};
  \node[pwsmall, anchor=west] (s1a) at ([xshift=2pt]s1.east) {are};
  \node[pwsmalltxt, anchor=west] at ([xshift=2pt]s1a.east) {$\}$};
  \node[pwtxt, anchor=west] (n1) at (5.25,4.08) {$N(S)=\{$};
  \node[pwnext, anchor=west] (n1a) at ([xshift=2pt]n1.east) {M3};
  \node[pwtxt, anchor=west] (n1z) at ([xshift=2pt]n1a.east) {$\}$};
  \node[pwsmalltxt, anchor=west] (s2) at (5.25,3.37) {$S=\{$};
  \node[pwsmall, anchor=west] (s2a) at ([xshift=2pt]s2.east) {are};
  \node[pwsmall, anchor=west] (s2b) at ([xshift=3pt]s2a.east) {blue};
  \node[pwsmalltxt, anchor=west] at ([xshift=2pt]s2b.east) {$\}$};
  \node[pwtxt, anchor=west] (n2) at (5.25,2.83) {$N(S)=\{$};
  \node[pwnext, anchor=west] (n2a) at ([xshift=2pt]n2.east) {M2};
  \node[pwtxt, anchor=west] (n2z) at ([xshift=2pt]n2a.east) {$\}$};
  \node[pwtxt, text=black!60] at (9.0,2.12) {M2 is not reached};
  \draw[black!30, dash pattern=on 2pt off 1.5pt, line width=0.4pt] (5.15,3.7) -- (12.85,3.7);
  \draw[black!30, dash pattern=on 2pt off 1.5pt, line width=0.4pt] (5.15,2.45) -- (12.85,2.45);
  \foreach \y in {4.33, 3.08}{
    \draw[pwframe] (9.3,\y-0.12) rectangle (11.3,\y+0.12);
    \draw[black, line width=0.7pt] (9.8,\y-0.2) -- (9.8,\y+0.2);
  }
  \fill[tzcommit] (9.3,4.33-0.12) rectangle (9.4,4.33+0.12);
  \fill[tzover]   (9.3,3.08-0.12) rectangle (10.55,3.08+0.12);
  \node[pwtxt] at (9.8,4.69) {$\theta$};
  \node[text=tzcommit, inner sep=0pt] at (9.0,4.33) {\ding{51}};
  \node[text=tzover,   inner sep=0pt] at (9.0,3.08) {\ding{55}};
  \node[pwok,  anchor=west] at (11.45,4.33) {$0.02\le\theta$};
  \node[pwbad, anchor=west] at (11.45,3.08) {$0.25>\theta$};
  \draw[-{Triangle[length=8pt,width=24pt]}, line width=14pt, draw=black!15] (13.0,3.74) -- (13.62,3.74);
  \foreach \y/\w in {5.6/Roses, 4.36/{red,}, 3.74/and, 3.12/violets}{ \node[pwtxt] at (14.1,\y) {\w}; }
  \node[pwcommit] at (14.1,4.98) {are};
  \node[pwbox]    at (14.1,2.5)  {M2};
  \node[pwcommit] at (14.1,1.88) {blue};
  \node[pwsub] at (14.1,6.05) {committed};
\end{tikzpicture}}

\providecolor{tzcommit}{HTML}{0F766E}
\providecolor{tzover}{HTML}{B42318}
\newcommand{\tfeightpt}{\fontsize{8}{9.6}\selectfont}
\newcommand{\tfsixpt}{\fontsize{6}{7.2}\selectfont}
\newcommand{\tffivept}{\fontsize{5.5}{6.6}\selectfont}
\tikzset{
  tftok/.style={font=\small, inner sep=0pt, execute at begin node=\strut},
  tfctx/.style={tftok, text=black!65},
  tfnew/.style={tftok, font=\small\bfseries, text=tzcommit},
  tfslot/.style={draw=black!30, fill=black!7, rounded corners=2pt, minimum width=0.55cm, minimum height=0.36cm,
                 inner sep=0pt, font=\tfeightpt, text=black!50},
  tfslotm/.style={tfslot, minimum width=0.5cm},
  tfblock/.style={draw=black!40, fill=black!4, rounded corners=4pt, line width=0.7pt},
  tfblocklbl/.style={font=\small, text=black!75, inner sep=0pt, execute at begin node=\strut},
  tfflow/.style={-{Stealth[length=4pt,width=3.4pt]}, line width=0.7pt, draw=black!45},
  tfwire/.style={-{Stealth[length=5pt,width=4.2pt]}, line width=0.7pt, draw=tzcommit},
  tfgroup/.style={draw=tzcommit, rounded corners=5pt, line width=0.8pt, fill=white},
  tfblk/.style={draw=tzcommit!65, fill=tzcommit!6, rounded corners=2pt, line width=0.6pt, minimum height=0.55cm,
                inner xsep=5pt, inner ysep=2pt, font=\tfeightpt, text=black!75, align=center},
  tffunnel/.style={draw=tzcommit!65, fill=tzcommit!6, line width=0.6pt},
  tfplus/.style={draw=tzcommit!70, fill=white, circle, inner sep=0pt, minimum size=0.3cm, font=\tfeightpt, text=tzcommit, line width=0.6pt},
  tfarchflow/.style={-{Stealth[length=3pt,width=2.6pt]}, line width=0.5pt, draw=tzcommit!70},
  tfaxis/.style={font=\tfeightpt, text=black!45, inner sep=0pt},
  tfdiag/.style={font=\small, text=black!35, inner sep=0pt},
}
\newif\iftfwhite
\newcommand{\tfcell}[3]{%
  \ifdim #3pt=0pt \node[tfdiag] at (#1,#2) {$\cdot$};
  \else
    \pgfmathsetmacro{\tfabs}{abs(#3)}\pgfmathsetmacro{\tfal}{min(100,210*\tfabs)}
    \ifdim #3pt>0pt \def\tfcol{tzcommit}\else\def\tfcol{tzover}\fi
    \fill[\tfcol!\tfal, rounded corners=1pt] ({#1-0.33},{#2-0.2}) rectangle ({#1+0.33},{#2+0.2});
    \pgfmathparse{\tfal>50}\ifnum\pgfmathresult=1 \def\tftxt{white}\else\def\tftxt{black!80}\fi
    \node[font=\tfeightpt, text=\tftxt, inner sep=0pt] at (#1,#2) {\tfnum{#3}};
  \fi}
\newcommand{\tfnum}[1]{\ifdim #1pt>0pt $\tfstrip\tfabs$\else $-\tfstrip\tfabs$\fi}
\def\tfstrip#1{\expandafter\tfstripz#1\relax}
\def\tfstripz#1#2\relax{\ifx0#1#2\else#1#2\fi}
\newcommand{\tbadapterfig}{%
\begin{tikzpicture}[x=1cm,y=1cm]
  \node[tfctx] at (0.45,4.5) {Roses};   \node[tfslot] at (1.2,4.5) {M1};  \node[tfctx] at (1.85,4.5) {red,};
  \node[tfctx] at (2.45,4.5) {and};     \node[tfctx] at (3.25,4.5) {violets};
  \node[tfslot] at (4.05,4.5) {M2};     \node[tfslot] at (4.75,4.5) {M3};
  \draw[tfblock] (0.15,2.7) rectangle (5.1,3.9);
  \node[tfblocklbl] at (2.6,3.3) {Frozen Transformer};
  \foreach \x in {1.2, 4.05, 4.75}{
    \draw[tfflow] (\x,4.28) -- (\x,3.9);
    \draw[tfflow] (\x,2.7) -- (\x,2.3);
  }
  \node[tfctx] at (0.45,2.05) {Roses};  \node[tfnew] at (1.2,2.05) {are};   \node[tfctx] at (1.85,2.05) {red,};
  \node[tfctx] at (2.45,2.05) {and};    \node[tfctx] at (3.25,2.05) {violets};
  \node[tfnew] at (4.05,2.05) {are};    \node[tfnew] at (4.75,2.05) {blue};
  \draw[tfwire] (5.1,3.3) -- (5.6,3.3);
  \draw[tfgroup] (5.6,1.85) rectangle (10.4,4.6);
  \node[fill=white, inner xsep=4pt, inner ysep=1pt, font=\tfeightpt\scshape, text=tzcommit, anchor=west] at (5.85,4.6) {learned adapter};
  \draw[tffunnel] (5.85,3.85) -- (5.85,2.75) -- (6.75,3.0) -- (6.75,3.6) -- cycle;
  \node[font=\tffivept, text=black!75, align=center] at (6.25,3.3) {down\\proj.};
  \draw[tfarchflow] (6.8,3.3) -- (7.1,3.3);
  \node[tfblk, minimum width=1.2cm] (tr) at (7.85,3.3) {$2\times$\\transformer};
  \node[tfblk, minimum width=1.3cm, minimum height=0.5cm] (rc) at (9.2,4.05) {receiver $i$};
  \node[tfblk, minimum width=1.3cm, minimum height=0.5cm] (sd) at (9.2,2.55) {source $j$};
  \draw[tfarchflow] (tr.north) |- (rc.west);
  \draw[tfarchflow] (tr.south) |- (sd.west);
  \node[tfplus] (pl) at (10.05,3.3) {$\otimes$};
  \draw[tfarchflow] (rc.east) -| (pl.north);
  \draw[tfarchflow] (sd.east) -| (pl.south);
  \draw[tfwire] (pl.east) -- (10.75,3.3);
  \foreach \c/\g in {1/M1, 2/M2, 3/M3} {\node[tfslotm] at ({11.1+\c*0.72},3.85) {\g};}
  \foreach \r/\g in {3.4/M1, 2.95/M2, 2.5/M3} {\node[tfslotm] at (11.1,\r) {\g};}
  \node[tfaxis] at (11.1,3.85) {$i\,\backslash\,j$};
  \tfcell{11.82}{3.4}{0}     \tfcell{12.54}{3.4}{0.03}  \tfcell{13.26}{3.4}{-0.02}
  \tfcell{11.82}{2.95}{0.12} \tfcell{12.54}{2.95}{0}    \tfcell{13.26}{2.95}{0.41}
  \tfcell{11.82}{2.5}{0.06}  \tfcell{12.54}{2.5}{0.09}  \tfcell{13.26}{2.5}{0}
  \node[tfblocklbl, align=center] at (12.54,4.5) {Entropy Finite\\Differences};
  \node[font=\footnotesize, inner sep=0pt] at (12.54,2.0) {$\hat{\mathbf D}_{23}\approx H_2(\mathbf x_t)-H_2(\mathbf x_t\oplus\hat x_3)$};
  \draw[tfflow] (12.95,2.16) -- (13.26,2.72);
\end{tikzpicture}}

\clearpage
\providecommand{\codefootnote}{\footnote{Code and adapters: \url{https://anonymous.4open.science/r/PIGDLM} (anonymized).}}
\section{Introduction}
\label{sec:introduction}

Masked diffusion language models (DLMs) enable parallel text generation by predicting many masked positions at once~\citep{austin2021structured,shi2024simplified,sahoo2024simple}.
Because these predictions are available from a single model evaluation, DLM decoders can reveal tokens in flexible orders and commit several positions in parallel, reducing the number of forward passes needed for generation.
Recent large-scale implementations demonstrate that this parallelism can yield substantial speedups: Nemotron-Labs-Diffusion~\citep{fu2026nemotron} decodes nearly $6\times$ more tokens per forward pass than Qwen3-8B at comparable accuracy, while DiffusionGemma~\citep{team2026diffusiongemma} reports up to $4\times$ faster token generation with competitive benchmark performance.

However, which positions should be unmasked in parallel?
A single DLM forward pass predicts every masked position from the same current canvas, so positions committed together cannot condition on one another.
This lack of intermediate conditioning can make individually plausible predictions jointly inconsistent: if one position could be either ``dog'' or ``cat'' and another either ``barks'' or ``meows,'' a parallel commit can produce incoherent pairs such as ``dog meows'' or ``cat barks.''~\citep{gu2018nonautoregressive,zhang2026generation}
Sequential decoding avoids this problem by updating the context after every commit, but doing so sacrifices the parallelism that makes DLMs attractive in the first place.
Parallel decoders must therefore decide which positions can be safely committed together and which should wait for additional context~\citep{campbell2022continuous,chen2025optimal}.

Existing parallel decoders select positions using two main signals: \emph{confidence} and \emph{dependence}.
Confidence-based methods commit positions whose individual predictions are sufficiently certain; Fast-dLLM~\citep{wu2025fastdllm}, for example, commits positions above a confidence threshold.
Dependency-based methods instead avoid jointly committing positions that appear strongly coupled: DAPD~\citep{kim2026dapd} constructs an attention-derived dependency graph, while DEMASK~\citep{ringel2026demask} learns pairwise conditional influence.
More broadly, these methods use confidence, dependence, or related signals to identify sets of positions that can be decoded together without intermediate conditioning.
A common underlying theoretical objective is to select positions with high conditional independence~\citep{wu2025fastdllm,zhang2026generation,chen2025optimal}.

In this work, we instead target a complementary objective: preserving the one-token-at-a-time decoding trajectory while committing multiple positions in parallel.
For any candidate sequence of positions proposed for a parallel round, we ask whether the same model would produce the same tokens if those positions were instead decoded one at a time in that sequence.
We call any resulting disagreement \emph{sequential--parallel inconsistency}.
If this inconsistency could be predicted in advance, parallel decoding would amount to selecting as many positions as possible while keeping the disagreement low.

But how can the decoder know this inconsistency in advance?
Naively, it would have to replay the candidate sequence one token at a time, defeating the computational advantage of parallel decoding.
Yet this sequential replay turns out not to be necessary.
We find that a lightweight adapter to the frozen DLM can accurately predict how committing one candidate would change the model's predictions at other positions using only information from the current forward pass.
This gives rise to a simple family of parallel decoding algorithms that we call \emph{Parallel Inconsistency Guidance} (\PIG{}).

Empirically, \PIG{} extends the accuracy--TPF frontier on LLaDA-8B-Instruct and Nemotron-Labs-Diffusion-3B-Instruct across four math and code benchmarks and multiple choice questions (MMLU), with its largest gains at high parallelism.
Against the strong Fast-dLLM baseline~\citep{wu2025fastdllm}, \PIGint{} reaches up to $1.7\times$ higher tokens per forward (TPF) at matched accuracy on LLaDA and $1.4\times$ over Nemotron's native confidence decoder; at the fastest confidence settings, \PIG{} retains up to $28$ additional accuracy points.
\PIG{} also improves over dependency-guided decoders, reaching up to $4.4\times$ higher TPF, and compares favorably with learned unmasking policies~\citep{jazbec2025unmasking}.
Notably, these gains require only a ${\sim}2$M-parameter predictor that is trained once per backbone on a single data mixture and consumes less than $0.1\%$ of the total forward pass MACs.
In summary, our contributions are as follows.\codefootnote{}

\begin{itemize}[leftmargin=1.2em, itemsep=2pt, topsep=2pt]

\item \textbf{Parallel Inconsistency Guidance.}
We formulate parallel decoding around \emph{sequential--parallel inconsistency}: the disagreement introduced by committing a set of positions in parallel rather than decoding the same positions one at a time.
This yields a family of parallel commit-set selection rules: \PIGint{} controls inconsistency within the current parallel round, while \PIGext{} controls inconsistency that can be avoided by deferring positions to a later round.

\item \textbf{Efficient one-pass prediction.}
We show that sequential--parallel inconsistency can be predicted before any parallel commits are made using information already available in the current DLM forward pass.
A ${\sim}2$M-parameter adapter to the frozen backbone, learned on a single dataset mixture, estimates the finite-difference structure used for selection, reaching $r\approx0.8$ per edge and $r\approx0.9$ on the aggregate statistics used by the decoder on held-out benchmarks.

\item \textbf{Strong parallel decoding across models.}
Across LLaDA and Nemotron on GSM8K, MATH-500, MBPP, and HumanEval, and MMLU, \PIG{} improves the accuracy--TPF frontier over confidence-based decoding, with its largest gains at aggressive parallelism.
We additionally compare against dependency-guided decoders and learned unmasking policies, and show that the same approach transfers across two DLM architectures with different inference structures.

\end{itemize}

\section{Related Work}
\label{sec:related}

\textbf{Confidence-based parallel decoding.}
Confidence is the most common criterion for parallel unmasking, from confidence-ranked masked generation~\citep{ghazvininejad2019maskpredict,chang2022maskgit} to top-$k$ selection~\citep{nie2025large}, confidence thresholds~\citep{wu2025fastdllm,cai2026confidence}, and entropy criteria~\citep{benhamu2025ebsampler}.
The intuition is that confident predictions are less likely to change as context grows, so several tokens can often be committed together safely.
Fast-dLLM~\citep{wu2025fastdllm}, for example, commits all positions above a confidence threshold, as does Nemotron's native diffusion decoder~\citep{fu2026nemotron}.
Other methods use entropy bounds~\citep{benhamu2025ebsampler}, adaptive confidence rules~\citep{liang2026cd4lm}, token-level stability criteria~\citep{kim2025klass}, or learned unmasking policies~\citep{jazbec2025unmasking}.
\PIG{} instead evaluates a proposed parallel commit set by its predicted disagreement with sequential decoding.

\textbf{Dependency-aware parallel decoding.}
A complementary line of work selects commits using dependencies among masked positions, estimated from attention~\citep{deng2026agdo,zhou2026attnsampler,kim2026dapd,sahin2026adas}, clustering~\citep{qi2026clad}, or learned conditional influence~\citep{ringel2026demask}.
The intuition is that plausible predictions can still conflict when committed together if they depend strongly on one another.
AGDO~\citep{deng2026agdo} and DAPD~\citep{kim2026dapd} use attention structure to guide decode order or avoid interacting positions, while CLAD~\citep{qi2026clad} and ADAS~\citep{sahin2026adas} use cluster-level and adaptive criteria.
DEMASK~\citep{ringel2026demask} instead learns pairwise conditional influence and constrains cumulative dependence within a parallel set.
Recent theory similarly relates parallelization error to distributional dependence, including conditional total correlation and inference scheduling~\citep{campbell2022continuous,chen2025optimal,zhang2026generation}.
\PIG{} uses these interactions only insofar as they predict whether parallel decoding will disagree with its sequential counterpart.

\textbf{Blockwise and hybrid decoding.}
Another line of work changes the generation protocol through blockwise decoding~\citep{arriola2025bd3lm}, alternative generation orders~\citep{hong2026oemdm}, or hybrid autoregressive--diffusion decoding~\citep{fu2026nemotron}.
The intuition is that structured canvas states can retain diffusion parallelism while recovering autoregressive properties such as stable prefixes and efficient caching.
Block Diffusion~\citep{arriola2025bd3lm} generates blocks autoregressively while denoising positions within each block, while OeMDM~\citep{hong2026oemdm} studies alternative diffusion generation orders.
Nemotron~\citep{fu2026nemotron} supports autoregressive, diffusion, and self-speculative decoding in one model, using diffusion proposals with autoregressive verification.
\PIG{} is complementary to these decoding protocols: given the positions currently available for decoding, it decides which should be committed together.

\section{Background}
\label{sec:background}

\textbf{Masked diffusion language models.}
Let $\mathcal{V}$ be a vocabulary and let $\mathbf{m}\notin\mathcal{V}$ denote a special \emph{mask}.
A generation canvas is a partially masked token sequence $\mathbf{x}\in(\mathcal{V}\cup\{\mathbf{m}\})^L$ with masked positions $M (\mbf{x}) =\{i\in[L]:x_i=\mathbf{m}\}$.
In contrast to autoregressive models that predict only the next token, a masked diffusion language model (MDLM)~\citep{austin2021structured,shi2024simplified,sahoo2024simple} produces a vocabulary distribution \(\Delta (\mcal{V})\) at every masked position in a single forward pass.
For a canvas \(\mbf{x}\) where \(i\) is masked, write $\hat{x}_i$ for its greedy prediction and $\nu_i$ for its confidence:
\begin{equation}
    q_\theta^i(\cdot\mid\mathbf{x})\in\Delta(\mathcal{V}),
    \qquad
    \hat{x}_i:=\arg\max_{v\in\mathcal{V}}q_\theta^i(v\mid\mathbf{x}),
    \qquad
    \nu_i:=q_\theta^i(\hat{x}_i\mid\mathbf{x}),
    \qquad i\in M (\mbf{x}).
    \label{eq:mdlm-output}
\end{equation}
In a decoding round, one first performs model evaluation on the canvas \(\mbf{x}\), selects a subset of masked positions \(S \subseteq M (\mbf{x})\), assigns each \(i \in S\) a token by sampling or argmax, and finally commits the tokens by updating \(\mbf{x}\) at each \(i \in S\).
Suppose \(y_S \in \mcal{V}^{\abs{S}}\) were the predicted tokens for \(S\), then let \(\mbf{x} \oplus y_S\) denote the updated canvas.
In this work, we focus on greedy (argmax) decoding.

\textbf{Sequential and parallel decoding.}
We next formalize our specific notion of sequential and parallel decoding.
Suppose \(S = (s_1, \ldots, s_k) \subseteq M (\mbf{x})\) is a subset (sequence) of such masked positions ordered by confidence, i.e., \(\nu_{s_i} \geq \nu_{s_j}\) for all \(i \geq j\).
Then for a canvas \(\mbf{x}\), the \textit{sequential reference} \(x_S ^\star\) is the tokens obtained by greedily decoding along \(S\).
We equivalently describe this below via stylized canvas transitions, where each arrow denotes one model evaluation followed by an update:
\begin{equation}
    \mathbf{x}
    \xrightarrow{x^\star_{s_1}}
    \mathbf{x}\oplus x^\star_{s_1}
    \xrightarrow{x^\star_{s_2}}
    \cdots
    \xrightarrow{x^\star_{s_k}}
    \mathbf{x}\oplus x^\star_S,
    \qquad
    x^\star_{s_i}:=\arg\max_{v\in\mathcal{V}}
    q_\theta^{s_i}\!\left(v\mid\mathbf{x}\oplus x^\star_{S_{<i}}\right).
    \tag{Sequential}
    \label{eq:background-seq}
\end{equation}
Here, $S_{<i}=(s_1,\ldots,s_{i-1})$ is a prefix of \(S\), where let \(S_{<1} = \emptyset\).
In contrast, our notion of parallel decoding unmasks all the argmax predictions \(\hat{x}_S\) from a single model evaluation, as follows:
\begin{equation}
    \mathbf{x}
    \xrightarrow{\hspace{5em}\hat{x}_S\hspace{5em}}
    \mathbf{x}\oplus\hat{x}_S,
    \qquad
    \hat{x}_{s_i}:=\arg\max_{v\in\mathcal{V}}
    q_\theta^{s_i}(v\mid\mathbf{x}).
    \tag{Parallel}
    \label{eq:background-par}
\end{equation}
That is, sequential decoding conditions later predictions on earlier commits in $S$ and requires $|S|$ model evaluations, while parallel decoding skips these intermediate updates and requires only one.
We use the sequential output $x^\star_S$ as the reference that the parallel output $\hat{x}_S$ attempts to match, and measure their disagreement via the edit (Hamming) distance.

\section{Methodology}
\label{sec:method}

Our main idea is to view sequential--parallel inconsistency as \textit{regret}, and use it to decide which positions to decode.
This perspective yields two complementary algorithms: \PIGint{} (\cref{sec:pig-int}) controls regret within the current round, while \PIGext{} (\cref{sec:pig-ext}) controls the future regret that can be avoided by delaying commits; both rely on a lightweight predictor described in~\cref{sec:implementation}.

\begin{figure}[!t]
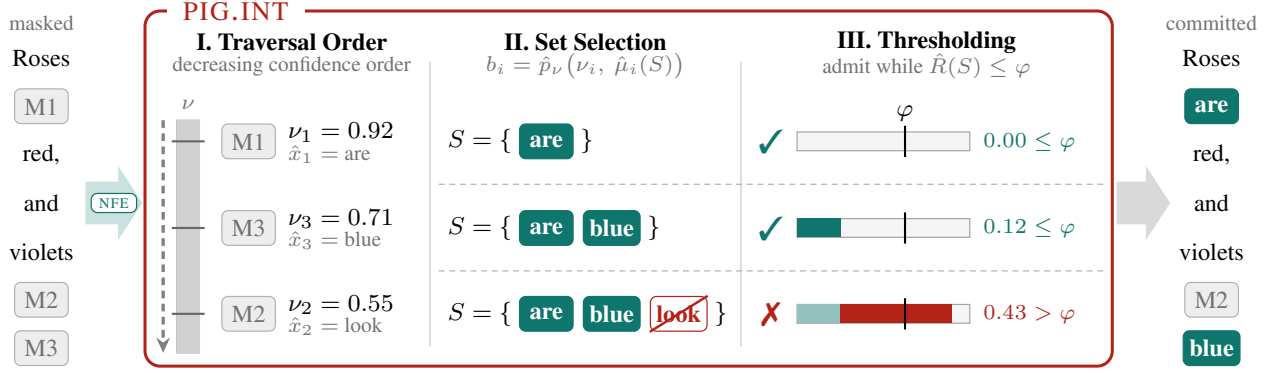

    \centering
    \squishfig{\tbpigintwalk}{0.90}
    \caption{\textbf{\PIGint{} controls single-round regret.}    
    One forward pass (\textsc{nfe}) predicts every masked slot and ranks them by confidence; that ranking is the traversal order (I).
    Candidates are iteratively added to $S$, each with its predicted incremental regret $b_i=\hat p_\nu(\nu_i,\hat\mu_i(S))$ (II), and admitted only while the predicted regret $\hat R(S)$ fits the budget $\varphi$ (III); the first candidate that does not fit closes the round.
    We then greedily decode all tokens in \(S\) in parallel, before starting the next round.
    }
    \label{fig:pig-int}
\end{figure}

\subsection{Parallel decoding as regret-constrained subset selection}
\label{sec:pig}
\label{sec:pig-int}

Given a partially masked canvas $\mathbf{x}$, let $S=(s_1,\ldots,s_k)$ be a confidence ordered sequence of positions from $M(\mbf{x})$ to be committed in parallel.
For a single round, our idealized objective is to maximize the number of parallel commits subject to a regret bound:
\begin{equation}
    \max_{\substack{S\subseteq M(\mathbf{x})\\ \text{confidence ordered}}}|S|
    \quad
    \text{subject to}
    \quad
    R_{\mathbf{x}}(S)
    :=
    \sum_{i \in S}
    \mathbf{1}\left\{
        x^\star_{s_i}\neq\hat{x}_{s_i}
    \right\}
    \leq \varphi.
    \label{eq:pig-oracle}
\end{equation}
We call $R_{\mathbf{x}}(S)$ the \emph{single-round regret}, the realized number of mismatches between the parallel output $\hat{x}_S$ and sequential reference $x^\star_S$.
However, solving this objective exactly is impractical for two main reasons: searching over ordered candidate sequence \(S\) is combinatorial, and even then, computing $R_{\mathbf{x}}(S)$ requires sequential replay at the cost of \(|S|\) model evaluations.

\textbf{Idea 1: Confidence ordering reduces the problem to prefix search.}
To avoid searching jointly over candidate orderings, we use current-pass confidence to define a fixed proposal order.
Specifically, we sort the masked positions so that $\nu_{s_1} \ge \nu_{s_2} \ge \ldots$ and scan candidates for \(S\) in that order.
This reduces the search over ordered candidate sets to choosing the longest \textit{prefix} of the confidence ordering that can be committed within the regret budget.

\textbf{Idea 2: Prefix regret is predictable from a single model evaluation.}
Even with the masked positions ordered, exactly computing each \(R_\mbf{x} (S)\) would still require sequential replay at the cost of \(|S|\) evaluations. However, a disagreement at the candidate can only arise if revealing $S$ moves the candidate's distribution over the vocabulary enough that the argmax changes.
Thus, inconsistency is governed by the slack in the top-1 vs top-2 predicted probabilities and by the size of the perturbation $S$ applies to it, both of which a single forward pass on \(\mbf{x}\) exposes up to two deliberate approximations: the slack is approximated via confidence (or entropy), and the perturbation through a sum of pairwise perturbation estimates over $j \in S$, the first-order term of a set effect whose higher-order interactions we truncate (\Cref{app:theory}).
We defer implementation details of \(\hat{R}_\mbf{x} (S)\) to~\cref{sec:implementation}, and we next present a simple greedy algorithm that we call \PIGint{}, also illustrated in~\cref{fig:pig-int}

\textbf{The \PIGint{} algorithm}
Given a canvas \(\mbf{x}\) and budget \(\varphi\), first initialize \(S = \varnothing\).
Then, run a single forward pass on \(\mbf{x}\) and sort the masked positions \(M(\mbf{x})\) in decreasing confidence.
Next, for each \(i \in M(\mbf{x})\), if \(\hat{R}_\mbf{x} (S \cup \{i\}) \leq \varphi\), we update \(S \gets S \cup \{i\}\).
Otherwise, we terminate this parallel decoding round and update the canvas as \(\mbf{x} \oplus \hat{x}_S\), where \(\hat{x}_S\) are the argmax tokens at \(S\).

\subsection{Look-ahead decoding as cross-round regret}
\label{sec:pig-ext}

We next take a step back and ask: \emph{can slowing down help?}
Once the masked positions are ordered, parallel decoding reduces to partitioning this ordering into rounds, with one model evaluation per round.
Whereas \PIGint{} greedily places each round boundary using single-round regret, we instead consider whether committing fewer tokens \textit{now} can reduce regret \textit{later}.

\begin{figure}[t]
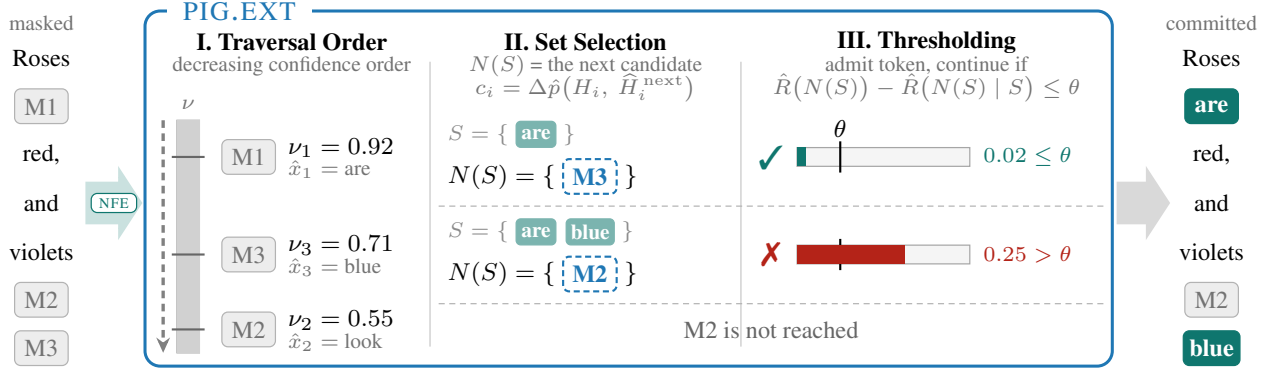

    \centering
    \squishfig{\tbpigextwalk}{0.90}
    \caption{\textbf{Cross-round regret measures the value of waiting.}
    On the canvas and confidence order of \Cref{fig:pig-int} (I), \PIGext{} admits each candidate and then prices the next candidate $N(S)$: $c_i$ is how much the mismatch risk of $i$ would drop after first committing $S$ (II).
    The round continues while $\hat R(N(S))-\hat R(N(S)\mid S)\le\theta$ (III): committing ``are'' changes little for M3, so the walk widens; committing ``blue'' would turn M2 from ``look'' into ``are'', so waiting is worth more than $\theta$, the round closes and commits only ``are'' and ``blue''.
    We implement with \(\abs{N(S)} = 16\) look-ahead positions.
    }
    \label{fig:value-of-waiting}
\end{figure}

For the fixed confidence ordering, let \(S\) be a candidate prefix and let \(N(S)\) denote the next candidate window in that ordering.
We write \(R_{\mathbf{x}}(N(S)\mid S)\) for the regret of decoding \(N(S)\) after first committing \(\hat{x}_S\) and re-evaluating, while keeping the sequential reference anchored at \(\mathbf{x}\).
We then define the \emph{cross-round regret}:
\begin{equation}
    C_{\mathbf{x}}(S)
    :=
    \underbrace{
        R_{\mathbf{x}}(S\cup N(S))-R_{\mathbf{x}}(S)
    }_{\text{decode $N(S)$ now}}
    -
    \underbrace{
        R_{\mathbf{x}}(N(S)\mid S)
    }_{\text{decode $N(S)$ after waiting}}.
    \label{eq:pig-ext-regret}
\end{equation}
Thus, a large positive $C_{\mathbf{x}}(S)$ means that waiting one round would avoid substantial regret on the upcoming candidates.
Using similar machinery as \PIGint{}; exactly computing cross-round regret requires additional model evaluations, but predicting it instead only requires a forward pass on \(\mbf{x}\).
We defer implementation to~\cref{sec:implementation}.
Nevertheless, access to a cross-round regret estimator yields a simple future--regret-aware algorithm, illustrated in~\cref{fig:value-of-waiting} and described next.

\textbf{\PIGext{}: controlling cross-round regret.}
\PIGext{} scans candidates in decreasing confidence, adding each candidate to $S$ before evaluating
$\hat{C}_{\mathbf{x}}(S)$ on the next window $N(S)$.
If $\hat{C}_{\mathbf{x}}(S)\le\theta$, waiting is predicted to buy little and the scan continues.
Otherwise, \PIGext{} closes the round and commits the current $S$ in parallel.

\subsection{A lightweight adapter enables one-pass regret prediction}
\label{sec:implementation}

\begin{figure}[!t]
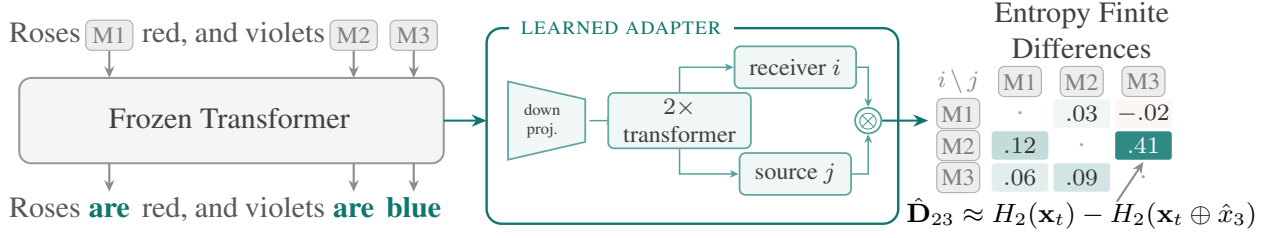

    \centering
    \squishfig{\tbadapterfig}{0.85}
    \caption{\textbf{Finite Differences Estimator (FDE).}
    The frozen DLM's pass feeds a learned adapter (${\approx}2.2$M parameters), which returns in the same pass the full directed matrix $\hat{\mathbf D}$ of entropy finite differences of \Cref{eq:finite-difference}, $\hat{\mathbf D}_{ij}\approx H_i(\mathbf x_t)-H_i(\mathbf x_t\oplus\hat x_j)$, the change in uncertainty at $i$ if $j$ were committed. Full implementation and training details in (\Cref{app:estimator}).}
    \label{fig:adapter}
\end{figure}

In \Cref{sec:pig}, we argued that sequential--parallel disagreement is governed by two signals: the slack in the argmax's probability, and the perturbation that committing $S$ applies to its distribution.
We summarize the slack by confidence $\nu_i$ or predictive entropy $H_i(\mbf{x}) := H(q_\theta^i(\cdot\mid\mbf{x}))$; the perturbation is counterfactual.
For masked positions $i,j\in M(\mbf{x})$, we use a lightweight adapter to estimate the pairwise perturbation through the directed entropy finite difference
\begin{equation}
    \hat{\mathbf{D}}_{ij}(\mbf{x})
    \approx
    H_i(\mbf{x})
    -
    H_i(\mbf{x}\oplus\hat{x}_j),
    \label{eq:finite-difference}
\end{equation}
where $\hat{x}_j$ is the current argmax at $j$.
Thus, $\hat{\mathbf{D}}_{ij}$ measures how committing $j$ changes uncertainty at $i$, but does not itself determine whether the argmax changes.
Learning all $\hat{\mathbf{D}}_{ij}$ requires counterfactual backbone evaluations, so a lightweight adapter instead predicts the full matrix $\hat{\mbf D}$ from the current decoding state (\cref{fig:adapter,app:estimator}).
Motivated by this representation, \PIGint{} summarizes a candidate's interaction with the accepted prefix as
\begin{equation}
    \hat{\mu}_i(S)
    :=
    \sum_{j\in S:\nu_j>\nu_i} 
    \max\left(
        |\hat{\mbf D}_{ij}|,
        |\hat{\mbf D}_{ji}|
    \right).
    \label{eq:aggregate-influence}
\end{equation}
This is a deliberately low-order summary, rather than an exact reconstruction of the set effect. In \Cref{app:theory}, we show that $\hat{\mu}_i(S)$ is the sum of first-order approximations of the larger entropy perturbation from $S$.
From these extracted features from a single forward pass, we calibrate disagreement jointly from confidence and pairwise-sum of perturbations using offline sequential replays:
\begin{equation}
\begin{aligned}
    \hat p_{\nu}(\nu,\mu)
    &\approx
    \Pr\left(
        x_i^\star\neq\hat{x}_i
        \mid
        \nu_i=\nu,\,
        \hat{\mu}_i=\mu
    \right),
    \\
    \hat p_H(H,\mu)
    &\approx
    \Pr\left(
        x_i^\star\neq\hat{x}_i
        \mid
        H_i=H,\,
        \hat{\mu}_i=\mu
    \right).
\end{aligned}
    \label{eq:mismatch-table}
\end{equation}
\PIGint{} then sums the confidence-conditioned mismatch probabilities along the prefix:
\begin{equation}
    \hat R_{\mbf{x}}(S)
    :=
    \sum_{i\in S}
    \hat p_\nu\left(
        \nu_{s_i},
        \hat{\mu}_{s_i}(S_{<i})
    \right).
    \label{eq:predicted-regret}
\end{equation}
The same representation can also estimate the value of waiting.
Because the canvas after committing $S$ is unavailable when the decision is made, \PIGext{} approximates each upcoming candidate's next-round entropy using the largest predicted directed entropy reduction from $S$:
\begin{equation}
    \hat H_i(\mbf{x}\mid S)
    :=
    H_i(\mbf{x})
    -
    \max_{j\in S}\hat{\mbf D}_{ij}.
    \label{eq:future-entropy}
\end{equation}
We then evaluate the entropy-conditioned disagreement model at the current and predicted next-round entropies:
\begin{equation}
    \hat C_{\mbf{x}}(S)
    :=
    \sum_{i\in N(S)}
    \left[
        \hat p_H\left(
            H_i(\mbf{x}),
            \hat{\mu}_i(S\cup N(S))
        \right)
        -
        \hat p_H\left(
            \hat H_i(\mbf{x}\mid S),
            \hat{\mu}_i(S\cup N(S))
        \right)
    \right].
    \label{eq:predicted-cross-regret}
\end{equation}
Their difference estimates the regret avoided by waiting.
The adapter and calibration tables are fitted once per backbone on a \textbf{single dataset mixture} ($3721$ examples drawn from OpenCoder, GSM8K-train, MATH-train, and MBPP-train), and reused at inference without additional backbone evaluations.
The adapter has approximately $2$M parameters and costs less than $0.1\%$ of a backbone forward pass in MACs; \cref{app:estimator} evaluates prediction quality, while \cref{app:theory} gives guarantees conditional on calibration.

\section{Experiments}
\label{sec:experiments}

\begin{figure}[!t]
    \centering
    \includegraphics[width=\textwidth]{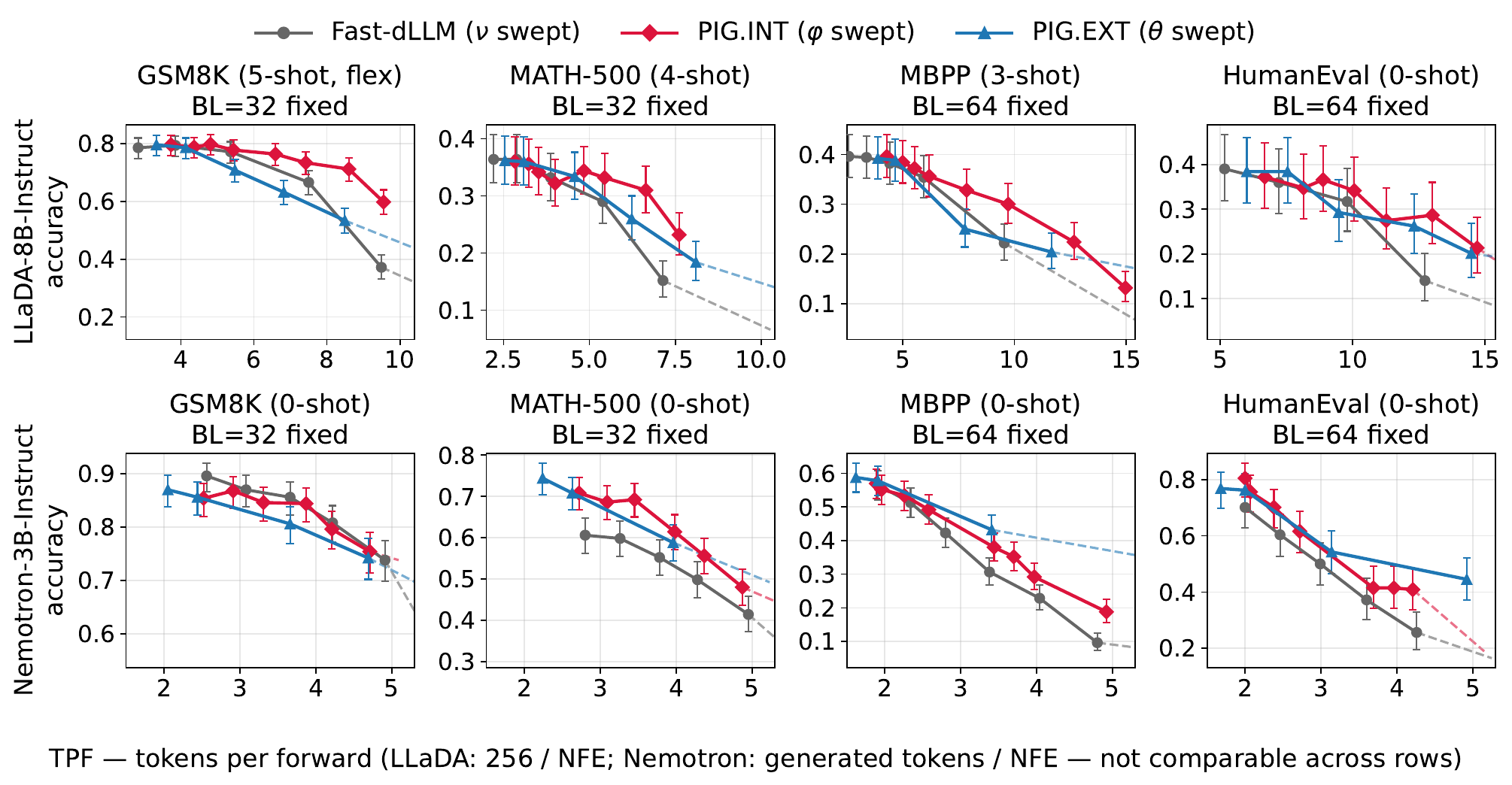}
    \caption{
    \textbf{\PIG{} improves the accuracy-TPF frontier of parallel decoding}
    On LLaDA-8B-Instruct (top) and Nemotron-Labs-Diffusion-3B-Instruct (bottom), both \PIGint{} (red) and \PIGext{} (blue) outperform a strong confidence-based baseline, Fast-dLLM (gray).
    We report Wilson $95\%$ error bars, and full numerical details in~\cref{app:results:grids,app:results:nemotron:grids}.
    }
    \label{fig:main-frontier}
\end{figure}

We now evaluate \PIG{} on parallel decoding, focusing on the tokens-per-forward (TPF) metric, the average number of tokens generated per model evaluation.
We give training details in~\cref{app:estimator}, and full evaluations in \Cref{app:results:grids}.

\textbf{Models, benchmarks, and baselines.}
We evaluate LLaDA-8B-Instruct~\citep{nie2025large} and Nemotron-Labs-Diffusion-3B~\citep{fu2026nemotron} on GSM8K, MATH-500, MBPP, and HumanEval.
For LLaDA, we use a $256$-token canvas and follow the Fast-dLLM~\citep{wu2025fastdllm} evaluation setup: GSM8K is 5-shot, MATH-500 4-shot, MBPP 3-shot, and HumanEval 0-shot.
For Nemotron, we use a $2048$-token canvas and evaluate all four benchmarks 0-shot.
Our primary baseline is confidence-threshold decoding with the same eligible positions as PIG: Fast-dLLM for LLaDA and Nemotron's native confidence decoder for Nemotron.
We use greedy commits throughout, report Wilson $95\%$ intervals, and compare TPF only within each model.

\textbf{Evaluation and PIG configurations.}
For GSM8K and MATH-500, we report flexible extraction in the main text and defer strict extraction to \Cref{app:results}; code is evaluated with pass@1.
\PIGint{} sweeps the single-round regret budget $\varphi$, while \PIGext{} sweeps the cross-round regret budget $\theta$; for LLaDA, we use $\varphi\in[0.015,3]$ and $\theta\in[0.5,4]$ with  \PIGext{} look-ahead window as $|N(S)|=16$.
On both models the \PIGext{} lane is the same operating rule: its two slowest points additionally floor admission at a confidence $\nu$ ($\theta{=}0.125$ with $\nu{=}0.9$ and $\theta{=}0.25$ with $\nu{=}0.8$), followed by the plain $\theta$ sweep.
The finite-difference estimator is trained once per backbone on training-side GSM8K, MATH, OpenCoder, and MBPP data, while the regret calibration tables are fitted on GSM8K-train and OpenCoder.
Both are reused unchanged across evaluation benchmarks, decoding budgets, and the semi-auto regressive block regimes considered below.

\textbf{PIG improves the accuracy--TPF frontier}
\label{sec:exp-frontier}

We evaluate \PIGint{} and \PIGext{} on LLaDA-8B-Instruct~\citep{nie2025large} and Nemotron-Labs-Diffusion-3B-Instruct~\citep{fu2026nemotron} on GSM8K, MATH-500, MBPP, and HumanEval, against Fast-dLLM~\citep{wu2025fastdllm} (on Nemotron it is the model's own decoder).
All three decoders share prompts, greedy commits, and the evaluation protocol; each is swept over its single knob (the threshold $\nu$, the budget $\varphi$, the budget $\theta$), and we measure decoding efficiency in TPF. For LLaDA-8B-Instruct, we additionally compare against DAPD, EB-Sampler, KLASS, and other learned policies. DAPD we run ourselves from the authors' released code under our protocol, so that comparison is paired; EB-Sampler and KLASS enter as the single operating points reported by \citet{kim2026dapd}, whose prompts and grader differ from ours (\Cref{app:results:dapd} measures the offset at the shared Fast-dLLM point).
\Cref{fig:main-frontier} shows the math benchmarks under fixed $32$-token blocks and the code benchmarks under fixed $64$-token blocks; the single-block and remaining block regimes, with every operating point tabulated, are in \Cref{app:results}.
Both decoders read their finite differences from one ${\sim}2$M-parameter adapter on the frozen backbone, fitted once per model on a \textbf{single data mixture} of planted entropy finite differences and never updated while decoding.
On held-out benchmarks it recovers the gold plants at a per-edge correlation of about $0.8$, and higher on the aggregate measurement the decoders actually consume, the pairwise-sum of perturbations $\hat\mu_i$.
\Cref{app:estimator} gives its architecture, training data and objective, and evaluates it on both backbones, the Nemotron twin at a block size held out of training.

\begin{table}[t]
\centering
\providecolor{deltapos}{HTML}{0F766E}\providecolor{deltaneg}{HTML}{B42318}
\caption{\textbf{Matched Accuracy TPF Increase.} Each cell: our TPF over that baseline's best operating point, and in parentheses our accuracy minus theirs at the \emph{measured} cell of ours that delivers it. Every cell is from \Cref{fig:llada-dapd}. Fast-dLLM and both DAPD variants are our own runs; EB-Sampler and KLASS are the cells published by \citet{kim2026dapd}. LLaDA-8B-Instruct, BL$=$64 fixed.}
\label{tab:matched-int}
\footnotesize
\setlength{\tabcolsep}{4pt}
\fitwidth{\begin{tabular}{l!{\color{black!22}\vrule width 0.4pt}r!{\color{black!22}\vrule width 0.4pt}r!{\color{black!22}\vrule width 0.4pt}r!{\color{black!60}\vrule width 0.7pt}r!{\color{black!22}\vrule width 0.4pt}r}
\toprule
 & \multicolumn{3}{c!{\color{black!60}\vrule width 0.7pt}}{our protocol (paired)} & \multicolumn{2}{c}{their protocol \citep{kim2026dapd}} \\
\cmidrule(lr){2-4} \cmidrule(lr){5-6}
benchmark & Fast-dLLM & DAPD-Direct & DAPD-Staged & EB-Sampler & KLASS \\
\midrule
\textit{PIG.INT} & \multicolumn{5}{c}{relative TPF increase (\textcolor{deltapos}{$+$} or \textcolor{deltaneg}{$-$} accuracy)} \\
\cmidrule(lr){2-6}
GSM8K & $1.15\times$ ($+0.0$) & $1.77\times$ (\textcolor{deltapos}{$+2.0$}) & $2.62\times$ (\textcolor{deltapos}{$+0.6$}) & $2.51\times$ (\textcolor{deltapos}{$+0.5$}) & $2.71\times$ (\textcolor{deltapos}{$+1.0$}) \\
MATH-500 & $1.13\times$ (\textcolor{deltaneg}{$-0.2$}) & $1.65\times$ (\textcolor{deltapos}{$+1.0$}) & $1.82\times$ (\textcolor{deltapos}{$+0.6$}) & $3.38\times$ (\textcolor{deltapos}{$+1.2$}) & $4.38\times$ (\textcolor{deltapos}{$+1.4$}) \\
MBPP & $1.66\times$ ($+0.0$) & $0.83\times$ (\textcolor{deltapos}{$+1.0$}) & $1.85\times$ (\textcolor{deltapos}{$+0.4$}) & $2.60\times$ (\textcolor{deltapos}{$+3.0$}) & $2.05\times$ (\textcolor{deltapos}{$+4.0$}) \\
HumanEval & $1.29\times$ (\textcolor{deltaneg}{$-1.8$}) & $1.46\times$ (\textcolor{deltapos}{$+2.4$}) & $4.28\times$ (\textcolor{deltapos}{$+3.1$}) & $2.88\times$ ($+0.0$) & $3.90\times$ (\textcolor{deltaneg}{$-0.6$}) \\
\midrule
\textit{PIG.EXT} & \multicolumn{5}{c}{relative TPF increase (\textcolor{deltapos}{$+$} or \textcolor{deltaneg}{$-$} accuracy)} \\
\cmidrule(lr){2-6}
GSM8K & $0.85\times$ (\textcolor{deltaneg}{$-0.6$}) & $1.04\times$ (\textcolor{deltapos}{$+4.4$}) & $1.91\times$ (\textcolor{deltapos}{$+1.4$}) & $1.47\times$ (\textcolor{deltapos}{$+2.9$}) & $1.59\times$ (\textcolor{deltapos}{$+3.4$}) \\
MATH-500 & $0.87\times$ (\textcolor{deltapos}{$+0.6$}) & $1.00\times$ (\textcolor{deltapos}{$+2.4$}) & $1.57\times$ (\textcolor{deltapos}{$+0.8$}) & $2.84\times$ (\textcolor{deltapos}{$+2.0$}) & $2.97\times$ (\textcolor{deltapos}{$+3.8$}) \\
MBPP & $1.50\times$ (\textcolor{deltaneg}{$-0.4$}) & $0.90\times$ (\textcolor{deltapos}{$+0.2$}) & $1.73\times$ (\textcolor{deltapos}{$+0.8$}) & $2.40\times$ (\textcolor{deltapos}{$+1.0$}) & $0.98\times$ (\textcolor{deltapos}{$+12.8$}) \\
HumanEval & $1.46\times$ (\textcolor{deltaneg}{$-0.6$}) & $1.10\times$ (\textcolor{deltapos}{$+6.7$}) & $4.06\times$ (\textcolor{deltapos}{$+0.6$}) & $3.25\times$ (\textcolor{deltapos}{$+1.2$}) & $4.40\times$ (\textcolor{deltapos}{$+0.6$}) \\
\bottomrule
\end{tabular}}
\end{table}

PIG consistently improves the accuracy--TPF frontier over Fast-dLLM, with the largest gains appearing as decoding becomes more parallel. We observe that globally, \PIGint{} performs best at the slower end of the frontier while \PIGext{} performs best at the faster end of the frontier. We note that through our experiments, PIG helps globally on the whole frontier on all tested datasets and block sizes with the most impressive results occurring at medium length block sizes. 

\Cref{tab:matched-int} reports the relative increase in TPF of PIG on LLaDA-8B-Instruct against Fast-DLLM and dependency-aware decoders DAPD~\citep{kim2026dapd}, EB-Sampler~\citep{benhamu2025ebsampler}, and KLASS~\citep{kim2025klass}, at the four-block ($64$-token) semi-autoregressive block decoding.
At matched accuracy \PIGint{} decodes $1.5$ to $1.8\times$ more tokens per forward than DAPD-Direct on three of the four benchmarks, and $0.8\times$ on MBPP, in every case at equal or higher accuracy; against the reported EB-Sampler and KLASS points it decodes with $2.1$ to $4.4\times$ higher TPF. Meanwhile, we observe that \PIGext{} struggles a bit more at the slower end of the accuracy-TPF frontier. Against Fast-dLLM, \PIGext{} decodes faster on code tasks only at $1.46-1.50\times$ higher TPF; against DAPD-Direct  it achieves an average increase of 3.4 percentage points at similar TPF ($0.9-1.1\times$); against the reported EB-Sampler and KLASS points it decodes $1.5$ to $4.4\times$ more tokens per forward.
\Cref{app:analyses} draws the full frontiers behind these numbers (\Cref{fig:llada-dapd}) and shows that PIG matches, at their own operating points, the RL unmasking policies of \citet{jazbec2025unmasking} that are trained per benchmark (\Cref{fig:llada-jazbec-bl32}).

In contrast, \Cref{tab:matched-ext} reports interpolated accuracy of the same set-up at fixed throughputs per model. At $10$ TPF on LLaDA, \PIGint{} is the most accurate lane on every benchmark, $4$ to $19$ points above Fast-dLLM; at $15$ TPF \PIGext{} takes over on GSM8K and MBPP with $10$ to $14$ points over Fast-dLLM, and on MATH-500 it is the only lane still measured at $15$ TPF.
On Nemotron-3B-Instruct, the three lanes stay within a few points of one another on GSM8K, while on code \PIGext{} retains $36.9$ against $9.1$ points on MBPP at $5$ TPF and $44.5$ against $19.5$ on HumanEval at $4.9$ TPF, the last measured cell of its sweep.

\begin{table}[t]
\centering
\providecolor{deltapos}{HTML}{0F766E}\providecolor{deltaneg}{HTML}{B42318}
\caption{\textbf{Accuracy at Matched TPF.} Reported accuracies and gains relative to Fast-dLLM in parentheses, when above (\textcolor{deltapos}{teal}) and below (\textcolor{deltaneg}{red}). Every lane is read at exactly the target throughput, linearly between the two measured cells that bracket it. A lane whose sweep ends before the target is not extrapolated: its last measured cell is given instead, as accuracy@TPF. All rows at BL$=$64 fixed.}
\label{tab:matched-ext}
\footnotesize
\setlength{\tabcolsep}{4pt}
\fitwidth{\begin{tabular}{l!{\color{black!22}\vrule width 0.4pt}r!{\color{black!22}\vrule width 0.4pt}r!{\color{black!22}\vrule width 0.4pt}r!{\color{black!22}\vrule width 0.4pt}r!{\color{black!22}\vrule width 0.4pt}r!{\color{black!22}\vrule width 0.4pt}r}
\toprule
benchmark & Fast-dLLM & PIG.INT & PIG.EXT & Fast-dLLM & PIG.INT & PIG.EXT \\
\midrule
\textit{LLaDA-8B-Instruct} & \multicolumn{3}{c}{interpolated accuracy @ TPF $=10$} & \multicolumn{3}{c}{interpolated accuracy @ TPF $=15$} \\
\cmidrule(lr){2-4} \cmidrule(lr){5-7}
GSM8K & $51.4$ & $70.5$ (\textcolor{deltapos}{$+19.1$}) & $56.0$ (\textcolor{deltapos}{$+4.6$}) & $25.6$ & $34.0$ (\textcolor{deltapos}{$+8.5$}) & $39.7$ (\textcolor{deltapos}{$+14.2$}) \\
MATH-500 & $14.1$ & $25.5$ (\textcolor{deltapos}{$+11.4$}) & $20.0$ (\textcolor{deltapos}{$+5.9$}) & $9.8$ {\scriptsize@12.2} & $15.4$ {\scriptsize@13.4} (---) & $13.7$ (---) \\
MBPP & $21.0$ & $29.3$ (\textcolor{deltapos}{$+8.3$}) & $22.4$ (\textcolor{deltapos}{$+1.4$}) & $7.8$ & $13.2$ {\scriptsize@15.0} (\textcolor{deltapos}{$+5.3$}) & $17.5$ (\textcolor{deltapos}{$+9.7$}) \\
HumanEval & $30.4$ & $34.3$ (\textcolor{deltapos}{$+3.8$}) & $28.7$ (\textcolor{deltaneg}{$-1.7$}) & $9.2$ & $20.2$ (\textcolor{deltapos}{$+11.0$}) & $19.7$ (\textcolor{deltapos}{$+10.5$}) \\
\midrule
\textit{Nemotron-3B-Instruct} & \multicolumn{3}{c}{interpolated accuracy @ TPF $=3$} & \multicolumn{3}{c}{interpolated accuracy @ TPF $=5$} \\
\cmidrule(lr){2-4} \cmidrule(lr){5-7}
GSM8K & $86.5$ & $83.7$ (\textcolor{deltaneg}{$-2.8$}) & $83.7$ (\textcolor{deltaneg}{$-2.8$}) & $76.1$ & $76.7$ (\textcolor{deltapos}{$+0.5$}) & $75.9$ (\textcolor{deltaneg}{$-0.3$}) \\
MBPP & $38.3$ & $43.7$ (\textcolor{deltapos}{$+5.5$}) & $47.2$ (\textcolor{deltapos}{$+8.9$}) & $9.1$ & $18.8$ {\scriptsize@4.9} (\textcolor{deltapos}{$+9.5$}) & $36.9$ (\textcolor{deltapos}{$+27.8$}) \\
HumanEval & $49.8$ & $55.8$ (\textcolor{deltapos}{$+6.0$}) & $57.0$ (\textcolor{deltapos}{$+7.2$}) & $18.8$ & $22.4$ (\textcolor{deltapos}{$+3.6$}) & $44.5$ {\scriptsize@4.9} (\textcolor{deltapos}{$+25.0$}) \\
\bottomrule
\end{tabular}}
\end{table}

Four further experiments sit in \Cref{app:analyses}.
We propose dynamic block decoding, which expands the active block when the candidate pool starves, lifting the frontier for every lane and extending \PIGint{}'s reach from $9.6$ to $12.0$ TPF on GSM8K and from $7.6$ to $9.7$ TPF on MATH-500 (\Cref{app:analyses:dynamic}).
Additionally, a local path-curvature diagnostic measures how committing a set shifts upcoming-candidate entropies, finding both variants widen the round where the decode path is locally flat and narrows it where it bends (\Cref{app:analyses:curvature}).
Next, we test generalization on MMLU, an unseen multiple-choice benchmark spanning $57$ subjects, and observe performance consistent with our main results (\Cref{app:analyses:domain-transfer}).
Finally, we audit both regret estimates during live decoding on LLaDA-8B-Instruct by replaying committed sets sequentially: \PIGint{}'s realized regret tracks the prediction and stays below $\varphi$ at every operating point, and \PIGext{}'s cross-round statistic tracks the realized cross-round regret at the slow end and over-predicts it at large $\theta$ (\Cref{app:analyses:calibration}).

\section{Discussion, Future Work, and Conclusion}
\label{sec:conclusion}
\label{sec:discussion}

\textbf{Parallel decoding for DLMs.}
Parallel decoding is the primary mechanism by which DLMs achieve test-time speedups.
Whereas existing approaches typically seek sets of positions that are approximately conditionally independent, we instead select parallel commits by directly controlling their disagreement with sequential decoding.
Empirically, this disagreement provides an effective decoding target, suggesting a complementary design principle for future parallel decoders.

\textbf{Why is disagreement predictable?}
Our key hypothesis is that sequential--parallel disagreement can be predicted from current confidence and pairwise perturbations from earlier commits.
Prior dependency-aware decoders also use pairwise interactions~\citep{kim2026dapd,ringel2026demask}, but to our knowledge \PIG{} is the first to target disagreement with sequential decoding directly.
That such a simple predictor works suggests that richer estimators could further improve performance.

\textbf{Limitations.}
We use greedy sequential decoding as the reference trajectory; other sequential policies could induce different notions of regret and different preferred parallel sets.
Our predictor and calibration transfer across tasks without benchmark-specific fitting, but still depend on the quality of training-side data.
Moreover, although \PIGint{} and \PIGext{} excel in different regimes, we currently lack a principled rule for choosing which to use on a new model or task.

\textbf{Future work.}
Our estimator could benefit from richer architectures, objectives, or training data.
Sequential--parallel regret may also be predictable from features beyond confidence and aggregate pairwise perturbation.
A natural next step is a unified decoder that decides when to control current-round regret and when to invoke \PIGext{}'s look-ahead criterion.
More generally, methods could jointly learn the regret estimator and selection policy rather than treating them separately.

\textbf{Conclusion.}
We study how to accelerate DLM inference by choosing which masked positions can safely be committed together.
Our algorithmic framework, \PIG{}, formulates this decision through sequential--parallel regret and shows that this quantity can be predicted cheaply from a single model pass.
Across two DLM families and standard math and code benchmarks, \PIGint{} and \PIGext{} improve the accuracy--compute frontier over confidence-based parallel decoding.

\bibliography{references}
\bibliographystyle{iclr2027_conference}

\vspace{\baselineskip}\noindent{\small\textit{Image credit.} The pig-face glyph in the title is from Twemoji
(\copyright\ Twitter, Inc.\ and other contributors), licensed CC~BY~4.0, \url{https://creativecommons.org/licenses/by/4.0/}.}

\clearpage
\appendix
\startcontents[appendix]
\section*{Appendix Contents}
\printcontents[appendix]{l}{1}{\setcounter{tocdepth}{2}}
\clearpage

\section{Theoretical Motivation and Analysis}
\label{app:theory}
\label{app:derivations}
\label{app:ordercost}

This appendix gives the mathematical grounding of \PIG{} in two parts. \Cref{app:theory:objective,app:theory:mobius,app:theory:lop}
motivate the representation the decoder reads from one forward pass: the decode-ordering objective, its
deployable entropy form, its M\"obius expansion (whose first-order coefficients are exactly the pairwise entropy
finite differences of \Cref{eq:finite-difference}), and the NP-hardness of the exact ordering problem on that
graph, which is why \PIG{} fixes the order by confidence and spends its budget on admission.
\Cref{app:theory:guarantees} then states what the two decoders certify: a per-pass budget on parallel--sequential
disagreement for \PIGint{} (Proposition \ref{prop:temp-int}) and a per-accelerated-token regret bound for \PIGext{}
(Proposition \ref{thm:temp-ext-cascade}), both under a single calibration assumption on the inconsistency estimator.

\subsection{The decode-ordering objective and its entropy form}
\label{app:theory:objective}
\label{sec:method:objective}
\label{sec:background:objective}

Recall the MDLM training objective as the masked cross-entropy
\begin{equation}
  \mathcal{L}_{\mathrm{MDLM}}(\theta)\;=\;\mathbb{E}_{t\sim U[0,1]}\;\mathbb{E}_{\mathbf{x}\sim p}
  \Big[\frac{1}{t}\sum_{i\in M(t)}-\log q_\theta\big(x_i\mid \mathbf{x}(t)\big)\Big],
  \label{eq:mdlm-elbo}
\end{equation}
and sampling a noise level and a uniform mask is the same experiment as sampling a uniformly random
\emph{order} $\pi$ and revealing a random prefix of it, so \Cref{eq:mdlm-elbo} is the expected
teacher-forced chain loss over random orders~\citep{ou2024radd, hoogeboom2021ardm}:
\begin{equation}
  \mathcal{L}_{\mathrm{MDLM}}(\theta)
  \;=\;\mathbb{E}_{\pi\sim\mathrm{Unif}}\;\mathbb{E}_{\mathbf{x}\sim p}
  \Big[\textstyle\sum_k -\log q_\theta\big(x_{\pi(k)}\mid x_{\pi(<k)}\big)\Big].
  \label{eq:elbo-as-orders}
\end{equation}
Training \emph{averages} over orders; decoding \emph{executes} one. For a fixed order the chain rule
factorizes both the data law and the sampler the order induces,
\begin{equation}
  p(\mathbf{x})=\prod_k p\big(x_{\pi(k)}\mid x_{\pi(<k)}\big),
  \qquad
  q_\theta^\pi(\mathbf{x})=\prod_k q_\theta\big(x_{\pi(k)}\mid x_{\pi(<k)}\big),
  \label{eq:factorization}
\end{equation}
the left for every $\pi$, since the data admits all orders equally. As such diffusion models can be viewed as \textit{any order auto-regressive models}. A finite-capacity model cannot
match $p$ under every factorization at once, so orders differ in how far the sampler deviates from
$p$; the deviation of an order is its expected teacher-forced loss,
\begin{multline}
  \mathbb{E}_{\mathbf{x}\sim p}\Big[\textstyle\sum_k -\log q_\theta\big(x_{\pi(k)}\mid x_{\pi(<k)}\big)\Big]
  \;=\; H(p)\;+\;\mathrm{KL}\!\big(p\,\|\,q_\theta^\pi\big),\\
  \text{where } \quad \mathrm{KL}\!\big(p\,\|\,q_\theta^\pi\big)=\sum_k \mathbb{E}_{x_{\pi(<k)}}\Big[\mathrm{KL}\big(p(\cdot\mid x_{\pi(<k)})\,\|\,q_\theta(\cdot\mid x_{\pi(<k)})\big)\Big],
  \label{eq:ce-kl}
\end{multline}
where the second equality substitutes \Cref{eq:factorization} and splits the sum by the tower rule.
$H(p)$ does not depend on $\pi$, so all order dependence sits in the KL term: the cost an order pays
is an accumulated sum of per-step conditional errors, one per reveal, and order selection at a fixed
model is
\begin{equation}
    \text{Idealized Inference Objective} = \min_\pi\mathrm{KL}(p\,\|\,q_\theta^\pi),
    \label{eq:idealized-objective}
\end{equation} a combinatorial problem over $n!$ orders of a cost
in which each token's term depends on its whole revealed set.

\paragraph{The entropy form.}
\label{sec:ordercost:measures}
\label{sec:method:entropy}
\label{sec:ordercost:spread}
The summands of \Cref{eq:ce-kl} are \emph{gold} readouts, measurable under teacher forcing and
invisible at deployment, where the revealed states are the sampler's own commits. The deployable
substitute is the model's own expectation of the same per-step loss,
\begin{equation}
H\big(q_\theta(\cdot\mid x_{\pi(<k)})\big)=\mathbb{E}_{y\sim q_\theta(\cdot\mid x_{\pi(<k)})}\big[-\log q_\theta(y\mid x_{\pi(<k)})\big],
\label{eq:entropy-readout}
\end{equation}
so that minimizing the cumulative decode entropy,
\begin{equation}
    \text{Practical Inference Objective} = \min_{\pi} \sum_k \mathbb{E}_{\hat x_{\pi(<k)}}\Big[H\big(q_\theta(\cdot\mid \hat x_{\pi(<k)})\big)\Big],
    \label{eq:practical-objective}
\end{equation}
selects the trajectory the model itself prices cheapest: each token is committed once its context
has collapsed the uncertainty about it. Everything below applies verbatim with entropy readouts in
place of gold readouts, and it is the entropy form that \PIG{} measures.

\subsection{The M\"obius expansion and the graph representation}
\label{app:theory:mobius}
\label{sec:lop:linearization}

Treat each token's readout as a \emph{set function} of its revealed context. Over the masked span
$V=\{1,\dots,n\}$ write
\begin{equation}
  f_i(S)\;=\;-H_i(\mathbf{x}\oplus x_S)\quad\text{(entropy)},
  \qquad S\subseteq V\setminus\{i\},
  \label{eq:setfun}
\end{equation}
so that order selection maximizes a sum of set functions, each evaluated on the prefix the order
gives it. Every set function decomposes by inclusion--exclusion through its \emph{M\"obius
transform}, the Taylor series of a set function~\citep{rota1964mobius, grabisch2016setfunctions}:
\begin{equation}
  G(\pi)\;=\;\sum_i f_i\big(\pi(<i)\big)\;=\;\sum_i\,\sum_{T\subseteq\pi(<i)} m_i(T),
  \label{eq:setfun-objective}
\end{equation}
where $m_i(T)$ is the part of $f_i$ attributable jointly to the subset $T$ and to nothing smaller:
$m_i(\varnothing)=f_i(\varnothing)$, $m_i(\{j\})=f_i(\{j\})-f_i(\varnothing)$, and so on. Truncating the
expansion at singletons is the first-order approximation in the sense of set functions, the same
move that turns a smooth loss into a gradient step.

Expanding around \emph{any} base context $S_0$,
the first-order coefficients are finite differences,
\begin{equation}
  f_i(S)\ \approx\ f_i(S_0)\ +\!\!\sum_{j\in S\setminus S_0}\!\! \mathbf{D}_{ij}(S_0),
  \qquad
  \mathbf{D}_{ij}(S_0)\;=\;f_i(S_0\cup\{j\})-f_i(S_0),
  \label{eq:ctx-edge}
\end{equation}
the discrete derivative of $f_i$ in direction $j$ at $S_0$: the gain in $i$'s readout from revealing
$j$ given that $S_0$ is already present. Substituting into \Cref{eq:setfun-objective} and grouping by
interaction order decomposes the objective exactly, in correspondence with
$f(x_0+h)=f(x_0)+\nabla f(x_0)^{\!\top}h+O(\|h\|^2)$:
\begin{equation}
  G(\pi)\;=\;
  \underbrace{\sum_i f_i(S_0)}_{\text{constant}}
  \;+\;
  \underbrace{\sum_{i,j\,:\,\pi(j)<\pi(i)}\!\!\mathbf{D}_{ij}(S_0)}_{\text{first order: collected edges}}
  \;+\;
  \underbrace{\sum_i\;\sum_{\substack{T\subseteq\pi(<i)\\ |T|\ge 2}}\! m_i(T)}_{\text{higher order: truncated}}.
  \label{eq:mobius-lop}
\end{equation}
\paragraph{The representation is the one we learn.}
The first-order term is a directed graph on the masked positions: the numbers $\mathbf{D}_{ij}(S_0)$, one per
ordered pair, are its adjacency matrix, and the term collects exactly the edges an order picks up.
At deployment the base context is the current canvas, $S_0=\mathbf{x}_t$, and under the entropy
readout of \Cref{eq:setfun} the edge is
\begin{equation}
  \mathbf{D}_{ij}(\mathbf{x}_t)\;=\;-H_i(\mathbf{x}_t\oplus\hat x_j)+H_i(\mathbf{x}_t),
  \label{eq:edge-is-fde}
\end{equation}
the directed entropy finite difference of \Cref{eq:finite-difference}. The matrix the adapter of
\Cref{sec:implementation} predicts in one pass is therefore the first-order M\"obius coefficient
field of the decode-ordering objective at the current state, re-expanded at every round; the
higher-order layer of \Cref{eq:mobius-lop} is what a single forward pass cannot see, and it is the
part that parallel commits expose (within a round, several edges are collected at once, none of
them re-evaluated in the others' presence). \PIGint{} aggregates the magnitudes of these edges into
$\hat\mu_i(S)$ (\Cref{eq:aggregate-influence}) as the one-pass proxy for how much of a candidate's
readout is being extrapolated through omitted interactions, and \PIGext{} reads their sign to
predict the next round's entropies (\Cref{eq:future-entropy}); the calibration tables of
\Cref{eq:mismatch-table} then convert the proxy into a mismatch probability.

\subsection{The first-order problem is a linear ordering problem, and it is NP hard}
\label{app:theory:lop}
\label{sec:lop}
\label{sec:lop:formulation}

\begin{definition}[Denoising-order LOP]
\label{def:lop}
Given an edge graph $\mathbf{D}\in\mathbb{R}^{n\times n}$, the denoising-order linear ordering problem is
\begin{equation}
  \pi^\star=\arg\max_{\pi\in S_n}\ W_{\mathbf{D}}(\pi),
  \qquad
  W_{\mathbf{D}}(\pi)=\sum_{i,j\,:\,\pi(j)<\pi(i)} \mathbf{D}_{ij},
  \label{eq:lop-objective}
\end{equation}
the total edge mass flowing forward in $\pi$: place each token \emph{after} the tokens that most
raise its readout.
\end{definition}

Maximizing the first-order term of \Cref{eq:mobius-lop} over orders is exactly \Cref{eq:lop-objective}.
Its maximizer is the maximum-weight acyclic subgraph and the complementary backward mass is the
minimum feedback arc set, so the LOP is NP-hard~\citep{karp1972reducibility} and
APX-hard~\citep{newman2000maxacyclic, charon2010updated}; under the Unique Games Conjecture no
polynomial algorithm beats a random order by a constant~\citep{guruswami2008beating}, and the
tournament case is Kemeny rank aggregation~\citep{kemeny1959mathematics, ailon2008aggregating}. The
hardness is intrinsic to the graph representation rather than to our edges: \citet{zhou2026attnsampler}
bound the order objective of a one-layer attention model and prove the corresponding weighted LOP
NP-hard on the attention graph, the same problem our first-order truncation poses on the
finite-difference graph, and then retreat to a closed-form rank by attention column sums (the gated
attention rule of AGDO~\citep{deng2026agdo}). Reached on either graph --- gold or entropy finite
differences here, raw attention there --- the conclusion is that the exact ordering problem is not
the object to solve per round. Generated spans are short enough that local search would land within
about $1\%$ of the LOP optimum~\citep{marti2012lolib}, but the truncation itself, not the solver, is
the binding error: the first-order objective prices a token's first two context slots exactly and
extrapolates every deeper slot through the omitted interactions, so an exact LOP solution optimizes a
largely speculative tail. \PIG{} therefore does not order: it takes the confidence order the model
already provides as the search order, and spends its one knob on \emph{admission} --- how many tokens
of that order can be committed in one pass before the omitted interactions, read through the
first-order edges, are predicted to change the tokens (\Cref{sec:pig}). The regret that \PIG{}
bounds is the parallel-versus-sequential disagreement on exactly the sets whose within-round edges
the first-order term collects without re-evaluation.

\subsection{Guarantees for \PIGint{} and \PIGext{}}
\label{app:theory:guarantees}
\label{app:temp-analysis}

We now give guarantees on the consistency of the parallel decode against the
sequential decode it stands in for. First, we provide the theoretical setup.

\paragraph{Setup.}

Recall the Sequential and parallel decoding from \Cref{sec:background}. Let $S=(s_1,\ldots,s_k)\subseteq M$ be an ordered sequence of distinct masked positions, and write $\mathbf{x}\oplus x_S$ for the canvas obtained by updating $S$ with tokens $x_S$.
We compare sequential and parallel decoding through the following stylized canvas transitions, where each arrow denotes one model evaluation followed by the indicated commit.
For the sequential reference $x^\star$, the model is evaluated separately before each commit:
\begin{equation}
    \mathbf{x}
    \xrightarrow{x^\star_{s_1}}
    \mathbf{x}\oplus x^\star_{s_1}
    \xrightarrow{x^\star_{s_2}}
    \cdots
    \xrightarrow{x^\star_{s_k}}
    \mathbf{x}\oplus x^\star_S,
    \qquad
    x^\star_{s_i}:=\arg\max_{v\in\mathcal{V}}
    q_\theta^{s_i}\!\left(v\mid\mathbf{x}\oplus x^\star_{s_{<i}}\right).
    \tag{Sequential}
\end{equation}
Here $s_{<i}=(s_1,\ldots,s_{i-1})$.
Parallel decoding instead commits the current predictions $\hat{x}_S$ from a single model evaluation:
\begin{equation}
    \mathbf{x}
    \xrightarrow{\hspace{5em}\hat{x}_S\hspace{5em}}
    \mathbf{x}\oplus\hat{x}_S,
    \qquad
    \hat{x}_{s_i}:=\arg\max_{v\in\mathcal{V}}
    q_\theta^{s_i}(v\mid\mathbf{x}).
    \tag{Parallel}
\end{equation}
Thus, sequential decoding conditions later predictions on earlier commits in $S$ and requires $|S|$ model evaluations, while parallel decoding skips these intermediate updates and requires only one.
We use the sequential output $x^\star_S$ as the reference that the parallel output $\hat{x}_S$ attempts to match, which we measure by their Hamming distance on \(S\) only.

Fix a round with canvas $\mathbf{x}_t$ and committed set $S=\{s_1,\dots,s_{|S|}\}\subseteq M$ in
decreasing confidence. Every set the walk admits is a prefix of the confidence order, so the
sequential chain feeding a position $i$ runs over $M_{\succ i}=\{j\in M:\nu_j>\nu_i\}$ whatever the
round, and only the two canvases the walks start from need naming. For a target anchor
$\mathbf{x}$ and a parallel anchor $\mathbf{x}'$ let
\begin{equation}
F_\theta(i;\mathbf{x},\mathbf{x}')\;=\;\mathbf{1}\big\{x^*_i\neq\hat{x}_i\big\},
\qquad
x^*_i:=\arg\max_{v\in\mathcal{V}}q_\theta^{i}\big(v\mid\mathbf{x}\oplus x^*_{M_{\succ i}}\big),
\qquad
\hat{x}_i:=\arg\max_{v\in\mathcal{V}}q_\theta^{i}\big(v\mid\mathbf{x}'\big),
\label{eq:temp-flip}
\end{equation}
be the event where the parallel decode fails to sample from the joint distribution over
language: the event that the parallel commit at $i$, predicted from $\mathbf{x}'$, differs from the
target sequential decode replayed in confidence order from $\mathbf{x}$. Within a round both walks
are anchored at the same canvas, which is what makes the round a single-pass commit, and we write
$F_\theta(i;\mathbf{x}_t,\mathbf{x}_t)$; \PIGext{} below advances the second argument by one round
and leaves the first where it is. Given the two anchors and a frozen model $q_\theta$, both decodes
are deterministic argmaxes, so $F_\theta$ is a deterministic $\{0,1\}$ quantity; all probabilistic
statements below are over the population $\mathcal{D}$ of configurations $(\mathbf{x}_t,S,i)$ that
decoding visits.

Aggregating rounds requires saying exactly \emph{which} sequential decode each round is compared
against; the following definition makes the reference precise.

\begin{definition}[Re-anchored sequential reference and piece-wise total inconsistency]
\label{def:temp-reference}
Run the parallel decoder for rounds $r=1,\dots,T$ (so $T=$ NFE), producing canvases
$\mathbf{x}_{t_1},\dots,\mathbf{x}_{t_T}$ and committed sets $S_1,\dots,S_T$. The
\emph{re-anchored reference} at round $r$ is the sequential decode of $S_r$ launched from the
canvas $\mathbf{x}_{t_r}$ that the \emph{parallel} decoder actually reached:
\begin{equation}
x^*_{s_i}\;=\;\arg\max_{v\in\mathcal{V}}\;q_\theta\big(v\;\big|\;\mathbf{x}_{t_r}\oplus
x^*_{s_{<i}}\big),\qquad s_i\in S_r,
\label{eq:temp-reference}
\end{equation}
conditioning within the round on its \emph{own} outputs $x^*_{s_{<i}}$, and across rounds on the
parallel decoder's canvas. The
\emph{piece-wise total inconsistency} of the decode is
\begin{equation}
R_{\text{total}}\;=\; \sum_{r=1}^{T} R(S_r)\;=\;\sum_{r=1}^{T}\ \sum_{i\in S_r} F_\theta(i;\mathbf{x}_{t_r},\mathbf{x}_{t_r}),
\label{eq:temp-decode-inconsistency}
\end{equation}
the total count of round-level disagreements against these re-anchored references.
$R_{\text{total}}$ is \emph{not} the Hamming distance between the parallel output and a
from-scratch sequential decode of the whole problem: after a disagreement, both trajectories in
every later round condition on the disagreeing token, so the reference forgives past divergence
and counts only the fresh disagreement each round introduces.
\end{definition}

\begin{definition}[$\bar\epsilon$-calibrated inconsistency estimator]
\label{def:temp-calibrated}
Let $\mathcal{T}_i$ be a $\mathcal{G}_i$-measurable summary of the walk's observables and  $\mathcal{G}_i\;=\;\sigma\Big(\big(\mathcal{T}_j\big)_{j\preceq i}\Big)$ be the sigma-algebra over all prior observable summaries of the round. A \emph{parallel inconsistency estimator} $\hat p$ maps a summary statistic $\mathcal{T}_i$ to $[0,1]$ and is called
\emph{$\bar\epsilon$-calibrated} if for every canvas~-~set pair $(\mathbf{x}_t,S)$, its error is bounded by $\bar{\epsilon}(\mathcal{T})$:
\begin{equation}
\Big|\,\mathbb{E}_{\mathcal{D}}\big[F_\theta(i;\mathbf{x}_t,\mathbf{x}_t)\;\big|\;\mathcal{G}_i\big]\;-\;\hat p\big(\mathcal{T}_i\big)\,\Big|\;\le\;\bar\epsilon(\mathcal{T}).
\label{eq:temp-cal}
\end{equation}
To simplify notation, we write $\bar{\epsilon} = \bar{\epsilon}(\mathcal{T}).$
\end{definition}

Calibration asks that the deployed estimate be $\bar\epsilon$-close to the truth
\emph{conditionally on everything the walk observes}.

\begin{proposition}[\PIGint{}: piece-wise parallel consistency over a whole decode]
\label{prop:temp-int}
Let $\hat p$ be an $\bar\epsilon$-calibrated inconsistency estimator
(Definition \ref{def:temp-calibrated}). Then every round of \Cref{alg:decode} satisfies
\begin{equation}
\mathbb{E}\big[R(S)\big]\;\le\;\mathbb{E}\big[\hat R(S)\big]+\bar\epsilon\,\mathbb{E}\big[|S|\big]\;\le\;\varphi+\bar\epsilon\,\mathbb{E}\big[|S|\big] ,
\label{eq:temp-round-cert}
\end{equation}
and the piece-wise total inconsistency for the whole decode with
budget $\varphi$ satisfies
\begin{equation}
\mathbb{E}\big[R_{\text{total}}\big]\;\le\;\varphi\cdot\mathbb{E}[T]\;+\;\bar\epsilon\cdot N
\;=\;\varphi\cdot\mathrm{NFE}\;+\;\bar\epsilon\cdot N,
\label{eq:temp-int-cert}
\end{equation}
where $N$ is the canvas length ($\sum_r|S_r|=N$, every position is committed exactly once). In
particular, with a well-estimated table ($\bar\epsilon\approx0$), $\varphi$ is an exchange rate:
at most $\varphi$ expected disagreements per forward pass spent, uniformly over the decode.
\end{proposition}

\begin{proof}
Because the walk visits candidates in decreasing $\nu$ and \Cref{eq:aggregate-influence}'s coupling sums only over
higher-confidence members ($\nu_j>\nu_i$), the estimated inconsistency
$b_i=\hat p(\nu_i,\hat\mu_i(S))$ assessed at admission time equals its value under the final
set: later admissions never change
$\hat\mu_i(S)$. For each candidate $i$, the admission event $\{i\in S\}$ is
$\mathcal{G}_i$-measurable, so by the tower property and calibration (Definition \ref{def:temp-calibrated}), with the convention
$F_\theta(i;\mathbf{x}_t,\mathbf{x}_t)\,\mathbf{1}\{i\in S\}=0$ for $i\notin S$,
\begin{align*}
\mathbb{E}\big[R(S)\big] &= \sum_{i}\mathbb{E}\big[F_\theta(i;\mathbf{x}_t,\mathbf{x}_t)\,\mathbf{1}\{i\in S\}\big]\\
&=\sum_{i}\mathbb{E}\big[\mathbb{E}[F_\theta(i;\mathbf{x}_t,\mathbf{x}_t)\mid\mathcal{G}_i]\,\mathbf{1}\{i\in S\}\big]\\
&\le\sum_{i}\mathbb{E}\big[\big(\hat p\big(\nu_i,\hat\mu_i(S)\big)+\bar\epsilon\big)\,\mathbf{1}\{i\in S\}\big]
\;=\;\mathbb{E}\big[\hat R(S)\big]+\bar\epsilon\,\mathbb{E}\big[|S|\big],
\end{align*}
the inequality by \Cref{eq:temp-cal}; the stopping rule of \PIGint{} enforces
$\hat R(S)\le\varphi$ path-wise, which gives \Cref{eq:temp-round-cert}. For
\Cref{eq:temp-int-cert}, apply \Cref{eq:temp-round-cert} to each round and sum:
$\sum_r|S_r|=N$ deterministically, the decode is Markov in the canvas, and $T$ is a stopping
time of the round sequence, so optional stopping applies to the sum of the per-round bounds.
\end{proof}
\begin{remark}
    We remark that in \cref{eq:temp-int-cert}, the total number of rounds $\mathbb{E}[T]$ is a function of $\varphi$. In fact, the total number of NFE's over the whole decode is inversely related to $\varphi$.
\end{remark}

\PIGext{} compares two parallel policies against one sequential target: committing a
window now, alongside the set the walk holds, or committing it one round later, after
that set has been revealed. The disagreement of the second policy is the following object.

Fix a canvas $\mathbf{x}_t$, a set $S\subseteq M$ and its window $N(S)$, the $h$ most
confident masked positions outside $S$. For $i\in N(S)$ the \emph{deferred
disagreement} is $F_\theta(i;\mathbf{x}_t,\mathbf{x}_t\oplus\hat x_S)$,
the event that the parallel commit at $i$, re-predicted once the parallel commits
$\hat{x}_S$ are revealed, differs from the target. Its value at $i$ depends on the window only through the members more
confident than $i$. The \emph{deferred inconsistency} of the window is
$R(N(S)\mid S)=\sum_{i\in N(S)}F_\theta(i;\mathbf{x}_t,\mathbf{x}_t\oplus\hat x_S)$. 

\begin{definition}[$(\bar\epsilon,\bar\delta)$-calibrated]
\label{def:temp-deferred}
An inconsistency estimator $\hat p$ is evaluated for the deferred round at a
$\mathcal{G}_i$-measurable summary statistic $\mathcal{T}'_i$ formed from the current
forward pass; we write $\hat p(\mathcal{T}_i')$ for this \emph{deferred
estimate} (an estimate formed now, not a quantity observed after $S$ is committed) and
$\hat R(N(S)\mid S)=\sum_{i\in N(S)}\hat p(\mathcal{T}_i')$. The estimator is called
\emph{$(\bar\epsilon,\bar\delta)$-calibrated} if it is $\bar\epsilon$-calibrated
(Definition \ref{def:temp-calibrated}) and its deferred estimate satisfies, for every canvas~-~set
pair $(\mathbf{x}_t,S)$ and every $i\in N(S)$,
\begin{equation}
\Big|\,\mathbb{E}_{\mathcal{D}}\big[F_\theta(i;\mathbf{x}_t,\mathbf{x}_t\oplus\hat x_S)\;\big|\;\mathcal{G}_i\big]\;-\;\hat p\big(\mathcal{T}_i'\big)\,\Big|\;\le\;\bar\epsilon+\bar\delta ,
\label{eq:temp-cal-deferred}
\end{equation}
where $\bar\delta\ge0$ is the excess error of the deferred estimate.
\end{definition}

\begin{lemma}[\PIGext{}: the marginal inconsistency]
\label{lem:temp-ext}
Let $\hat p$ be an $(\bar\epsilon,\bar\delta)$-calibrated inconsistency estimator
(Definition \ref{def:temp-deferred}). Fix a round at canvas $\mathbf{x}_t$, a set $S$ reached by the
walk, and its window $N(S)$. Then widening the round by $N(S)$, instead of granting $N(S)$ one extra forward pass, costs

\begin{equation}
\mathbb{E}\big[R(S\cup N(S))\big]-\mathbb{E}\big[R(S)+R(N(S)\mid S)\big]\;\le\;\theta+|N(S)|\,(2\bar\epsilon+\bar\delta).
\label{eq:temp-ext-window}
\end{equation}
\end{lemma}

\begin{proof}
For $i\in N(S)$ write $\mathcal{T}_i$ for the statistic of $i$ as a member of the widened
round $S\cup N(S)$ (Definition \ref{def:temp-calibrated}), so that
$\hat R(N(S))=\sum_{i\in N(S)}\hat p(\mathcal{T}_i)$ and
$\hat R(N(S)\mid S)=\sum_{i\in N(S)}\hat p(\mathcal{T}'_i)$ are the two summands of \PIGext{}. Calibration at the pair $(\mathbf{x}_t,S\cup N(S))$
(\Cref{eq:temp-cal}) places
$\mathbb{E}_{\mathcal{D}}[F_\theta(i;\mathbf{x}_t,\mathbf{x}_t)\mid\mathcal{G}_i]$ within
$\bar\epsilon$ of $\hat p(\mathcal{T}_i)$, and deferred calibration
(\Cref{eq:temp-cal-deferred}) places
$\mathbb{E}_{\mathcal{D}}[F_\theta(i;\mathbf{x}_t,\mathbf{x}_t\oplus\hat x_S)\mid\mathcal{G}_i]$ within
$\bar\epsilon+\bar\delta$ of $\hat p(\mathcal{T}'_i)$; subtracting,
\begin{equation}
\Big|\,\mathbb{E}_{\mathcal{D}}\big[F_\theta(i;\mathbf{x}_t,\mathbf{x}_t)-F_\theta(i;\mathbf{x}_t,\mathbf{x}_t\oplus\hat x_S)\;\big|\;\mathcal{G}_i\big]
\;-\;\big(\hat p(\mathcal{T}_i)-\hat p(\mathcal{T}'_i)\big)\Big|\;\le\;2\bar\epsilon+\bar\delta ,
\label{eq:temp-ext-token}
\end{equation}
the per-token form of the bound, reused in the proof of Proposition \ref{thm:temp-ext-cascade}. The events $\{i\in N(S)\}$ are $\mathcal{G}_i$-measurable, so the tower property and
\Cref{eq:temp-ext-token} bound it by
$\mathbb{E}\big[\hat R(N(S))-\hat R(N(S)\mid S)\big]+|N(S)|\,(2\bar\epsilon+\bar\delta)$.
Given the stopping rule of \Cref{alg:pig-ext}, we know $\Big\vert \hat p(\mathcal{T}_i)-\hat p(\mathcal{T}'_i) \Big\vert\leq \theta$ and the claim follows.
\end{proof}

To aggregate Lemma \ref{lem:temp-ext} over a decode we apply Definition \ref{def:temp-deferred} one admission at a time within each round,
against the round's own target, with the parallel walk one deferral step deep.

\begin{proposition}[\PIGext{}: piece-wise deferred parallel inconsistency over a whole decode]
\label{thm:temp-ext-cascade}
Order each round's admissions $s_1,\dots,s_{|S_r|}$ in decreasing confidence with prefixes
$S_{r,k}=\{s_1,\dots,s_k\}$, and define the round's \emph{cascade floor}
\begin{equation}
C(S_r)\;=\;\sum_{k=2}^{|S_r|} F_\theta\big(s_k;\ \mathbf{x}_{t_r},\ \mathbf{x}_{t_r}\oplus\hat{x}_{S_{r,k-1}}\big),
\qquad C_{\text{total}}=\sum_{r=1}^{T}C(S_r),
\label{eq:temp-cascade-floor}
\end{equation}
the disagreements its admissions would still suffer, against the round's own target, if
each were re-predicted after its prefix. Let $\hat p$ be an
$(\bar\epsilon,\bar\delta)$-calibrated inconsistency estimator whose per-token deferral
terms are nonnegative. Then, \Cref{alg:pig-ext} satisfies
\begin{equation}
\mathbb{E}\big[R_{\text{total}}\big]\;\le\;\mathbb{E}\big[C_{\text{total}}\big]\;+\;\big(\theta+2\bar\epsilon+\bar\delta\big)\big(N-\mathrm{NFE}\big).
\label{eq:temp-ext-cert}
\end{equation}
\end{proposition}

\begin{proof}
Fix a round and drop $r$. In the notation of \Cref{eq:temp-flip} the disagreement at $s_k$ does not
record which round contained it, so
$R(S)=\sum_{k=1}^{|S|}F_\theta(s_k;\mathbf{x}_t,\mathbf{x}_t)$; the $k=1$ term vanishes, the first
token of a round never being in disagreement, $x^*_{s_1}=\hat{x}_{s_1}$ by
\Cref{eq:background-seq}. For $k\ge2$, $s_k$ is the most confident member of
$N(S_{k-1})$, and its statistics $\mathcal{T}_{s_k},\mathcal{T}'_{s_k}$ are those
of the pair $(S_{k-1},N(S_{k-1}))$. The admission event $\{s_k\in S\}$ is
$\mathcal{G}_{s_k}$-measurable, so multiplying \Cref{eq:temp-ext-token}, taken at
$(S,i)=(S_{k-1},s_k)$, by its indicator and taking expectations,
\begin{align*}
\mathbb{E}\big[F_\theta(s_k;\mathbf{x}_t,\mathbf{x}_t)\,\mathbf{1}\{s_k\in S\}\big]
&\le\;\mathbb{E}\big[F_\theta(s_k;\mathbf{x}_t,\mathbf{x}_t\oplus\hat x_{S_{k-1}})\,\mathbf{1}\{s_k\in S\}\big]\\
&\quad+\;\mathbb{E}\Big[\big(\hat p(\mathcal{T}_{s_k})-\hat p(\mathcal{T}'_{s_k})\big)\,\mathbf{1}\{s_k\in S\}\Big]
\;+\;(2\bar\epsilon+\bar\delta)\,P(s_k\in S).
\end{align*}
The bracket is the term of $s_k$ in the estimated gap of $S_{k-1}$; the remaining terms of
that gap are nonnegative by hypothesis, so on $\{s_k\in S\}$ it is at most
$\hat R(N(S_{k-1}))-\hat R(N(S_{k-1})\mid S_{k-1})\le\theta$ by \Cref{alg:pig-ext}'s admission rule.
Summing over $k\ge2$,
$\mathbb{E}[R(S)]\le\mathbb{E}[C(S)]+(\theta+2\bar\epsilon+\bar\delta)\,\mathbb{E}\big[|S|-1\big]$;
summing over rounds, $\sum_r(|S_r|-1)=N-T$, and optional stopping applies as in
Proposition \ref{prop:temp-int}.
\end{proof}

\clearpage

\section{The Entropy Finite Differences Estimator: Implementation and Training}
\label{app:estimator}

This appendix is self-contained: it specifies what the estimator of \Cref{fig:adapter} consumes,
what it is made of, what it is trained on, and how it was trained, for the LLaDA-8B-Instruct
instantiation used in every LLaDA experiment of this paper; the Nemotron-Diffusion-3B Instruct twin is
described at the end by its differences. Every number is read from the harvest and training code and
the shipped checkpoint, not from a design document.

\subsection{The finite-difference estimator}
\label{app:estimator:fde}

\paragraph{What the estimator predicts.}
For a canvas state $y_t$ with masked positions $\mathcal{M}$, the target is the directed
\emph{entropy finite difference}
\begin{equation}
\hat{\mathbf{D}}_{ij}\;\approx\;H_i(y_t)\;-\;H_i\big(y_t\ \text{with position } j \text{ committed to its argmax}\big),
\qquad i\in\mathcal{M},\ j\in\mathcal{M},\ i\neq j,
\label{eq:app-delta}
\end{equation}
the change in the model's own entropy at receiver $i$ when source $j$ is committed --- positive when
committing $j$ makes $i$ more certain, negative when it makes $i$ less certain. A single forward pass
of the frozen model with $j$ committed measures one full column $\delta_{\cdot j}$ exactly; the
estimator's job is to return the whole matrix from the \emph{one} forward pass the decoder already
runs on $y_t$, so that the two set statistics of \Cref{alg:decode} --- the aggregate coupling
$\mu_i(P)$ and the local curvature $\hat H(P)$ --- cost no extra model evaluations.

\paragraph{Inputs (per canvas position, from the frozen backbone).}
The backbone is never trained, never grad-enabled and never modified: it runs in evaluation mode
under \texttt{no\_grad}, and a forward hook on the final normalisation layer (the input to the
language-model head) reads the post-norm hidden state $h_i\in\mathbb{R}^{4096}$ for every canvas
position. The estimator also receives the canvas token id at each position (the mask token at masked
positions), the two attention maps $A_t, A_{t-1}\in\mathbb{R}^{256\times256}$ of the current and
previous round, and nine per-position scalars: the entropy $H_i$ and confidence $\nu_i$ of the
position's predictive distribution, the top-1/top-2 probability margin, a decoded flag, a three-way
structural class (decoded / committable / other) as a one-hot, and the decode fraction $t$. The
canvas length is fixed at $L=256$ throughout.

\paragraph{Architecture (${\approx}2.2$M parameters).}
Per position, three linear maps embed the inputs --- $h_i$ through $4096\!\to\!192$, the token
embedding through $4096\!\to\!64$, the nine scalars through $9\!\to\!32$ --- and a fused linear
layer with LayerNorm produces $z_i^{(0)}\in\mathbb{R}^{192}$. Two pre-norm transformer layers over
the $256$ positions follow ($d=192$, six heads of dimension $32$, rotary relative positions, GELU
feed-forward of width $768$), each modulated by FiLM conditioning on $t$ (a two-layer MLP maps $t$ to
a per-channel scale and shift). A pairwise \emph{attention-bias} MLP maps, for every ordered pair
$(i,j)$, the five features $[A_t[i,j], A_t[j,i], A_{t-1}[i,j], A_{t-1}[j,i], \log(1+|i-j|)]$ through
$5\!\to\!16\!\to\!7$: six outputs are added to the six heads' attention logits, the seventh is an
additive edge bias $b_{ij}$ on the readout. The readout is bilinear and directed: a receiver map
$f:\mathbb{R}^{192}\!\to\!\mathbb{R}^{96}$ and a source map $g:\mathbb{R}^{192}\!\to\!\mathbb{R}^{96}$
(each a two-layer GELU MLP) and a learned $W\in\mathbb{R}^{96\times96}$ give
\begin{equation}
\hat{\mathbf{D}}_{ij}\;=\;\frac{f(z_i)^{\top}\,W\,g(z_j)}{\sqrt{96}}\;+\;b_{ij},
\label{eq:app-readout}
\end{equation}
computed for all pairs at once as one $256\times256$ matrix product. Two different maps for the two
roles are what make $\hat{\mathbf{D}}$ directed ($\hat{\mathbf{D}}_{ij}\neq\hat{\mathbf{D}}_{ji}$); the $1/\sqrt{96}$
is the usual dot-product normalisation that keeps the readout at unit scale at initialisation and has
no expressive role.

\paragraph{Cost as a function of canvas length.}
With $L$ canvas positions the per-position projections cost $O(L)$ ($\approx L\times1.05$M
multiply-adds, dominated by the $4096\to192$ and $4096\to64$ maps), the two transformer layers
$O(L\,d^2 + L^2 d)$ with $d=192$, and the pairwise terms --- the attention-bias MLP over all
$(i,j)$, the bilinear readout $F W G^{\top}$ and the edge bias --- $O(L^2)$ with small constants
($\approx 192 + 96$ multiply-adds per pair). At $L=256$ the total is ${\approx}0.6$\,GMAC
(${\approx}1.2$\,GFLOP) per round, of which the $L$-linear projections are about half; the $L^2$
terms overtake them only beyond $L\approx10^3$. Memory is $O(L^2)$ for the two attention maps and
the bias tensor. Measured in the production decoder on a shared H100 (the node was not reserved,
so this is an upper bound), the estimator adds $12.8$\,ms per decoding round against $17.9$\,ms
for the attention-input U-Net it replaced, inside the $15$\,ms per-round budget the decoder is
held to.

\paragraph{Training data: states.}
States are harvested by decoding \emph{training-side} problems with the frozen model and
snapshotting the canvas every $24$ commits, so that the first snapshot of a decode is taken at
$t=24/256\approx0.09$ and the last near $t\approx0.94$. Four corpora, all
training splits: GSM8K, MATH (the Hendrycks training set), OpenCoder, and MBPP-train; no test set
of any benchmark enters the harvest. Two decoding policies generate the states: the model's own
one-token-per-round confidence-greedy decode, and confidence-greedy-to-budget set decodes at
$\theta\in\{25,50\}$, each on the free canvas and with eligibility restricted to the leftmost $64$
positions --- so the estimator sees the states that budgeted parallel decoding actually visits, in
both block regimes. In total $34{,}142$ states over $3{,}721$ distinct problems, carrying
${\approx}111$M labelled directed edges; the mixture by corpus, in problems, states and edges:

\begin{center}
\small
\setlength{\tabcolsep}{5pt}
\begin{tabular}{lrrrrrrr}
\toprule
corpus & problems & \% & greedy singles & budgeted sets & states & \% & edges (\%) \\
\midrule
GSM8K (train)  & 1{,}111 & 29.9 & 6{,}750 & 3{,}162 & 9{,}912  & 29.0 & 32.4M (29.2) \\
MATH (train)   & 1{,}230 & 33.1 & 6{,}750 & 4{,}276 & 11{,}026 & 32.3 & 35.7M (32.2) \\
OpenCoder      & 1{,}110 & 29.8 & 6{,}750 & 3{,}233 & 9{,}983  & 29.2 & 32.3M (29.1) \\
MBPP (train)   & 270     & 7.3  & 1{,}620 & 1{,}601 & 3{,}221  & 9.4  & 10.4M (9.4) \\
\midrule
total          & 3{,}721 & 100  & 21{,}870 & 12{,}272 & 34{,}142 & 100 & 111M (100) \\
\bottomrule
\end{tabular}
\end{center}

\noindent Edge counts are $|\text{eligible receivers}|\times24$ per state and vary with $t$; the
percentages show that the mixture is balanced in problems, states and edges alike except for MBPP.
No state with a fully masked canvas ($t=0$) is in the training set: the first round of every decode
--- prompt plus an all-mask span --- is outside the harvested range, and the estimator extrapolates
to it through its $t$-conditioning. The held-out evaluation of \Cref{fig:estimator-quality-detail} does
include $t=0$ states (its first bin), where the estimator tracks the gold local curvature as it
does elsewhere; adding a $t=0$ snapshot to the harvest is a one-line change we have not made.

Splits are by \emph{problem}, never by state: the hash of (corpus, problem id) modulo $10$ assigns
$0$--$7$ to training, $8$ to validation and $9$ to an internal test split ($27{,}231$ / $3{,}468$ /
$3{,}443$ states), so no problem contributes states to two splits. The MBPP arm is about a quarter
the size of the other three, so the code side of the mixture is dominated by OpenCoder.

\paragraph{Training data: labels.}
Each state costs about $28$ additional forward passes of the frozen model. Twenty-four
\emph{single plants} --- the $16$ most confident eligible positions and $8$ eligible positions drawn
at random, to break the selection bias of confidence --- each commit one source $j$ and re-run the
model, giving the exact column $\delta_{\cdot j}$ of \Cref{eq:app-delta} at every eligible receiver:
${\approx}232\times24\approx3{,}300$ labelled directed edges per state, ${\approx}111$M in total.
The loss is evaluated \emph{only} on this labelled block --- the eligible receivers of the $24$
planted sources --- never on the full $256\times256$ output: entries whose source was not planted
have no label and receive no gradient, so the estimator is supervised on ${\approx}5\%$ of the
matrix it emits, at every state.
Three \emph{joint plants} of sizes $4$, $8$ and $12$ (confidence-greedy and random sets alternating)
give the gold non-additive curvature $|H(S)|$ of a set; these are used for evaluation and
diagnostics only and never enter the loss. Each state is stored once in fp16 (${\approx}2.3$\,MB).

\paragraph{Objective and optimisation.}
The shipped model is the plainest recipe in the training ladder: a Huber loss ($\delta_{\rm
Huber}=0.2$, comparable to the typical label magnitude of $0.1$--$0.2$ nats) on the labelled edges
only, unweighted, with no aggregate-mass term and no auxiliary head. AdamW with learning rate
$3\times10^{-4}$ (warm-up, then cosine decay to $3\times10^{-5}$), weight decay $0.01$, gradient
clipping at $1.0$, gradient accumulation over $16$ states per step, ten epochs configured; the
model is selected by validation per-edge correlation at the end of every epoch and training stops
when it fails to improve. The shipped checkpoint is the epoch-$7$ model, validation per-edge
correlation $0.827$. Held-out quality on the four \emph{test} benchmarks --- per-edge, aggregate
coupling and local curvature --- is reported in \Cref{fig:estimator-quality}
(\Cref{fig:estimator-quality-detail} shows the same evaluation as gold-versus-estimate scatter and
as recovered edge mass by magnitude).

\begin{figure}[t]
  \centering
  \includegraphics[width=\textwidth]{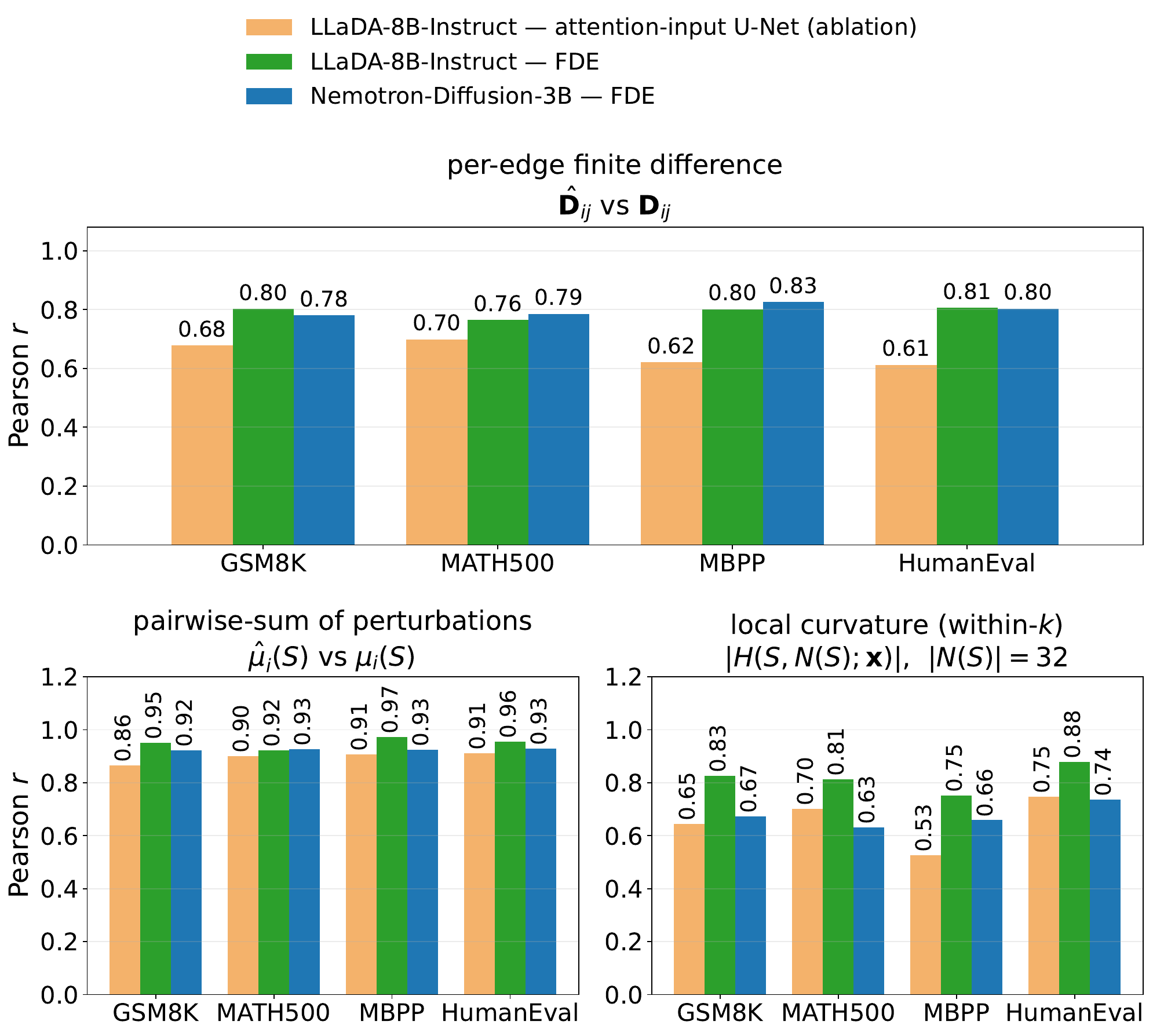}
  \caption{\textbf{Estimator quality on held-out test states.} Pearson correlation between the
  exact quantity and its estimate, per benchmark, for the three quantities the decoders consume:
  the per-edge finite difference $\hat{\mathbf{D}}_{ij}$ against $\mathbf{D}_{ij}$ (left), the
  aggregate coupling $\hat\mu_i(S)$ along the confidence walk (middle), and the local curvature
  $|H(S,N(S);\mathbf{x})|$ over the next $32$ candidates (right; median over set sizes of the
  per-size correlation). Green: the LLaDA-8B-Instruct estimator; orange: the attention-input U-Net
  it replaced (ablation); blue: the Nemotron-Diffusion-3B twin on block-$64$ test states.}
  \label{fig:estimator-quality}
\end{figure}

\begin{figure}[t]
  \centering
  \providecommand{\detailfigscale}{1}%
  \includegraphics[width=\detailfigscale\textwidth]{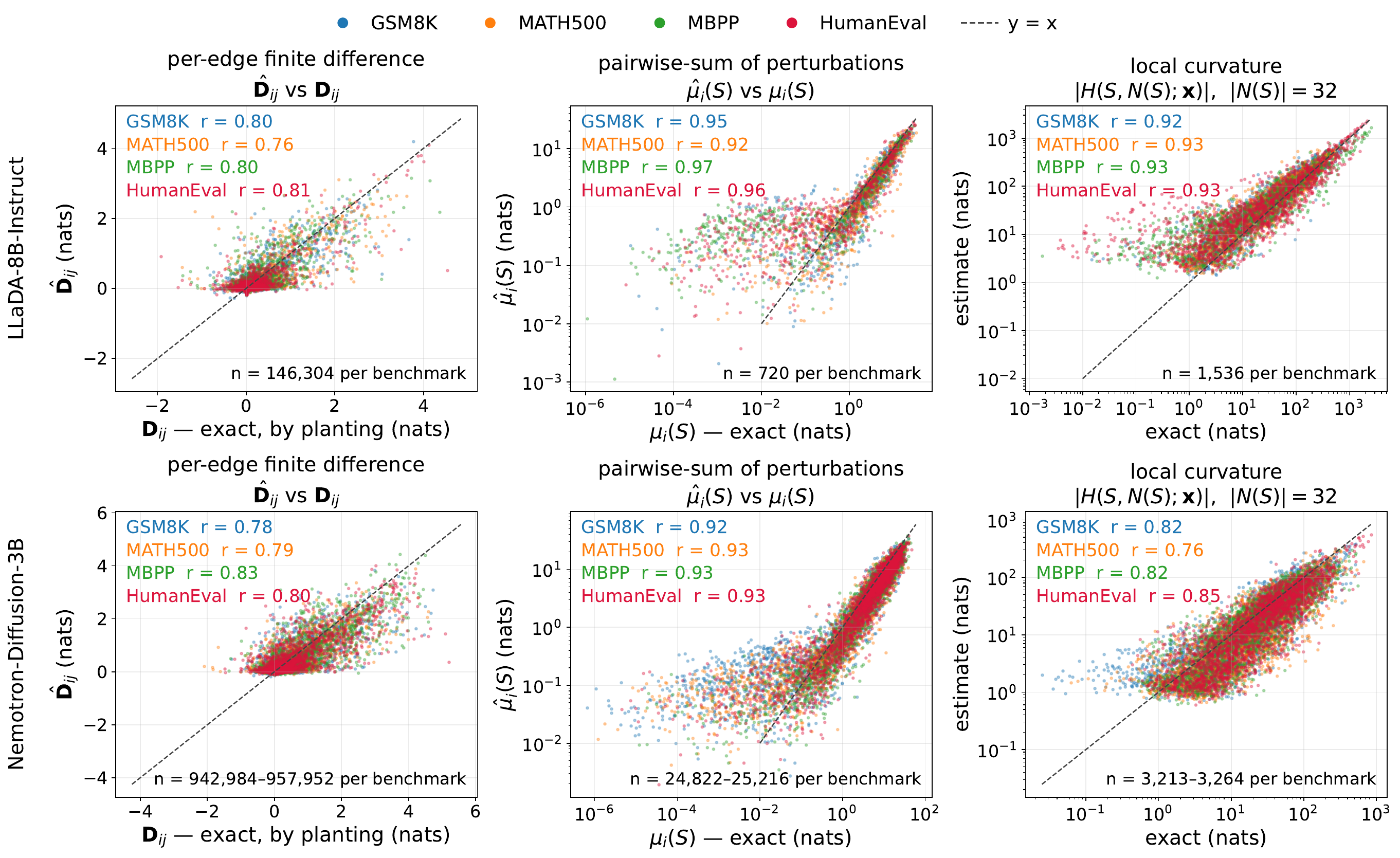}\\[4pt]
  \includegraphics[width=\detailfigscale\textwidth]{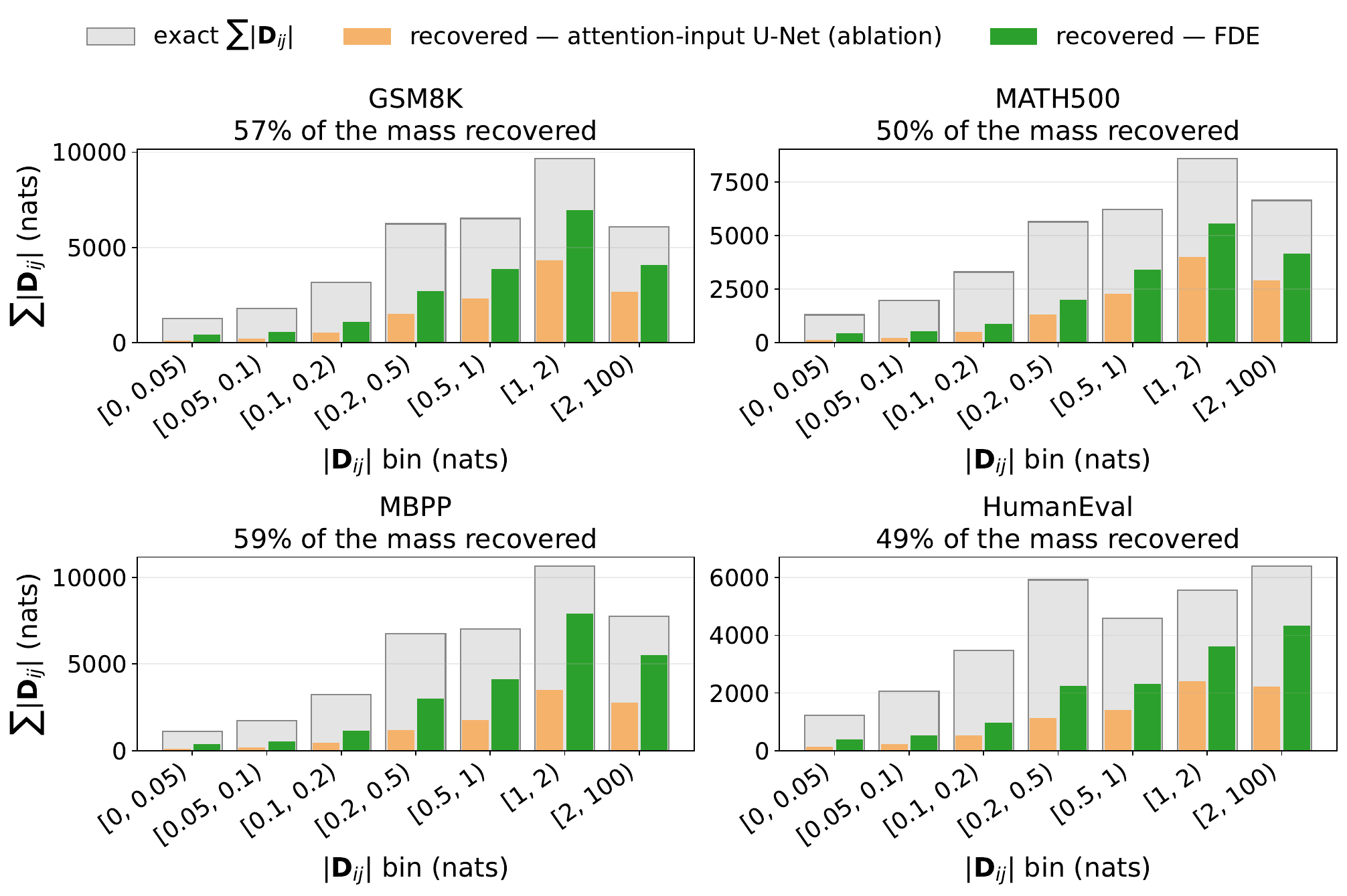}
  \caption{\textbf{The same evaluation in detail.} Top: exact (x) versus estimate (y) on test
  states, one colour per benchmark, LLaDA-8B-Instruct (upper row) and Nemotron-Diffusion-3B (lower
  row), with the per-benchmark correlation printed in each panel; the curvature correlation here
  pools set sizes $|S|\in\{2,\dots,24\}$ and is therefore higher than the within-size value of
  \Cref{fig:estimator-quality}. Bottom: on LLaDA test states, the exact mass
  $\sum|\mathbf{D}_{ij}|$ in each magnitude bin (grey) and the sign-consistent mass the estimator
  recovers (green; $49$--$59\%$ of the total, the U-Net's $26$--$34\%$ in orange): the estimator
  under-shoots in magnitude but does so uniformly across bins.}
  \label{fig:estimator-quality-detail}
\end{figure}

\Cref{alg:estimator-train} states the whole procedure in one place. Its two halves are the two
costs of this estimator. The harvest is the expensive one: every stored state pays one forward pass
to record the model's own readouts and then twenty-four more, one per planted source, each
committing a single position and re-scoring the canvas so that an exact column of the
finite-difference matrix can be read off. Nothing is approximated at this stage; the labels are
differences of entropies the frozen model actually produced. The fit is the cheap one, a Huber
regression of the adapter's output onto that labelled block, and its only subtlety is that the loss
touches only the planted columns, so roughly nineteen twentieths of the matrix the adapter emits is
never supervised and is held to account solely by the held-out evaluation.

\paragraph{The Nemotron-Diffusion-3B twin.}
The same architecture and recipe are re-instantiated for the block-causal
Nemotron-Labs-Diffusion-3B with three differences forced by the model. (i) The hidden state is
$3072$-dimensional and is read from the model's own last-hidden-state output rather than a hook (its
diffusion head is applied to that output, so this reproduces the logits exactly); the two input
projections are re-dimensioned accordingly. (ii) The model never attends across blocks, so the
estimator operates on one block: the transformer layers run over the block's positions with
block-relative rotary positions, the structural class adds a state for the autoregressively seeded
block head, the pairwise bias MLP takes three features, and states are harvested at block lengths
$\{8,32,128\}$ stratified over problems --- block length $64$ is deliberately never harvested and is
the held-out generalisation geometry (it is also the sliding-window operating block). (iii) Plants
re-run only the current block against the cached prefix, so a fully planted state costs roughly
$0.5$\,s instead of ${\approx}5$\,s. Corpora, split rule (by problem), plant recipe ($16$
top-confidence $+$ up to $8$ random sources), Huber-$0.2$ objective and optimiser are unchanged.
The harvest holds $24{,}167$ states; the shipped checkpoint is the best-validation epoch
(epoch $7$), at a validation edge correlation of $0.794$ --- read from the checkpoint's stored
metadata, as for LLaDA.

\subsection{The calibrated mismatch table}
\label{app:estimator:table}
\label{app:results:matrix}

The estimator of \Cref{app:estimator:fde} returns finite differences; the decoder needs
probabilities. Both decoders obtain them from a second, far simpler object: a table of empirical
mismatch rates indexed by two features of \Cref{eq:mismatch-table}, the committing position's
aggregate coupling $\hat\mu$ to the positions already accepted and either its confidence $\nu$,
giving the table $\hat p_\nu$ that \PIGint{} reads (\Cref{fig:flip-matrix}), or its entropy $H$,
giving the table $\hat p_H$ that \PIGext{} reads (\Cref{fig:entropy-table}). Neither table is a
model. Each is a binned average of labelled commits, frozen once per backbone and read by lookup
at decode time. We describe $\hat p_\nu$ first, then the two points on which $\hat p_H$ differs.

\paragraph{Labels and the data mixture.} The labels come from sequential-prefix replay, and from a
\emph{different} corpus mixture than the estimator's. \Cref{alg:table-fit} gives the procedure.
Training-split problems from two corpora only, GSM8K-train and OpenCoder, are decoded one token per
round in decreasing confidence, so that every state the probe visits is one the decoder would
actually reach. Every sixteenth commit the probe stops, builds the production one-pass graph at
that canvas, and assembles about thirty candidate sets from six samplers: confidence-greedy sets
grown until their coupling crosses a cut, top-$K$ and bottom-$K$ by confidence, uniform random
sets, sets grown greedily along the largest edges, and mixtures of the top half by confidence with
a random half. The spread is the point: a table fitted only on the sets a budget would choose would
have no data in the cells a budget must be able to price in order to refuse them. Each set is then
replayed sequentially from that same canvas, and every member yields one row, its confidence, its
aggregate coupling to the members ahead of it, and a binary label for whether the parallel commit
survived the replay. No test split of any benchmark is touched. The LLaDA table is fitted on $792{,}841$ labelled commits, the Nemotron table on
$601{,}730$. The narrower mixture is deliberate: the table has two input features and a few dozen
cells, so it saturates on far less data than the estimator's $111$M labelled edges, and the two
corpora were chosen to bracket the math and code sides of the evaluation.

\paragraph{Binning.} Confidence is cut at the fixed edges $0,\,0.2,\,0.4,\,0.6,\,0.7,\,0.8,\,0.9,
\,0.95,\,1$, eight bins that tighten where admission actually happens. The coupling axis has no
natural scale, so its five bins are cut at the $20$th, $40$th, $60$th and $80$th percentiles of
the \emph{positive} couplings observed in the probe, with the top edge at the $99.9$th percentile
of all of them; a cell holding fewer than $50$ commits falls back to the pooled rate of its
confidence row, which is why no cell is ever read off a handful of observations. The cell value is
the raw mismatch frequency, with no smoothing, shrinkage or parametric fit, so the table carries no
assumption about how risk varies between cells and can express the non-monotone coupling profile we
observe below $\nu=0.6$. Two consequences are worth stating. The coupling edges are quantiles of
\emph{the estimator's own} predicted couplings rather than of gold ones, so the table is matched to
the estimator it will be read beside and must be refitted whenever that estimator changes. And
because the probe advances one token per round while the sets it scores are much wider, the table
is fitted on states a fast decoder reaches but on sets a slow decoder would never commit, which is
what lets a single table price both the budget's acceptances and its refusals.

\paragraph{Why the table is fitted per model.} The coupling edges are model-specific and differ
substantially: $0.26/0.87/2.38/5.91/35.4$ on LLaDA against $0.13/0.76/2.47/5.56/26.7$ on Nemotron,
because Nemotron's edge magnitudes are smaller. Levels differ too, with Nemotron mismatching more
at low confidence ($0.363$ against $0.265$ in the cheapest cell). A table fitted on one backbone
mis-prices the other, so each backbone is decoded with its own; \Cref{fig:flip-matrix} shows both.

\begin{figure}[t]
  \centering
  \includegraphics[width=\textwidth]{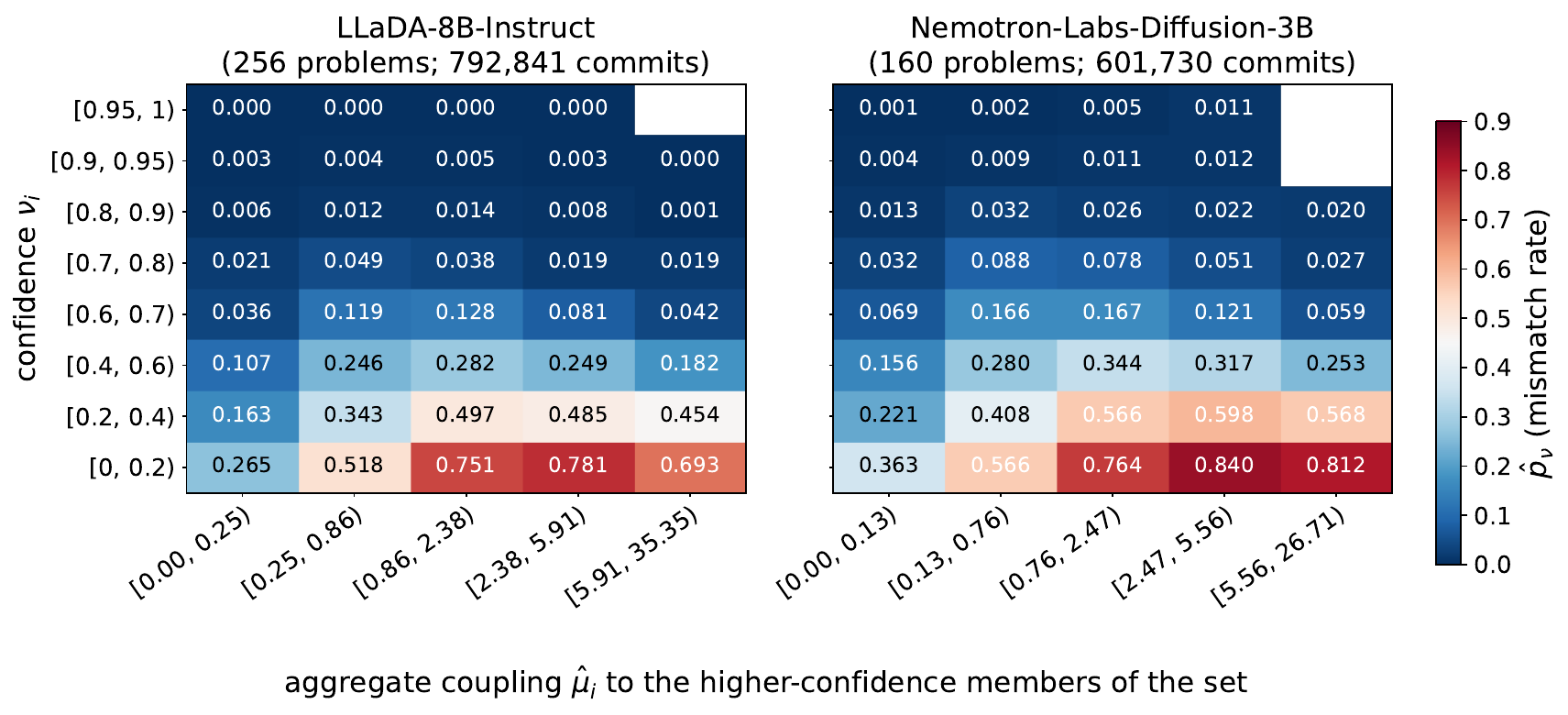}
  \caption{\textbf{The calibrated mismatch table $\hat p_\nu$ read by \PIGint{}}, one per backbone
  (left LLaDA-8B-Instruct, right Nemotron-Labs-Diffusion-3B). Rows: the committing position's
  confidence $\nu_i$ at the fixed edges of \Cref{alg:table-fit}; columns: its aggregate coupling
  $\hat\mu_i$ to the higher-confidence members of the set, at the quantile edges of each model's
  own fit. A cell is the empirical rate at which a parallel commit differed from its sequential
  replay, pooled over the GSM8K-train and OpenCoder probes; a cell holding fewer than $50$ commits
  is left blank and is read by the decoder at the pooled rate of its row.}
  \label{fig:flip-matrix}
\end{figure}

\paragraph{Corpus transfer, and one transfer we assume rather than measure.}
\Cref{fig:flip-matrix-full} shows the two per-corpus fits and their difference. In the decision
region ($\nu\ge0.6$), where a budget actually spends itself, the largest absolute disagreement
between the GSM8K-train and OpenCoder fits is $0.024$ on LLaDA and $0.028$ on Nemotron, which is
why the table we use pools them. Differences above $0.05$ occur only at $\nu<0.4$, where math
tokens mismatch more than code tokens at matched coupling, reaching $-0.18$ in the LLaDA
$[0,0.2)\times[0.25,0.86)$ cell. With $0.3$--$0.6$M commits per corpus even a difference of $0.005$
is statistically significant, so the agreement above is a statement about size, not about noise.
Regime transfer is a weaker claim: the LLaDA table was fitted on free-canvas states and the
Nemotron table on a sliding $32$-token window, and both are then applied at every block geometry we
report. We assume that transfer; we have not measured it.

\begin{figure}[t]
  \centering
  \includegraphics[width=\textwidth]{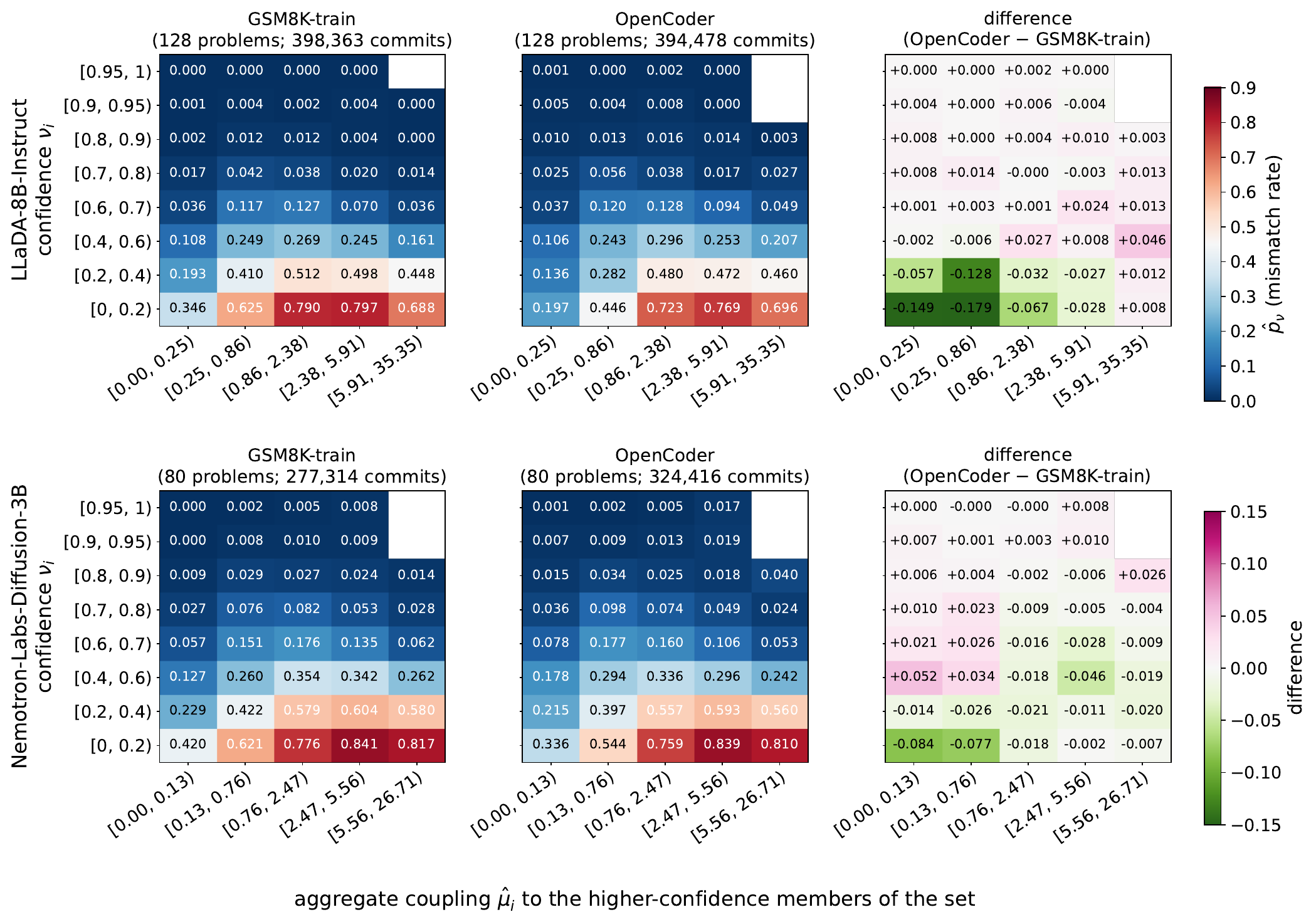}
  \caption{\textbf{Corpus transfer of $\hat p_\nu$.} The table of \Cref{fig:flip-matrix} fitted on
  each training corpus alone, GSM8K-train (left) and OpenCoder (middle), at the pooled bin edges,
  and their difference (right; green where code mismatches less than math at matched coupling,
  magenta where more). Top row LLaDA-8B-Instruct, bottom row Nemotron-Labs-Diffusion-3B. A cell is
  drawn only where both corpora hold at least $50$ commits.}
  \label{fig:flip-matrix-full}
\end{figure}

\paragraph{The entropy-indexed table $\hat p_H$.}
\PIGext{} prices each upcoming candidate $i\in N(S)$ twice, at its current entropy $H_i$ and at
the entropy it is predicted to have once $S$ is committed, $\hat H_i(\mathbf{x}\mid
S)=H_i-\max_{j\in S}\hat{\mathbf D}_{ij}$ (\Cref{eq:future-entropy,eq:predicted-cross-regret}).
The second argument is a prediction on the entropy axis to which no confidence corresponds, so
the lookahead needs the mismatch rate indexed by the entropy of the committing position's
predictive distribution rather than by its confidence. \Cref{fig:entropy-table} shows the
resulting table for both backbones. It is fitted by \Cref{alg:table-fit} with two changes. The row
feature is the receiver's entropy $H_{s_k}$ at the probed canvas, in nats, cut at the fixed edges
$0,\,0.1,\,0.25,\,0.5,\,1,\,1.75,\,2.5,\,3.5,\,4.5,\,6$ and $\infty$, ten bins that tighten at low
entropy where the confident candidates sit (a position at confidence $0.8$ typically has an
entropy near $1$ nat); and the coupling edges are re-derived on the table's own rows by the same
quantile rule, so on LLaDA they differ slightly from those of $\hat p_\nu$
($0.26/0.89/2.45/6.03/36.0$) and on Nemotron they coincide with them. The thin-cell rule is
unchanged. The top row, $H\ge6$ nats, holds no commit on LLaDA and a single one on Nemotron, since
receiver entropies at probe states stay below about $5.5$ nats; it is filled with zeros on LLaDA
and with the row below it on Nemotron, is not read in practice, and is not drawn. The rows are
labelled commits of the same probe. On Nemotron they are the same $601{,}730$ commits: the probe
was re-run recording each member's entropy alongside its confidence, and reproduced the
confidence table's probe exactly. On LLaDA the confidence probe did not record entropy, so
$\hat p_H$ is fitted on a separate probe of the same procedure over $64$ problems per corpus,
$395{,}854$ commits, half the size of the confidence probe.

Two checks. Held out by problem (fitted on the even-numbered problems of each corpus, scored on
the odd), the entropy table is calibrated on both backbones, with a calibration slope of $1.02$
on LLaDA and $1.00$ on Nemotron, at a Brier score of $0.148$ and $0.129$ against $0.137$ and
$0.120$ for the confidence table on the same rows. Confidence is therefore the slightly sharper
single axis for the within-round price, which is why \PIGint{} keeps it, and entropy is the axis
the lookahead requires. Corpus transfer is weaker than for $\hat p_\nu$ and sits elsewhere: the
two per-corpus fits agree within $0.054$ on LLaDA and $0.041$ on Nemotron below $1.75$ nats, but
at high entropy code tokens mismatch less than math tokens at matched coupling, by up to $0.19$
on LLaDA between $2.5$ and $4.5$ nats (and $0.29$ in one cell above $4.5$ nats that holds only
$52$ GSM8K-train commits) and up to $0.17$ on Nemotron above $4.5$ nats. Unlike the confidence
table, whose decisions fall almost entirely where the corpora agree, the lookahead sums over the
next $h$ candidates whatever their entropy, so this disagreement can reach a decision. Every
\PIGext{} result in the paper uses the pooled table, which averages the two corpora at their
commit counts; we have not isolated the effect of the disagreement on the decoded frontier.

\begin{figure}[t]
  \centering
  \includegraphics[width=\textwidth]{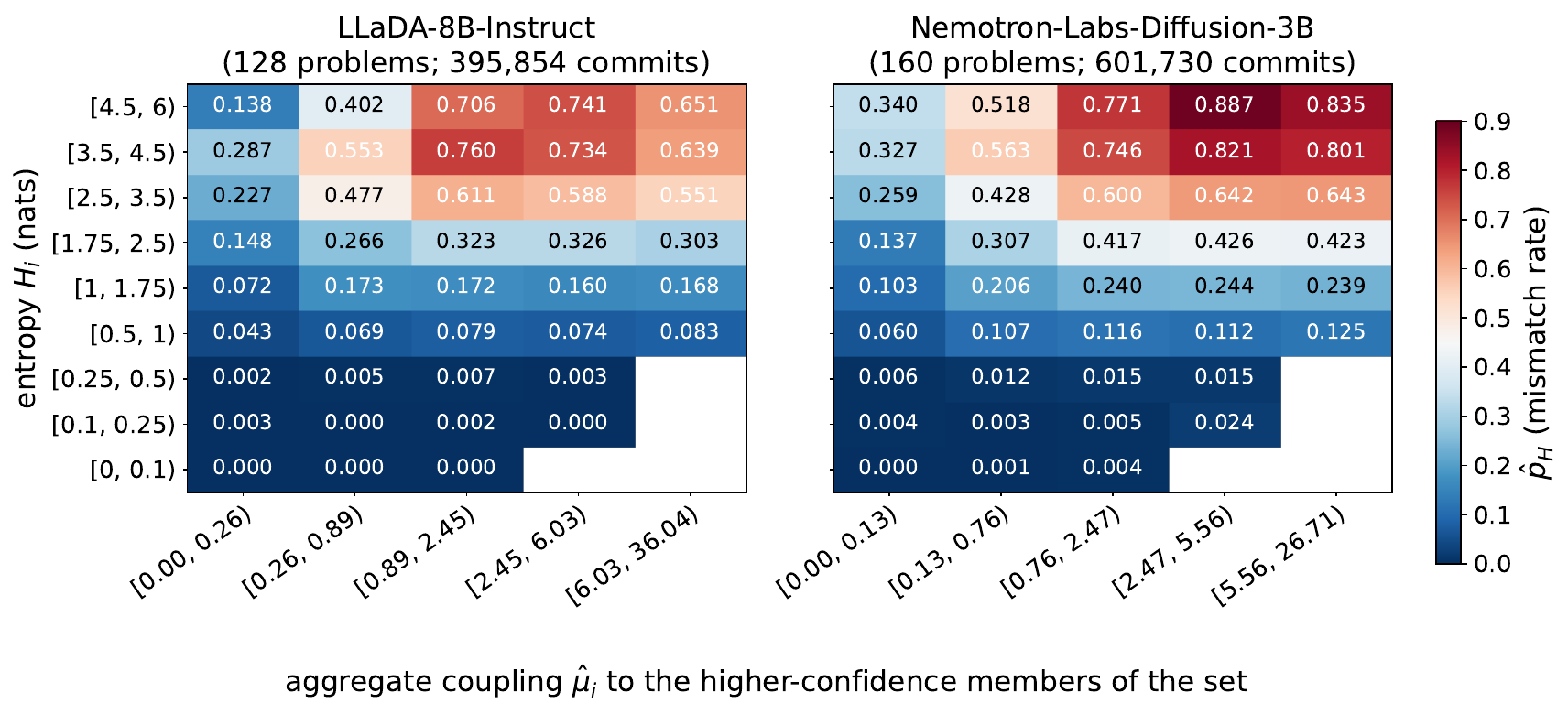}
  \caption{\textbf{The entropy-indexed mismatch table $\hat p_H$ read by \PIGext{}}, one per
  backbone (left LLaDA-8B-Instruct, right Nemotron-Labs-Diffusion-3B). As \Cref{fig:flip-matrix},
  with the committing position's entropy $H_i$ in nats as the row feature, at the fixed edges
  given in the text; the empty top row ($H_i\ge6$) is not drawn. Fitted on the labelled commits of
  the same probe procedure, pooled over GSM8K-train and OpenCoder (Nemotron: the same $601{,}730$
  commits as \Cref{fig:flip-matrix}; LLaDA: a separate probe of $395{,}854$); a cell holding fewer
  than $50$ commits is left blank.}
  \label{fig:entropy-table}
\end{figure}

\clearpage

\section{Algorithms}
\label{app:pig}
\label{app:algorithms}

The three boxes below state the decoders of \Cref{sec:method} in full. \Cref{alg:iterative_diffusion}
is the outer loop shared by every method in the paper: one forward pass per round, a selection rule
that returns the set $S$ to commit, and a parallel commit. \Cref{alg:pig-int} is \PIGint{}, the
selection rule that spends the single-round regret budget $\varphi$ of \Cref{eq:pig-oracle} along the
confidence order; \Cref{alg:pig-ext} is \PIGext{}, the rule that widens the
round until the predicted cross-round
exceeds $\theta$. Both read the same one-pass quantities --- the argmax predictions $\hat x_i$,
confidences $\nu_i$, entropies $H_i$, and the adapter's finite-difference matrix $\hat{\mathbf D}$ of
\Cref{eq:finite-difference} --- and the same calibration tables $\hat p_\nu,\hat p_H$ of
\Cref{eq:mismatch-table}; they differ only in which regret their single knob bounds, and are
alternatives rather than stages of one decoder. Fast-dLLM is \Cref{alg:iterative_diffusion} with the
selection rule $S=\{i\in M:\nu_i\ge\nu\}$.

\begin{algorithm}[h]
\caption{Iterative parallel decoding (outer loop, shared by every decoder).}
\label{alg:iterative_diffusion}
\begin{algorithmic}[1]
\Require prompt tokens, canvas length $N$, selection rule $\textsc{Select}\in\{\text{\PIGint{}}(\varphi),\ \text{\PIGext{}}(\theta),\ \text{Fast-dLLM}(\nu)\}$;
        the adapter and the calibration tables $\hat p_\nu,\hat p_H$, fitted once per backbone
\State $\mathbf{x}_0\gets$ prompt followed by $N$ mask tokens;\quad $M\gets$ its masked positions;\quad $t\gets0$
\While{$M\neq\varnothing$}
    \Statex \hspace{\algorithmicindent}\textbf{Phase 1: one forward pass}
    \State $q_\theta^i(\cdot\mid\mathbf{x}_t)$ for $i\in M$ from the backbone;\quad
           $\hat x_i\gets\arg\max_v q_\theta^i(v\mid\mathbf{x}_t)$,\ $\nu_i\gets\max_v q_\theta^i(v\mid\mathbf{x}_t)$,\ $H_i\gets H\big(q_\theta^i(\cdot\mid\mathbf{x}_t)\big)$
    \State $\hat{\mathbf D}\gets\textsc{Adapter}(\text{backbone features})$ \Comment{one-pass estimate of \Cref{eq:finite-difference}, $\hat{\mathbf D}\in\mathbb{R}^{|M|\times|M|}$}
    \Statex \hspace{\algorithmicindent}\textbf{Phase 2: select the set to commit}
    \State $S\gets\textsc{Select}\big(M,\ \nu,\ H,\ \hat{\mathbf D}\big)$ \Comment{\Cref{alg:pig-int} or \Cref{alg:pig-ext}}
    \Statex \hspace{\algorithmicindent}\textbf{Phase 3: commit in parallel}
    \State $\mathbf{x}_{t+1}\gets\mathbf{x}_t\oplus\hat x_S$;\quad $M\gets M\setminus S$;\quad $t\gets t+1$
\EndWhile
\State \Return $\mathbf{x}_t$ \Comment{$t$ forward passes; TPF $=N/t$}
\end{algorithmic}
\end{algorithm}

\begin{algorithm}[h]
\caption{\PIGint{} --- selection under the single-round regret budget $\varphi$.}
\label{alg:decode}
\label{alg:iterative}
\label{alg:pig-int}
\begin{algorithmic}[1]
\Require masked positions $M$; confidences $\nu_i$; adapter matrix $\hat{\mathbf D}$; budget $\varphi$; table $\hat p_\nu$
\State $\bar{\mathbf D}_{ij}\gets\max(|\hat{\mathbf D}_{ij}|,|\hat{\mathbf D}_{ji}|)$ for $i,j\in M$ \Comment{\Cref{eq:aggregate-influence}}
\State $S\gets\varnothing$;\quad $B\gets\varphi$
\For{$i\in M$ in decreasing $\nu_i$}
  \State $\hat\mu_i(S)\gets\sum_{j\in S}\bar{\mathbf D}_{ij}$ \Comment{aggregate influence of the set already accepted}
  \State $b_i\gets\hat p_\nu\big(\nu_i,\ \hat\mu_i(S)\big)$ \Comment{predicted probability that $\hat x_i\neq x^\star_i$, \Cref{eq:mismatch-table}}
  \If{$b_i\le B$}
    \State $S\gets S\cup\{i\}$;\quad $B\gets B-b_i$ \Comment{the token fits the remaining budget}
  \Else
    \State \textbf{break} \Comment{the round closes at the first candidate that does not fit}
  \EndIf
\EndFor
\State \Return $S$
\end{algorithmic}
\end{algorithm}

\begin{algorithm}[h]
\caption{\PIGext{} --- selection under the cross-round regret budget $\theta$.}
\label{alg:pig-ext}
\begin{algorithmic}[1]
\Require masked positions $M$; confidences $\nu_i$; entropies $H_i$; adapter matrix $\hat{\mathbf D}$; window size $h$; budget $\theta$; table $\hat p_H$
\State $S\gets\varnothing$
\For{$j\in M$ in decreasing $\nu_j$}
  \State $S\gets S\cup\{j\}$ \Comment{admit first: the round widens by the next most confident token}
  \State $N(S)\gets$ the $h$ most confident positions of $M\setminus S$
  \State $\widehat H_i(\mathbf{x}\mid S)\gets H_i-\max_{k\in S}\hat{\mathbf D}_{ik}$ for $i\in N(S)$ \Comment{entropy after committing $S$, \Cref{eq:future-entropy}}
  \State $\hat{C}(S)\gets\sum_{i\in N(S)}\Big[\hat p_H\big(H_i,\hat\mu_i(S\cup N(S))\big)-\hat p_H\big(\widehat H_i(\mathbf{x}\mid S),\hat\mu_i(S\cup N(S))\big)\Big]$ \Comment{\Cref{eq:predicted-cross-regret}}
  \If{$\hat{C}(S)>\theta$}
    \State \textbf{break} \Comment{waiting one round is predicted to remove more than $\theta$ regret on $N(S)$; $j$ stays in $S$}
  \EndIf
\EndFor
\State \Return $S$
\end{algorithmic}
\end{algorithm}

In \Cref{alg:pig-ext} there is no admission test: the statistic closes the round, it never refuses a
token, and the token that first pushes the predicted gain of waiting past $\theta$ stays committed.
Since $N(S)$ is re-read after every admission, near the end of a canvas it is clipped by the pool
boundary, the gain of waiting shrinks, and the walk simply runs out the pool.

The two boxes below are offline and are run once per backbone, before any decoding. They produce
the two objects every decoder above reads: \Cref{alg:estimator-train} fits the adapter that returns
$\hat{\mathbf D}$ in one pass, and \Cref{alg:table-fit} fits the mismatch tables $\hat p_\nu$ and
$\hat p_H$ that convert a confidence, or an entropy, and an aggregate coupling into a probability.
Neither is touched at decode
time. \Cref{app:estimator} discusses both.

\begin{algorithm}[t]
\caption{Fitting the entropy finite differences estimator (offline, once per model).}
\label{alg:estimator-train}
\begin{algorithmic}[1]
\Require frozen dLLM $p_\theta$; training-side problems $\mathcal{D}$; decoding policies $\Pi$
        (confidence-greedy singles; confidence-greedy-to-budget sets at $\theta\in\{25,50\}$, free
        canvas and leftmost-$64$ block); snapshot period $s=24$; sources per state $16$ top-confidence
        $+\,8$ random
\Statex \textbf{Harvest} (states and exact labels)
\For{each problem $q\in\mathcal{D}$ and policy $\pi\in\Pi$}
  \State decode $q$ with $p_\theta$ under $\pi$; every $s$ commits, snapshot the canvas $y$
  \For{each snapshot $y$, with masked positions $C$}
    \State forward $p_\theta(y)$: hidden states $h_i$, token ids, attention maps $A_t,A_{t-1}$,
           scalars $(H_i,\nu_i,\dots,t)$ for all $i$; $H_i(y)$ for $i\in C$
    \State $S\gets$ the $16$ most confident $i\in C$ $\cup$ $8$ further $i\in C$ drawn at random
    \For{each source $j\in S$} \Comment{one extra forward per source}
      \State $y^{(j)}\gets y$ with $j$ committed to $\hat x_j$;\ forward $p_\theta(y^{(j)})$
      \State $\delta_{ij}\gets H_i(y)-H_i(y^{(j)})$ for all $i\in C\setminus\{j\}$
             \Comment{exact column $\delta_{\cdot j}$, \Cref{eq:app-delta}}
    \EndFor
    \State store $\big(h, \text{ids}, A_t, A_{t-1}, \text{scalars}, C, S, \{\delta_{ij}\}\big)$ in fp16
  \EndFor
\EndFor
\State split the stored states by \emph{problem}: $\mathrm{md5}(\text{corpus}{:}q)\bmod 10\in\{0..7\}$
       train, $8$ validation, $9$ test
\Statex \textbf{Fit}
\State initialise the adapter $g_\phi$ (\Cref{eq:app-readout}); AdamW, lr $3\!\times\!10^{-4}$
       warm-up then cosine to $3\!\times\!10^{-5}$, weight decay $0.01$, clip $1.0$
\For{epoch $=1,\dots,10$}
  \For{each training state, in batches of $16$ states per optimiser step}
    \State $\hat\delta\gets g_\phi(h,\text{ids},A_t,A_{t-1},\text{scalars})$ \Comment{one $L\times L$ matrix}
    \State $\mathcal{L}\gets\frac{1}{|C|\,|S|}\sum_{i\in C}\sum_{j\in S,\,j\neq i}
           \mathrm{Huber}_{0.2}\big(\hat\delta_{ij}-\delta_{ij}\big)$
           \Comment{labelled block only; no gradient elsewhere}
    \State gradient step on $\phi$
  \EndFor
  \State evaluate the per-edge correlation $\mathrm{corr}(\hat\delta,\delta)$ on the validation
         states; keep the checkpoint if it improved, stop if it has not
\EndFor
\State \Return the best-validation checkpoint (epoch $7$, $r=0.827$ for LLaDA-8B-Instruct)
\end{algorithmic}
\end{algorithm}

\begin{algorithm}[t]
\caption{Fitting the calibrated mismatch table $\hat p_\nu$ (offline, once per model). The
entropy-indexed table $\hat p_H$ follows the same procedure with each member's entropy $H_{s_k}$
in place of its confidence $\nu_{s_k}$ and the entropy edges of \Cref{app:estimator:table} in
place of $\mathrm{CT}$.}
\label{alg:table-fit}
\begin{algorithmic}[1]
\Require frozen dLLM $q_\theta$; the fitted estimator $g_\phi$ of \Cref{alg:estimator-train};
        training-side problems $\mathcal{D}$ from GSM8K-train and OpenCoder; probe period $16$
        commits; samplers $\Sigma$; thin-cell floor $\tau=50$
\Statex \textbf{Probe} (labelled commits, on the decoder's own state distribution)
\For{each problem $q\in\mathcal{D}$}
  \State decode $q$ committing \emph{one} position per round, the most confident
         \Comment{every visited state is on-distribution}
  \For{every $16$th commit, at canvas $\mathbf{x}$ with eligible positions $M$}
    \State $\hat{\mathbf D}\gets g_\phi(\mathbf{x})$ \Comment{the production one-pass graph}
    \State build ${\approx}30$ candidate sets $S\subseteq M$ with $\Sigma$ (confidence-greedy to a
           coupling cut; top-$K$; bottom-$K$ by confidence; random-$K$; greedy max-edge; mixed),
           deduplicated by membership
    \For{each set $S=(s_1,\dots,s_{|S|})$ in decreasing confidence}
      \State replay $S$ sequentially from $\mathbf{x}$:
             $x^\star_{s_k}=\arg\max_v q_\theta\big(v\mid\mathbf{x}\oplus x^\star_{s_{<k}}\big)$
      \For{each member $s_k\in S$}
        \State emit the row $\big(\nu_{s_k},\ \hat\mu_{s_k}(S_{<k}),\
               \mathbf{1}\{x^\star_{s_k}\neq\hat x_{s_k}\}\big)$,
               $\hat\mu_{s_k}(S_{<k})=\sum_{j\in S_{<k}}\max\big(|\hat{\mathbf D}_{s_kj}|,|\hat{\mathbf D}_{js_k}|\big)$
      \EndFor
    \EndFor
  \EndFor
\EndFor
\Statex \textbf{Fit} (bin, average, freeze)
\State $\mathrm{CT}\gets(0,\,0.2,\,0.4,\,0.6,\,0.7,\,0.8,\,0.9,\,0.95,\,1]$ \Comment{fixed; tighter where admission happens}
\State $\mathrm{MT}\gets\big(0,\ \text{the }20,40,60,80\text{th percentiles of }\{\hat\mu:\hat\mu>0\},\
       \text{the }99.9\text{th percentile of }\hat\mu\,\big]$
\State $N_{ab}\gets$ rows in cell $(a,b)$;\quad $K_{ab}\gets$ mismatches in cell $(a,b)$;\quad
       $P_{ab}\gets K_{ab}/N_{ab}$
\For{each cell with $N_{ab}<\tau$}
  \State $P_{ab}\gets\big(\sum_b K_{ab}\big)\big/\big(\sum_b N_{ab}\big)$
         \Comment{pooled along its confidence row}
\EndFor
\State \Return the frozen table $(\mathrm{CT},\mathrm{MT},P)$
       \Comment{LLaDA: $792{,}841$ rows, $8\times5$ cells}
\end{algorithmic}
\end{algorithm}

\clearpage

\section{Additional Experiments and Analyses}
\label{app:analyses}

This appendix collects the experiments and discussions that support \Cref{sec:experiments} but are
not frontier grids: dynamic block decoding (\Cref{app:analyses:dynamic}), the curvature
analysis and its diagnostic on all four benchmarks (\Cref{app:analyses:curvature}), the
transfer of both decoders to MMLU, a benchmark outside the mathematics and code their learned
components were fitted on (\Cref{app:analyses:domain-transfer}), the audit of both decoders'
regret estimates on the decoder's own trajectories, which measures the calibration premises of
\Cref{app:theory} during live decoding (\Cref{app:analyses:calibration}), the
comparison with the dependency-guided decoders, including
our own reproduction of DAPD and its prompt ablation (\Cref{app:analyses:dapd}), the comparison
with a learned unmasking policy at matched block geometry (\Cref{app:analyses:jazbec}), the
definitions behind the one-number summaries of \Cref{tab:matched-int,tab:matched-ext}
(\Cref{app:results:matched}), and two shorter notes on cross-paper protocol drift and on a
preliminary information-theoretic diagnostic. The complete grids and the tables of every operating
point are in \Cref{app:results}; the protocol under which every lane of this appendix was run is
stated there (\Cref{app:results:llada,app:results:nemotron}).

\subsection{Dynamic block decoding unlocks additional parallelism}
\label{app:analyses:dynamic}
\label{sec:exp-dynamic-blocks}

Fixed-block decoding is a common heuristic for DLM inference, restricting each round to masked positions inside a predefined block.
As decoding progresses through a block, however, the number of remaining masked candidates can shrink even when the decoder could safely commit many more tokens.
We call this effect \emph{candidate starvation}: parallelism becomes limited by the available candidate pool rather than by the decoder's own regret criterion.
We therefore introduce \emph{dynamic block decoding}, which expands the active block as needed to maintain $B$ masked candidates whenever possible.
This changes only candidate eligibility; Fast-dLLM, \PIGint{}, and \PIGext{} otherwise use their original selection rules.

\begin{figure}[H]
    \centering
    \includegraphics[width=0.95\textwidth]{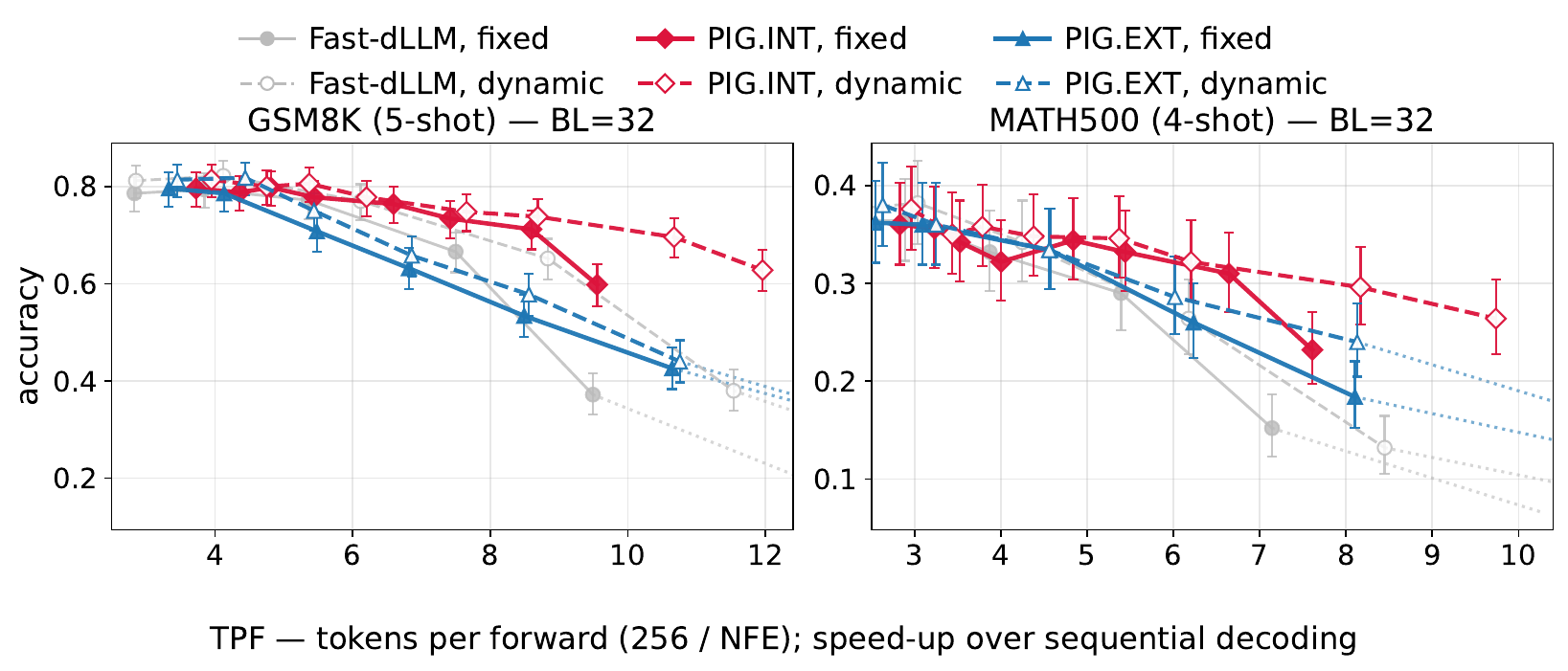}
    \caption{\textbf{Dynamic block decoding unlocks additional parallelism.}
    Accuracy versus TPF on GSM8K and MATH-500 under fixed and dynamic $32$-candidate decoding.
    Filled markers and solid lines denote fixed blocks, while open markers and dashed lines denote dynamic blocks; every decoding rule is evaluated under both candidate geometries.}
    \label{fig:dynamic-blocks}
\end{figure}

Dynamic block decoding improves the accuracy--TPF frontier even for Fast-dLLM, confirming that fixed blocks themselves leave useful parallelism on the table.
The gains are substantially larger for PIG, whose regret-based selection rules can exploit additional low-regret candidates exposed by the larger block.
On GSM8K, \PIGint{} reaches $7.6$ TPF at $0.75$ accuracy (one Wilson interval below Fast-dLLM's slowest point) under dynamic blocks, compared with $7.0$ TPF under fixed blocks, and its sweep now extends past $10$ TPF, where it holds $0.71$ accuracy against $0.54$ for dynamic-block Fast-dLLM.
On MATH-500, \PIGint{} extends from $7.6$ to $9.7$ TPF and holds $0.28$ accuracy at $9$ TPF where its fixed-block sweep has already ended, while \PIGext{} gains $4$ to $5$ points at $7$ to $9$ TPF ($0.27$ against $0.23$, and $0.22$ against $0.17$).
These results suggest that candidate selection and candidate availability are distinct bottlenecks, leaving further opportunities for DLM parallelization by addressing both jointly.

\Cref{tab:dynamic-blocks} reads the fixed and dynamic $32$-token-block sweeps of
\Cref{fig:dynamic-blocks} (LLaDA-8B-Instruct, GSM8K and MATH-500) in the two directions used
throughout the paper: the throughput each lane reaches before it gives up one Wilson interval of
the slowest fixed-block Fast-dLLM point's accuracy, and the accuracy each lane holds at two fixed
throughputs. At matched accuracy the dynamic window helps every lane by a similar factor
($1.12$--$1.13\times$ for Fast-dLLM, $1.04$--$1.08\times$ for \PIGint{}, $1.02$--$1.14\times$ for
\PIGext{}). The difference between the rules appears at the fast end: the dynamic window extends
\PIGint{}'s sweep from $9.6$ to $12.0$ TPF on GSM8K and from $7.6$ to $9.7$ on MATH-500, and at
$10$ and $9$ TPF respectively \PIGint{} holds $0.71$ and $0.28$ accuracy, $17$ points above its
fixed-block extrapolation and $15$--$16$ points above dynamic-block Fast-dLLM, while \PIGext{}
gains $3$--$5$ points at every throughput listed. The larger candidate pool therefore matters most
to the rules that would otherwise run out of admissible candidates inside a block.

\begin{table}[H]
\centering
\caption{\textbf{Fixed versus dynamic $32$-token blocks on LLaDA-8B-Instruct}, the sweeps of \Cref{fig:dynamic-blocks}. $\mathrm{TPF}^{\star}$: largest throughput at which the lane stays within one Wilson 95\% half-width of the slowest fixed-block Fast-dLLM point's accuracy ($a_0$; the matched-accuracy rule of \Cref{tab:matched-int}). acc@$T$: accuracy linearly interpolated at $T$ TPF; $^{\S}$ the sweep ends short of $T$ and is extended along its last segment. Best geometry per lane and column in bold.}
\label{tab:dynamic-blocks}
\footnotesize
\setlength{\tabcolsep}{4pt}
\fitwidth{\begin{tabular}{llrrrrrr}
\toprule
dataset & lane & \multicolumn{2}{c}{$\mathrm{TPF}^{\star}$} & \multicolumn{2}{c}{acc @ $T_1$} & \multicolumn{2}{c}{acc @ $T_2$} \\
\cmidrule(lr){3-4} \cmidrule(lr){5-6} \cmidrule(lr){7-8}
 & & fixed & dynamic & fixed & dynamic & fixed & dynamic \\
\midrule
GSM8K ($a_0{=}0.786$; $T_1{=}8$, $T_2{=}10$) & Fast-dLLM & $5.80$ & $\mathbf{6.58}$ & $0.592$ & $\mathbf{0.688}$ & $0.343$ & $\mathbf{0.535}$ \\
 & PIG.INT & $6.97$ & $\mathbf{7.56}$ & $0.723$ & $\mathbf{0.745}$ & $0.544^{\S}$ & $\mathbf{0.710}$ \\
 & PIG.EXT & $4.75$ & $\mathbf{5.44}$ & $0.563$ & $\mathbf{0.604}$ & $0.458$ & $\mathbf{0.488}$ \\
\midrule
MATH-500 ($a_0{=}0.364$; $T_1{=}7$, $T_2{=}9$) & Fast-dLLM & $4.23$ & $\mathbf{4.75}$ & $0.163$ & $\mathbf{0.216}$ & $0.101$ & $\mathbf{0.122}$ \\
 & PIG.INT & $5.99$ & $\mathbf{6.20}$ & $0.281$ & $\mathbf{0.311}$ & $0.120^{\S}$ & $\mathbf{0.279}$ \\
 & PIG.EXT & $4.84$ & $\mathbf{4.93}$ & $0.229$ & $\mathbf{0.265}$ & $0.167$ & $\mathbf{0.217}$ \\
\bottomrule
\end{tabular}}
\end{table}

\subsection{PIG adapts to the information curvature of decoding}
\label{app:analyses:curvature}
\label{app:results:curvature}
\label{app:temp-curvature}
\label{sec:exp-curvature}

Why can PIG substantially vary its parallel width while maintaining accuracy?
One interpretation is that PIG adapts the size of the parallel decode set $S$ to keep the deviation from the sequential $1$ NFE per token trajectory smallest. We define this deviation through the decode trajectory curvature of second-order finite differences of entropy,
\begin{equation}
    H(S, N(S);\mathbf{x})
    := \sum_{i \in S} \sum_{j \in N(S)} H_j(\mathbf{x}) - H_j(\mathbf{x} \oplus x_i) - \sum_{j \in N(S)} H_j(\mathbf{x}) - H_j(\mathbf{x} \oplus x_S)
    \label{eq:information-curvature}
\end{equation}
Large curvature indicates that committing the current set $S$ substantially changes the distribution over the vocabulary of the upcoming tokens in $N(S)$. Note that curvature is guaranteed to be zero for any set $S$ of cardinality $1$ which corresponds to the sequential oracle. As such, keeping curvature minimal throughout ensures a trajectory that adheres more closely to the sequential $1$ NFE per token sequential oracle.

We measure fixed parallel trajectory's curvature and compare it with the size of $S$ selected by \PIGint{}, \PIGext{}, and fastDLLM.
\Cref{fig:curvature-width} reports the curvature dynamics together with $|S|$ selected by PIG on all four benchmarks.

\begin{figure}[H]
    \centering
    \includegraphics[width=\textwidth]{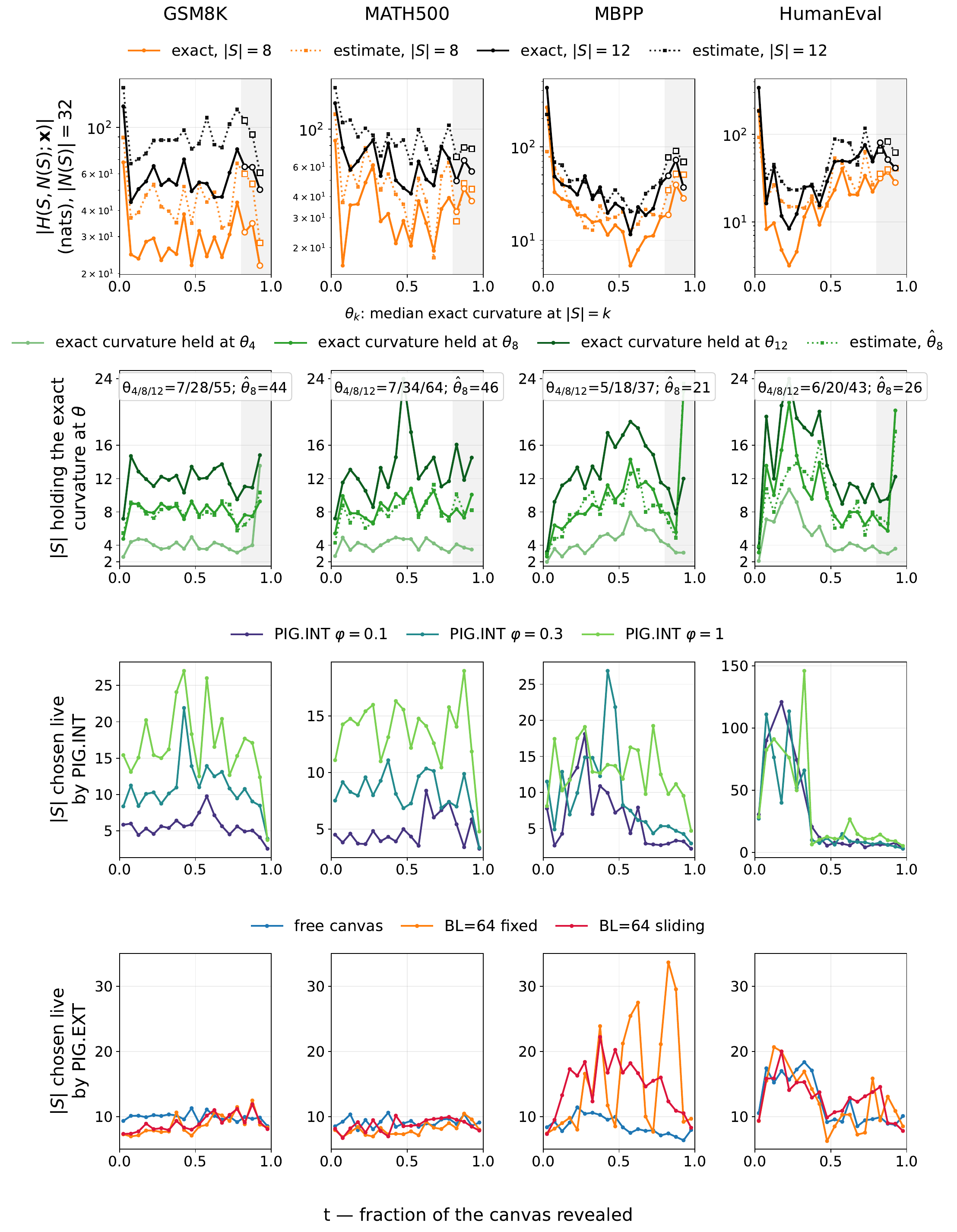}
    \caption{\textbf{PIG adapts its decoding width to information curvature, all four benchmarks.}
    Diagnostic on LLaDA-8B-Instruct across GSM8K, MATH-500, MBPP, and HumanEval.
    The top row compares exact and estimated changes in the uncertainty of upcoming candidates, while the remaining rows show parallel widths induced by fixed regret budgets and live PIG decoding.}
    \label{fig:curvature-width}
\end{figure}

Both PIG variants vary $S$ systematically in line with keeping a constant curvature, taking larger steps in flatter regions and smaller steps as curvature increases.
Confidence-based decoding tracks this structure less closely, since token-wise confidence does not directly measure how committing one set changes the predictions that follow.
The relationship is especially direct for \PIGext{}, whose cross-round regret explicitly measures how much another model evaluation would change the upcoming candidates.

\subsection{Domain transfer: PIG carries over to MMLU without refitting}
\label{app:analyses:domain-transfer}

Everything \PIG{} learns is fitted once per backbone, on mathematics and code alone. The
finite-difference estimator, the adapter of \Cref{fig:adapter}, is trained on decoding states
harvested from the training splits of GSM8K, MATH, OpenCoder and MBPP (\Cref{app:estimator:fde}),
and the calibrated mismatch table on commits replayed from GSM8K-train and OpenCoder
(\Cref{app:estimator:table}). None of these corpora is a multiple-choice benchmark, and none
covers a subject outside mathematics and programming. The frontiers of \Cref{sec:experiments}
therefore leave a question open: whether what these components learn is a property of the
backbone's decoding, which should carry over to any input, or a regularity of the math and code
they were fitted on. We test this on a benchmark that departs from the training mixture in both
format and subject matter.

MMLU \citep{hendrycks2021mmlu} poses four-option multiple-choice questions in $57$ subjects that
span the natural sciences, the humanities, the social sciences and professional fields such as law
and medicine. We draw $500$ of its $14{,}042$ test questions, stratified by subject in proportion
to subject size with a fixed seed, so that every subject appears ($4$ to $52$ questions each).
Only $55$ of the $500$ come from its mathematics and computer-science subjects; the largest single
subject is professional law ($52$ questions), followed by moral scenarios ($32$).
Nemotron-3B-Instruct answers zero-shot under the multiple-choice prompt of the model's released
evaluation harness (NeMo-Skills), which asks for a final line of the form
\texttt{Answer: \textbackslash boxed\{A\}}, and the boxed letter is graded against the key. The
canvas is $2{,}048$ tokens, decoding runs in fixed blocks of $16$, $32$ and $64$ tokens, and TPF
is measured as on every Nemotron panel, as generated tokens before end-of-sequence per forward.
Nothing is refitted or retuned: the runs load the same estimator checkpoint and tables as the math
and code grids, and sweep subsets of the same knob values, Fast-dLLM at
$\nu\in\{0.9,0.7,0.5\}$, \PIGint{} at $\varphi\in\{0.015,0.05,0.1,0.3,0.5\}$, and \PIGext{} under
the same composite rule, its two slowest points ($\theta=0.125$ and $0.25$) floored at confidence
$\nu\ge0.9$ and $\nu\ge0.8$ and the rest plain at $\theta\in\{0.5,1,1.5\}$.

\begin{figure}[H]
  \centering
  \includegraphics[width=\textwidth]{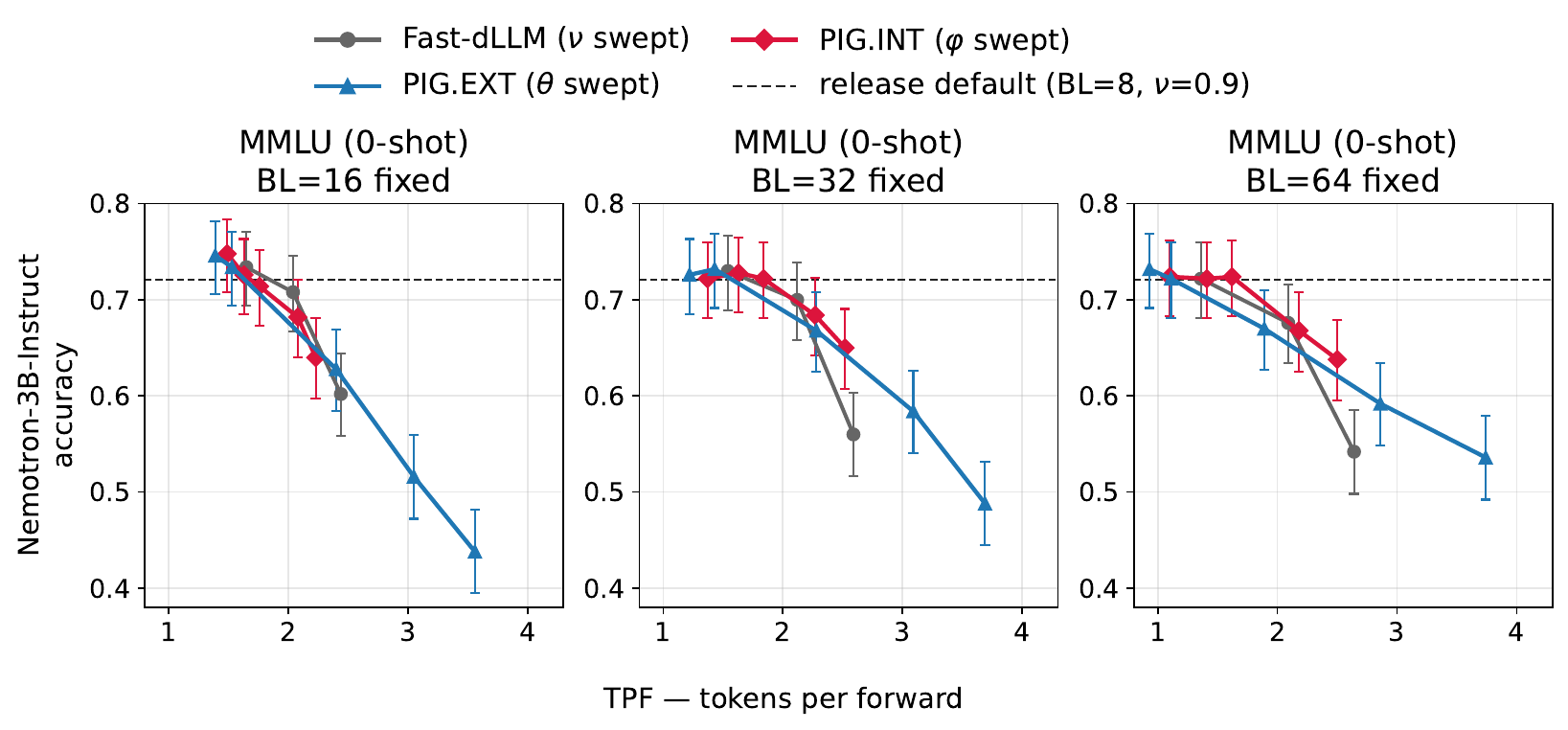}
  \caption{\textbf{Domain transfer: Nemotron-3B-Instruct on MMLU}, zero-shot, $500$ questions
  stratified over all $57$ subjects, at fixed blocks of $16$, $32$ and $64$ tokens (canvas
  $2{,}048$). Grey: Fast-dLLM (on Nemotron, the model's own confidence-threshold decoder, the same
  rule); red: \PIGint{}; blue: \PIGext{}, whose two slowest operating points additionally floor
  admission at a confidence $\nu$. The estimator and tables are the ones fitted on math and code
  (\Cref{app:estimator}), used unchanged. Dashed line: the accuracy the model's release reports at
  its default setting (BL$=$8, $\nu=0.9$) \citep{fu2026nemotron}. Error bars are Wilson $95\%$
  intervals.}
  \label{fig:domain-transfer}
\end{figure}

\Cref{fig:domain-transfer} shows the frontier keeping the shape it has on Nemotron's math and code
benchmarks. At their most conservative settings the three decoders agree, at $0.72$--$0.75$
accuracy for every block size, around the $0.721$ that the model's release reports at its default
setting \citep{fu2026nemotron}: the harness is consistent with the release on a benchmark our
components never saw, and neither \PIG{} variant gives up accuracy at its slow end. As
Fast-dLLM's threshold is lowered to $\nu=0.5$, its accuracy falls by $13$ to $18$ points, to
$0.602$, $0.560$ and $0.542$ at BL$=$16, $32$ and $64$. At BL$=$32 and $64$, \PIGint{} is more
accurate than Fast-dLLM at matched TPF (Fast-dLLM's sweep linearly interpolated at the \PIGint{}
cell's TPF) on every one of its four cells inside Fast-dLLM's range, by $0.3$ to $6.8$ and $0.3$
to $6.2$ points respectively; at the fast end, where \PIGint{} runs within $0.15$ TPF of
Fast-dLLM's fastest point, the two lanes' Wilson intervals no longer overlap ($0.650$ at $2.52$ TPF
against $0.560$ at $2.59$ on BL$=$32, $0.638$ at $2.50$ against $0.542$ at $2.64$ on BL$=$64).
\PIGext{} extends the frontier past Fast-dLLM's fastest point: on BL$=$32 it reaches $0.584$ at
$3.09$ TPF, faster and more accurate than that point, and on BL$=$64 it holds $0.592$ at $2.86$ TPF
and matches Fast-dLLM's fastest accuracy at $1.4\times$ its throughput ($0.536$ at $3.74$ against
$0.542$ at $2.64$). At BL$=$16 the lanes stay within each other's intervals, \PIGint{} $1.2$ to
$1.7$ points below Fast-dLLM on its three cells at matched TPF and \PIGext{} $1.5$ points above at
$2.40$ TPF, which is also how \PIGint{} and Fast-dLLM compare on GSM8K at that block size
(\Cref{fig:nemotron-grid-bl16}).

Neither component is specific to a domain by construction, although nothing guaranteed that it
would transfer. The estimator's target is defined by the backbone alone, as the change in the
model's own entropy at one position when another is committed, and the table indexes an empirical
mismatch rate by two features, the committing position's confidence and its predicted coupling,
whose meaning does not depend on the subject of the question. \Cref{app:estimator:table} found the
table's GSM8K-train and OpenCoder fits within $0.03$ of each other where the budget is spent; MMLU
is a stronger version of that test, a domain neither corpus resembles, read on the decoded frontier
rather than on the table. The evidence is one benchmark on one backbone at $n=500$, which puts the
Wilson half-widths near four points, and every lane stays below $4$ TPF, a narrower band than on
math and code. Within that scope, the three decoders keep the ordering they have on math and code
without any component being refitted, which is what one would expect if what \PIG{} learns is how
the backbone's predictions depend on one another rather than what mathematics and code look like.

\subsection{The adapter is a calibrated predictor of regret}
\label{app:analyses:calibration}

\begin{figure}[H]
  \centering
  \includegraphics[width=\textwidth]{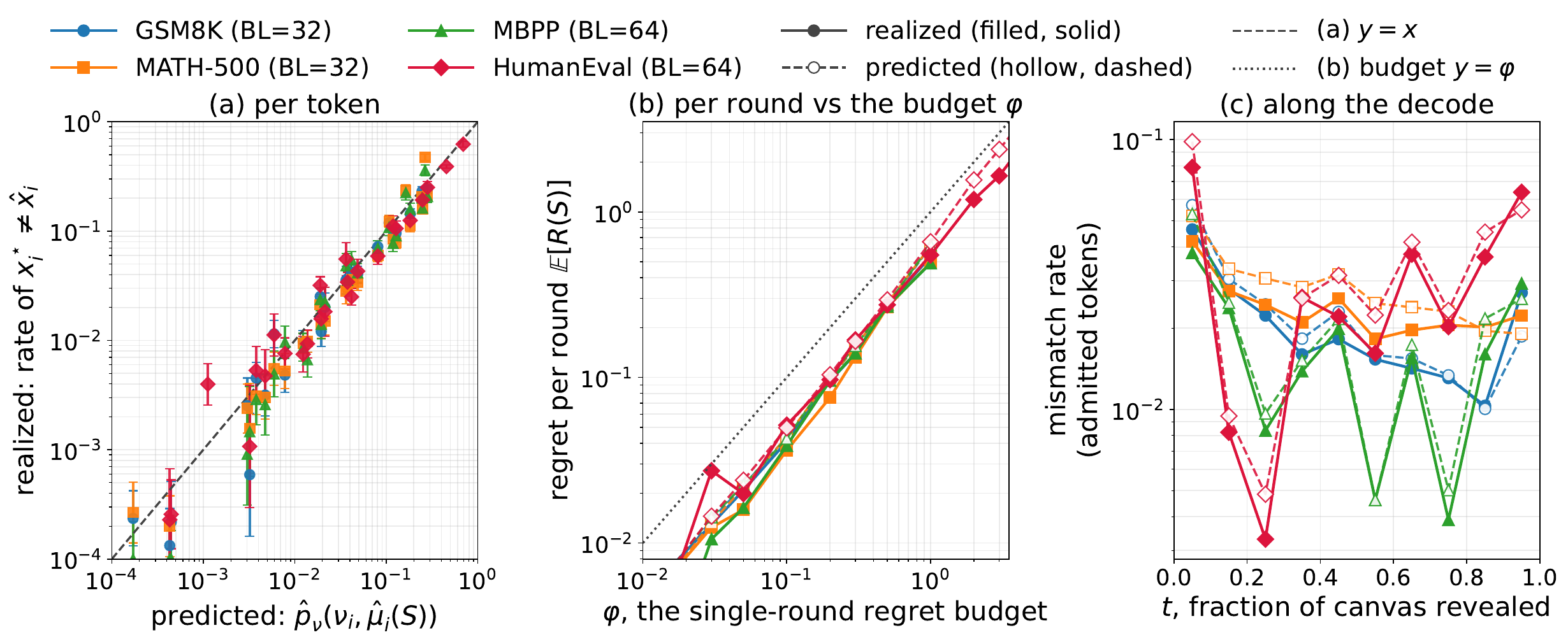}
  \caption{\textbf{The regret calibration on the decoder's own trajectories: \PIGint{} on
  LLaDA-8B-Instruct} at the geometry of \Cref{fig:main-frontier} (GSM8K and MATH-500 at BL$=$32
  fixed, MBPP and HumanEval at BL$=$64 fixed), every third round audited by replaying the committed
  set one token at a time from the round's own canvas in walk order (Definition \ref{def:temp-reference}).
  (a) Realized disagreement
  rate per table cell against the cell's value, pooled over the $\varphi$ ladder (Wilson $95\%$). (b) Per $\varphi$:
  realized regret per audited round $\mathbb{E}[R(S)]$ (filled) and predicted
  $\mathbb{E}[\hat R(S)]$ (hollow) against the budget line $y=\varphi$. (c) Realized (solid) and
  predicted (dashed) mismatch rate per decile of the decode, token-weighted.}
  \label{fig:phat-calibration}
\end{figure}

The guarantees of \Cref{app:theory} rest on one empirical premise: that the mismatch table is
calibrated on the states the decoder actually visits, Definition~\ref{def:temp-calibrated}, from which
Proposition~\ref{prop:temp-int} bounds the realized regret of every \PIGint{} round by its predicted regret
and hence by the budget,
$\mathbb{E}[R(S)]\le\mathbb{E}[\hat R(S)]+\bar\epsilon\,\mathbb{E}|S|\le\varphi+\bar\epsilon\,\mathbb{E}|S|$
(\Cref{eq:temp-round-cert}). The table was fitted offline on replayed probes
(\Cref{app:estimator:table}); whether it stays calibrated during live decoding, on the sets the
budget actually admits, is a separate question, and we measured it on the runs of
\Cref{fig:main-frontier} themselves.

\paragraph{What was measured.} Every \PIGint{} cell of the LLaDA-8B-Instruct frontier at the
geometry of \Cref{fig:main-frontier} (GSM8K and MATH-500 at BL$=$32 fixed, MBPP and HumanEval at
BL$=$64 fixed; $\varphi$ from $0.015$ to $1$, and to $10$ on HumanEval; $36$ cells) ran with an
audit switched on: at every third decoding round (every eighth on the two slowest HumanEval
cells), once the walk had chosen its set $S$, the harness replayed $S$ one token at a time from a
copy of the round's own pre-commit canvas, in walk order, which is exactly the re-anchored
reference of Definition~\ref{def:temp-reference}. For every admitted token after the seed it recorded the
confidence $\nu_i$ and coupling $\hat\mu_i(S_{<i})$ the token was admitted under and the realized
disagreement $\mathbf{1}\{x^\star_i\neq\hat x_i\}$; the predicted value is the table read at those
same two features, the lookup the walk itself performed, so realized and predicted are paired per
token. The decode is untouched by the audit: the parallel set is what is committed, and the replay
forwards are not counted in NFE, so the throughputs and accuracies of \Cref{fig:main-frontier} are
those of the audited runs. The population is the admitted tokens, $754{,}594$ of them over the
$36$ cells; calibration conditional on admission is what Proposition~\ref{prop:temp-int} requires, and refused
candidates are not replayed here. The seed token of each round is consistent by construction and
is excluded from both sides.

\paragraph{What we observe.} \Cref{fig:phat-calibration} reads the audit three ways.
\emph{Per token} (a): grouping the admitted tokens by the table cell they were priced in, the
realized mismatch rate follows the table's value over three decades. Token-weighted over each
benchmark's whole ladder, realized over predicted is $0.92$ on GSM8K, $0.84$ on MATH-500, $0.86$
on MBPP and $0.83$ on HumanEval; per cell, over the $72$ cells holding at least $500$ tokens at
$\hat p_\nu\ge0.01$, where the budget is spent, the ratio runs from $0.49$ to $1.79$ with median
$0.84$. The table errs on the conservative side. \emph{Per round} (b): the realized regret
$\mathbb{E}[R(S)]$ and the predicted $\mathbb{E}[\hat R(S)]$ track each other across two decades
of $\varphi$ on all four benchmarks, realized over predicted lying between $0.7$ and $1.1$ on
every cell with at least $1{,}000$ audited rounds except the slowest MBPP cell ($0.29$, six
mismatches in $3{,}146$ rounds), and both stay below the budget line: realized regret is
$0.13$--$0.56\,\varphi$ against a predicted $0.41$--$0.66\,\varphi$ on those cells, and the largest
realized regret over all $36$ cells is $0.91\,\varphi$. The margin has two sources. The walk closes
at the first candidate that does not fit, so it typically spends about half of $\varphi$, and the
table's slight over-prediction adds to it. \emph{Along the decode} (c): realized over predicted
stays within $0.7$--$1.0$ for the first nine deciles of the canvas on every benchmark and rises to
$1.1$--$1.5$ in the last, where few tokens remain; the saw-tooth of the two BL$=$64 benchmarks is
block structure (the tails of the four $64$-token blocks, where the few remaining candidates are
near-certain), and the two curves move through it together.

Empirically, then, Definition~\ref{def:temp-calibrated} holds on the decoder's own trajectories with a
small $\bar\epsilon$ and, where it does not hold exactly, errs in the safe direction, and the
conclusion of Proposition~\ref{prop:temp-int} is met with margin at every budget: $\varphi$ is an exchange
rate the decode does not exceed. The refused candidates are never replayed on these runs, so the
statement is conditional on admission, as the proposition requires. It is one backbone,
LLaDA-8B-Instruct, at one geometry per benchmark; HumanEval has $164$ problems, and its two
slowest cells carry about $430$ audited rounds against more than $1{,}200$ elsewhere.

\begin{figure}[H]
  \centering
  \includegraphics[width=\textwidth]{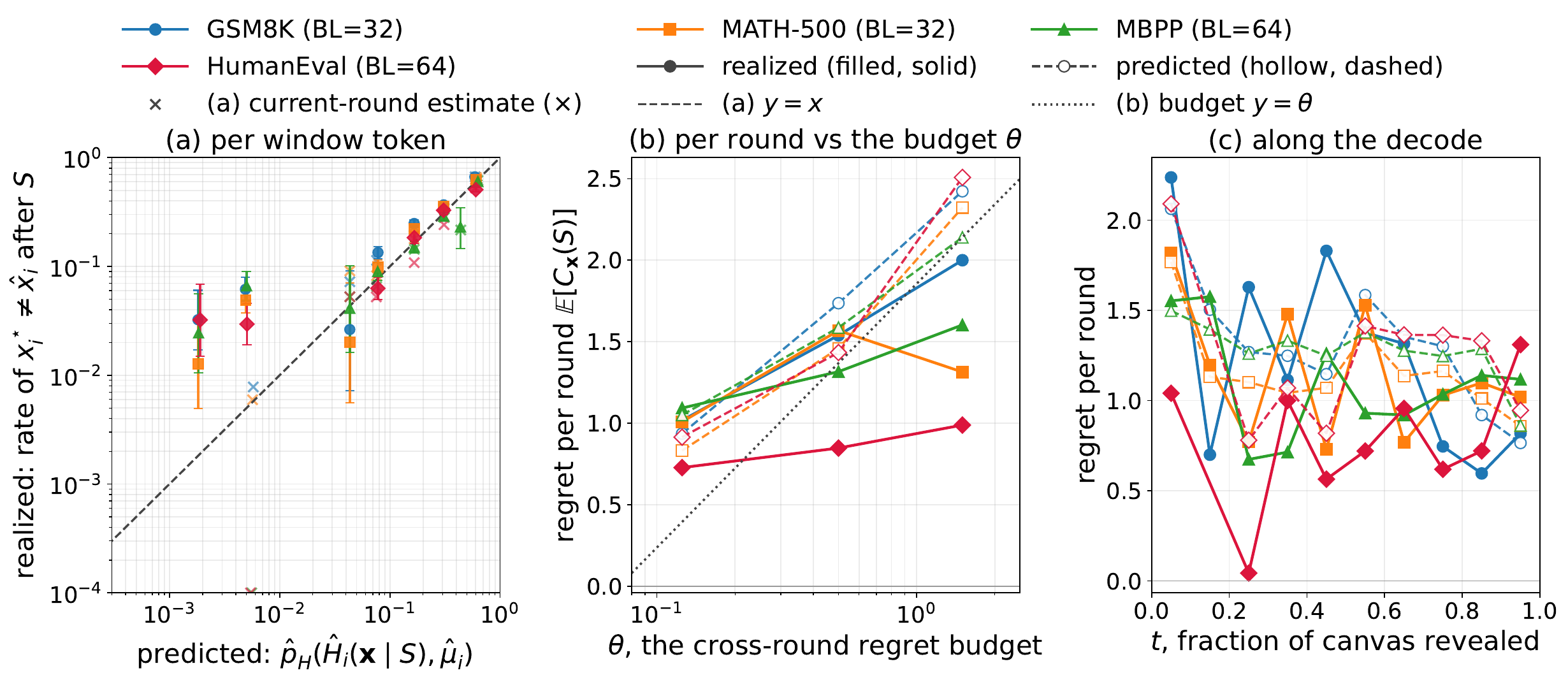}
  \caption{\textbf{The cross-round regret calibration on the decoder's own trajectories: \PIGext{}
  on LLaDA-8B-Instruct} at the geometry of \Cref{fig:main-frontier} (composite rule, $32$ problems
  per cell, every sixth round audited by replaying $S$ and then $N(S)$ sequentially from the
  round's own canvas, Definition~\ref{def:temp-reference}). (a) Realized deferred disagreement of
  the window tokens against the deferred estimate $\hat p_H(\hat H_i(\mathbf{x}\mid S),\hat\mu_i)$
  (Wilson $95\%$); crosses: the current-round estimate of the same tokens. (b) Per $\theta$:
  realized cross-round regret $C_{\mathbf{x}}(S)$ per audited round (filled) against the rule's
  $\hat C_{\mathbf{x}}(S)$ (hollow) and the budget line $y=\theta$. (c) The same per decile of the
  decode.}
  \label{fig:pigext-calibration}
\end{figure}

\paragraph{\PIGext{}: the deferred estimate and the cross-round regret.}
\PIGext{} rests on a second premise. Its statistic $\hat C_{\mathbf{x}}(S)$
(\Cref{eq:predicted-cross-regret}) evaluates the table twice for every window token, at the
current entropy and at the entropy predicted once $S$ is committed, and
Definition~\ref{def:temp-deferred} asks that the second, deferred evaluation be calibrated as well,
with an excess error $\bar\delta$; Lemma~\ref{lem:temp-ext} then bounds the realized cross-round
regret, $C_{\mathbf{x}}(S)$ of \Cref{eq:pig-ext-regret}, of a set the walk has passed by
$\theta+|N(S)|(2\bar\epsilon+\bar\delta)$. We measured both on the composite \PIGext{} rule at the
geometry of \Cref{fig:main-frontier}, three cells per benchmark, the floored point
($\theta=0.125$, $\nu\ge0.9$) and the plain $\theta=0.5$ and $1.5$, on the first $32$ problems of
each benchmark. At every sixth round, once the rule had chosen $S$, the harness replayed $S$ and
then its window $N(S)$, the $16$ most confident remaining positions of the block, one token at a
time from the round's own canvas, committing the sequential argmax at each step, which is the
reference of Definition~\ref{def:temp-reference} extended to the window, and it read the deferred
prediction of every window token from the canvas the decoder actually reached after committing
$S$. Each window token therefore carries two realized disagreements, committed now alongside $S$
and committed one round later, and the two table values the rule priced it with,
$\hat p_H(H_i,\hat\mu_i)$ and $\hat p_H(\hat H_i(\mathbf{x}\mid S),\hat\mu_i)$; per round, the
realized $C_{\mathbf{x}}(S)$ is the difference of the two disagreement counts and
$\hat C_{\mathbf{x}}(S)$ the difference of the two table sums, the number the rule compared with
$\theta$. The replays cost $|S|+|N(S)|$ forward passes per audited round, which is why these cells
are small, $73$ to $469$ audited rounds and $1{,}100$ to $5{,}500$ window tokens each. As before,
the decode itself is untouched and the replay forwards are not counted in NFE.

\Cref{fig:pigext-calibration} reads the audit the same three ways. \emph{Per window token} (a):
binned by its predicted value, the deferred estimate is calibrated wherever it is not tiny. For
$\hat p_H\ge0.05$ the realized deferred disagreement over the prediction lies between $1.13$ and
$1.72$ on GSM8K, $1.03$ and $1.33$ on MATH-500, $0.53$ and $1.16$ on MBPP and $0.81$ and $1.10$
on HumanEval, close to the diagonal above $\hat p_H\approx0.2$. In the tail it under-predicts:
where the table predicts less than $1\%$, $1$ to $7\%$ of the window tokens still change once $S$
is revealed, while the current-round estimate of the same tokens is accurate there ($0$ to
$0.8\%$ realized against $0.02$ to $0.6\%$ predicted), so the miss is specific to the deferred
term, an absolute $\bar\delta$ of a few percent. \emph{Per round} (b): the realized cross-round
regret per audited round is $0.73$--$1.09$ at the floored point, $0.85$--$1.57$ at $\theta=0.5$
and $0.99$--$2.00$ at $\theta=1.5$, against a predicted $0.83$--$1.05$, $1.43$--$1.74$ and
$2.14$--$2.51$: the statistic tracks the realized regret at the slow end (realized over predicted
$0.80$--$1.21$) and over-predicts it as $\theta$ grows ($0.59$--$1.08$ at $\theta=0.5$,
$0.39$--$0.83$ at $\theta=1.5$), the same conservative direction as $\hat p_\nu$ in the \PIGint{}
audit, larger here. Both quantities exceed $\theta$ at every point, the predicted one by
construction: \Cref{alg:pig-ext} admits a token and then tests, so the committed $S$ is the first
set whose statistic exceeds the budget, and the panel reports $\hat C_{\mathbf{x}}$ at that set,
the value of the test that closed the round, on every round the statistic closed ($100\%$ of the
audited rounds in the plain cells, $74$--$81\%$ in the floored cells, the rest being closed by the
floor). The budget bounds the prefix that passed the last test, which is how
Proposition~\ref{thm:temp-ext-cascade} uses it, admission by admission; the committed set carries
one more admission, whose effect on a $16$-token window is of order one predicted disagreement
whatever $\theta$ is, which is why the excess dwarfs the budget at $\theta=0.125$ and is comparable
to it at $1.5$. The rounds the statistic did not close show the other side: on the floor-closed
rounds of the floored cells $\hat C_{\mathbf{x}}$ is $0.02$--$0.03$ and the realized regret
$0.00$--$0.29$ per round. The audit did not record the statistic at the prefix, so the panel
compares realized and predicted regret at the committed set, the comparison calibration needs, and
not the committed set against the budget. \emph{Along the decode} (c): the realized regret per round falls from
$1.0$--$2.2$ in the first decile of the canvas to $0.8$--$1.3$ in the last and the prediction from
$1.5$--$2.1$ to $0.8$--$1.0$; the two move together, the realized curve noisy at $32$ problems.

Definition~\ref{def:temp-deferred} therefore holds on the decoder's own trajectories up to a tail
in which the deferred estimate is too optimistic by a few percent, and every other error is in the
safe direction: \PIGext{} predicts more benefit from waiting than it realizes, and waits
accordingly. The caveats are the cell size, $32$ problems with every sixth round audited, and
that under the floor the window is the paper's $N(S)$, the rule's own above-floor pool being
nearly empty there.

\subsection{Comparison with dependency-guided decoders}
\label{app:analyses:dapd}

\paragraph{Provenance of the reported points.}
No external number in this appendix was produced by us; each is transcribed from a published
table or an official repository, and the decoders themselves are due to their respective authors.
Fast-dLLM is the confidence-threshold parallel decoder of \citet{wu2025fastdllm}; EB-Sampler is
the entropy-bounded unmasking rule of \citet{benhamu2025ebsampler}; KLASS is the KL-guided
stability sampler of \citet{kim2025klass}; DAPD-Staged and DAPD-Direct are the attention-graph
independent-set decoders of \citet{kim2026dapd}; the RL unmasking policy is that of
\citet{jazbec2025unmasking}. In \Cref{fig:llada-dapd}, the BL=64 Fast-dLLM, EB-Sampler and KLASS
points are Table~3 of \citet{kim2026dapd} (LLaDA-8B-Instruct, four $64$-token blocks, accuracy and
steps), which their repository README reproduces verbatim \citep{kim2026dapdcode}; the two DAPD
points at BL=64 appear only in that README (its ``4-block'' rows, torch \texttt{2.5.1+cu121}
according to the README) \citep{kim2026dapdcode}, the paper itself reporting DAPD at a single
$256$-token block only (its Tables~3 and~5), which we do not draw for the reasons given below.
Steps are converted to TPF as $256/\text{steps}$, since their harness performs one forward per
step; the Wilson intervals on these points use their problem counts (GSM8K $1{,}319$, MATH-500
$500$, MBPP $500$, HumanEval $164$), not ours. In \Cref{fig:llada-jazbec-bl32}, the RL-policy and Fast-dLLM points are Appendix~J of
\citet{jazbec2025unmasking} (Tables~2 and~3, the block-$32$ rows, values the authors state
correspond directly to their plotted points),
with their reported three-seed spread as the error bar; that their evaluation is $0$-shot under
the chat template, and that their block-$32$ decoder unmasks contiguous blocks in order, we read
from the official code \citep{jazbec2025code} rather than the paper. Our own lanes on every panel
are the cells tabulated in \Cref{app:results:tables}.

\paragraph{Our reproduction of DAPD.}
The filled black markers in \Cref{fig:llada-dapd} are DAPD run by us, under our protocol, on every
panel. We transcribed the decoder from the released implementation
\citep{kim2026dapdcode}: the dependency graph is the head-averaged attention of the last $30\%$
of layers (ten of LLaDA's $32$), symmetrised, restricted to the masked positions of the active
block and normalised per row by its maximum; candidates are ranked by normalised degree times
normalised confidence; a greedy independent set accepts a candidate when its dependency to every
accepted token is at most $\tau$, with $\tau$ rising linearly in decoded fraction between the
endpoints of their runner scripts (Staged $0.01$--$0.15$ on HumanEval and $0.01$--$0.05$ on GSM8K;
Direct $0.01$--$0.05$ and $0.005$--$0.05$); Direct first commits every position of confidence
$1.0$, Staged also commits every position of confidence at least $0.9$ once half the canvas is
decoded, confidence being read in the model's logits precision as in their code. The
transcription was then checked against their code itself: we ran their \texttt{generate\_dapd}
unmodified, on our prompts and weights, over all $164$ HumanEval problems at their default single
block, and obtained pass@1 $0.256$ at $27.7$ forwards (Staged) and $0.280$ at $24.9$ (Direct)
against $0.262$ at $27.4$ and $0.287$ at $25.5$ from our transcription; per problem, the two agree
on pass/fail for $86\%$ of problems with a median forward-count difference of two, the residual
being the one deliberate numerical difference (their code recomputes attention in single
precision, ours reads the model's half-precision maps). The BL=64 cells drawn use our transcription
at the same $\tau$ schedules with the same $n$ as our lanes, on all four benchmarks.

Under our protocol, DAPD sits on the Fast-dLLM frontier rather than above it: on GSM8K, Staged
reaches $0.776$ at $2.3$ TPF and Direct $0.746$ at $4.2$ against Fast-dLLM's $0.800$--$0.810$ at
$2.9$--$4.0$; on MATH-500, Staged $0.364$ at $2.0$ and Direct $0.348$ at $3.2$ against
$0.380$--$0.368$ at $2.2$--$3.0$; on MBPP, Staged $0.380$ at $2.7$ and Direct $0.386$ at $5.2$,
at parity with Fast-dLLM and the two \PIG{} lanes at that speed; on HumanEval, Staged $0.256$ at
$3.0$ and Direct $0.317$ at $6.9$ against $0.390$ at $5.2$ and $0.360$ at $7.2$. Our forward
counts at BL=64 track their README's on every benchmark (MATH-500 $126$ and $80$ against $123$
and $85$; MBPP $95$ and $50$ against $97$ and $37$), while the accuracies move with the protocol:
higher than theirs by eight to nine points on MATH-500, equal on MBPP, and lower by about twelve
points on HumanEval even though our Fast-dLLM at the same geometry is not; so the offset is
specific to DAPD under our prompt. A prompt ablation with their code locates most of it: run on HumanEval at
their default single block under lm-eval's raw completion prompt, the prompt their released
runners use, their code reproduces the step counts of their Table~3 exactly and DAPD-Direct's
accuracy within a point, and DAPD-Staged's to within six points ($n{=}164$); under our
chat-template prompt the same code at the same $\tau$ decodes in $27$ forwards instead of $45$,
so the prompt changes the dependency graph itself, not only the answers. Our observations are
therefore that a substantial part of the discrepancy between the reported DAPD numbers and our
runs comes down to the prompt protocol, which reshapes the attention graph DAPD selects from; it
does not account for all of it, the ablation covers one benchmark and one geometry, and we have
not run their code inside their own harness.

\paragraph{DAPD's operating points on the BL=64 plane.}
\label{app:results:dapd}
\Cref{fig:llada-dapd} places every BL=64 operating point that \citet{kim2026dapd} report for
LLaDA-8B-Instruct on our BL=64 fixed-block panels, the geometry the two match: Fast-dLLM,
EB-Sampler and KLASS at BL=64, and both DAPD variants at BL=64 (DAPD-Staged, the algorithm of
their main text, and the acceleration variant DAPD-Direct), the latter two reported only in the
project README. Their setting is otherwise close to ours --- canvas $256$, greedy, lm-eval default
shots ($5/4/0/3$), one forward per step so TPF $=256/\text{steps}$ --- but each method is a single
best operating point, no $\tau_t$ sweep is reported, and the comparison is not paired: prompts,
graders and the GSM8K problem set (their full $1{,}319$, our first $500$) differ, their KV-cache
use is unstated, and their released lm-eval runners default to \emph{no} chat template for the
Instruct model where Fast-dLLM's own harness applies it, so whether their LLaDA numbers are raw
few-shot cannot be determined from the repository.

Read at this geometry, DAPD's advantage is confined to the single-block setting the paper
headlines. At BL=64 DAPD-Staged needs $93$--$123$ steps ($2.1$--$2.8$ TPF), no faster than
their own Fast-dLLM; DAPD-Direct is faster (GSM8K $0.758$ at $4.2$ TPF, HumanEval $0.427$ at
$4.3$, MBPP $0.388$ at $7.0$) and on the code panels sits at or above our lanes at those speeds,
while on the math panels our \PIGint{} and \PIGext{} lanes pass through or above it. We do not
draw their single-block results: there the baselines are run with end-of-sequence suppression
after collapsing without it, a protocol our decoder does not use, and their BL=64 Fast-dLLM on
MBPP ($0.206$ at $6.2$ TPF against our $0.354$ at $5.9$, and $0.404$ in their own suppressed
single-block run) shows the same truncation at work.

\begin{figure}[H]
  \centering
  \includegraphics[width=\textwidth]{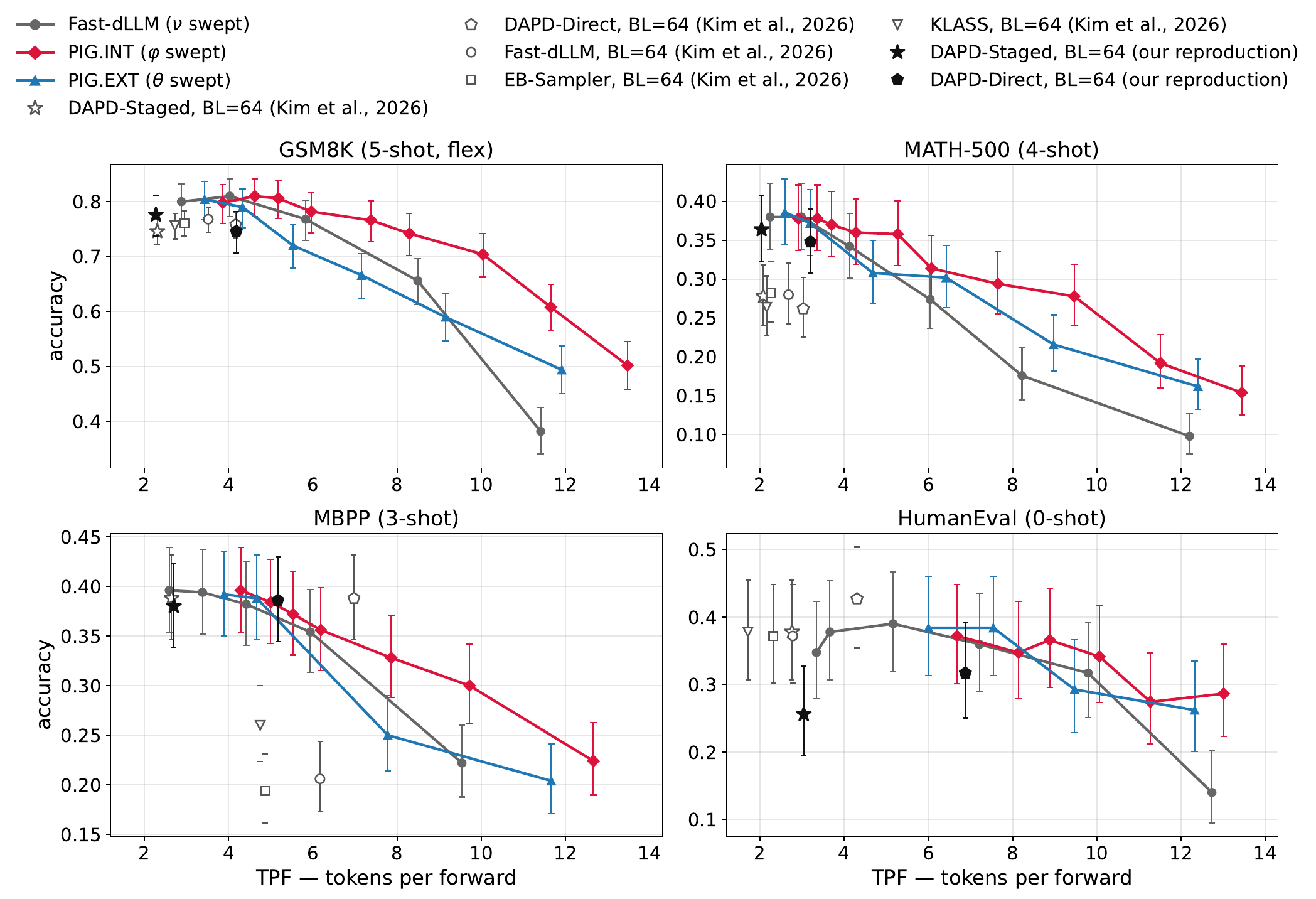}
  \caption{\textbf{Our BL=64 frontier against every BL=64 operating point DAPD reports} on
  LLaDA-8B-Instruct (math above, code below; flexible extraction for GSM8K/MATH-500, pass@1 for
  code with Wilson 95\% bars. Lines: our Fast-dLLM, \PIGint{} and \PIGext{} sweeps at
  BL=64 fixed. Hollow
  markers, all at BL=64 (four $64$-token blocks): Fast-dLLM \citep{wu2025fastdllm}, EB-Sampler
  \citep{benhamu2025ebsampler} and KLASS \citep{kim2025klass} as reported in Table~3 of
  \citet{kim2026dapd}; DAPD-Staged (star) and DAPD-Direct (pentagon) from the project README
  \citep{kim2026dapdcode}, absent from the paper. Filled black star and pentagon: DAPD-Staged and DAPD-Direct run by us
  under our protocol.}
  \label{fig:llada-dapd}
\end{figure}

\subsection{A learned unmasking policy at matched block geometry}
\label{app:analyses:jazbec}
\label{app:results:jazbec-bl32}
\Cref{fig:llada-jazbec-bl32} places the published BL=32 series of \citet{jazbec2025unmasking} on
our BL=32 fixed-block GSM8K and MATH-500 panels: their RL unmasking policy, one policy per reward
weight $\alpha\in\{0,0.3,1,3,10\}$ with their three-seed spread, and Fast-dLLM as run by them
(their Appendix~J, Tables~2 and~3). Their decoder unmasks contiguous $32$-token blocks in order
on a $256$-token canvas with one forward per step, so block geometry, canvas and NFE accounting
match our BL=32 fixed regime exactly; what differs is the prompt (theirs $0$-shot under the chat
template, ours $5$-shot GSM8K and $4$-shot MATH-500) and the grader, so the figure is a reference
at matched geometry rather than a paired comparison. Points beyond
$14$ TPF --- their fastest policy ($\alpha{=}10$, about $25$ TPF, GSM8K $0.35$ and MATH-500 $0.15$),
their fastest Fast-dLLM threshold, and our $\theta{=}3$ cell --- are not drawn; a dashed segment
continues each series toward its hidden neighbour. Read this way, the learned policy sits on our
lanes in the $2$--$4.5$ TPF band where it is reported, and the range in which \PIGint{} and
\PIGext{} extend the frontier ($5$--$14$ TPF) is one where the policy is reported only at its
$\alpha{=}10$ end. Their own Fast-dLLM sweep tracks ours to within about $0.03$ over $5$--$14$ TPF
on GSM8K, which serves as the protocol-alignment check between the two rigs.

\begin{figure}[H]
  \centering
  \includegraphics[width=\textwidth]{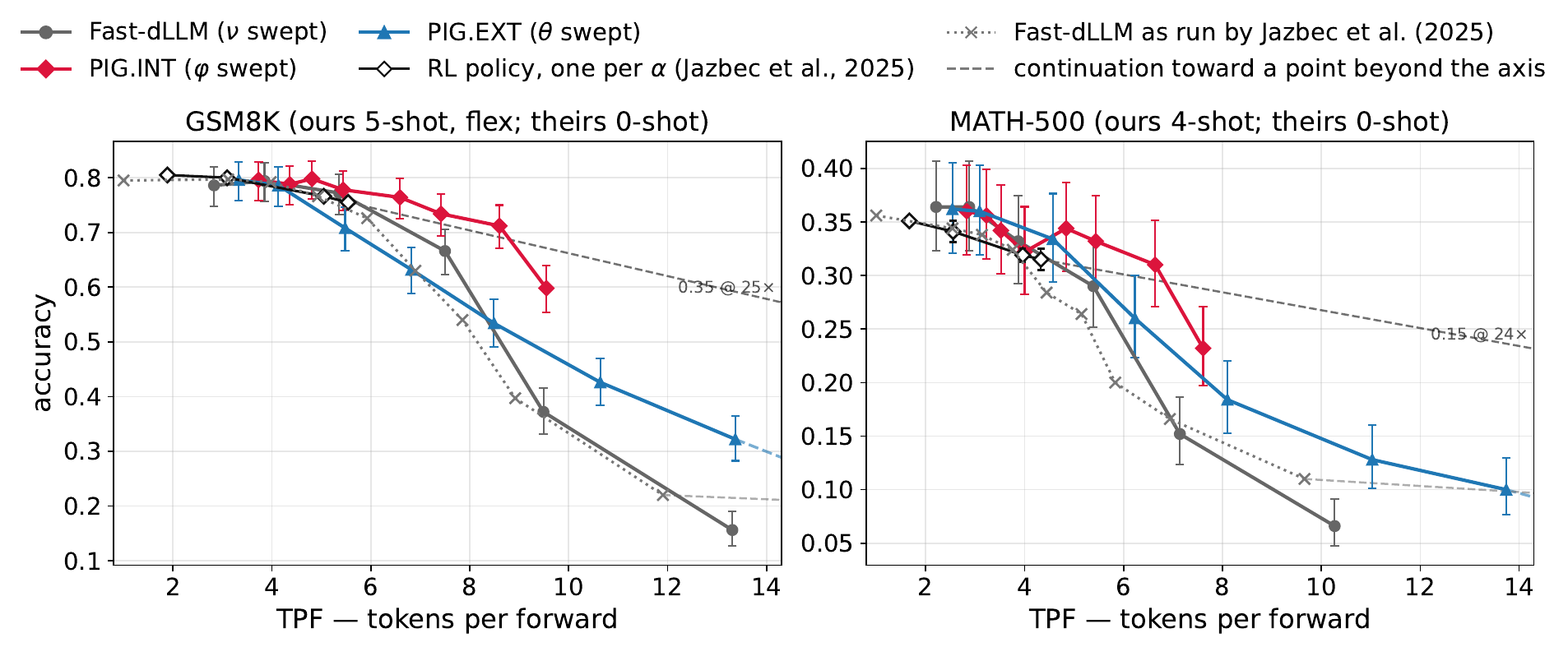}
  \caption{\textbf{Our BL=32 frontier against the published BL=32 series of
  \citet{jazbec2025unmasking}} on LLaDA-8B-Instruct (GSM8K flexible extraction, MATH-500 with Wilson 95\% bars. Lines: our Fast-dLLM, \PIGint{} and \PIGext{} sweeps at BL=32
  fixed. Black diamonds:
  their RL policy, one per reward weight $\alpha$, $\pm$ their three-seed spread; dotted grey:
  Fast-dLLM as run by them; both at BL=32 fixed, canvas $256$, $0$-shot.
  Same block geometry, canvas and NFE accounting.}
  \label{fig:llada-jazbec-bl32}
\end{figure}

\subsection{One number per comparison: matched-accuracy throughput and accuracy at a fixed throughput}
\label{app:results:matched}
The frontier figures compare sweeps; \Cref{tab:matched-int,tab:matched-ext} in \Cref{sec:exp-frontier} reduce each comparison to one number.
\Cref{tab:matched-int} compares both decoders against every baseline of this appendix in the
direction they are built for, throughput at matched accuracy. Each baseline supplies one target,
its best operating point: an accuracy $a_{\mathrm{ref}}$ at a throughput
$\mathrm{TPF}_{\mathrm{ref}}$. Against it we report the fastest \emph{measured} cell of the lane's
sweep whose accuracy is at least $a_{\mathrm{ref}}$, and give two numbers for that cell, its
speed-up $\mathrm{TPF}/\mathrm{TPF}_{\mathrm{ref}}$ and its accuracy minus $a_{\mathrm{ref}}$ in
points; where no cell of a lane reaches the target, the lane's most accurate cell is reported and
the difference is negative. Nothing is interpolated and no tolerance is applied, so both numbers
name a run we performed. Fast-dLLM supplies the most accurate cell of its own sweep on our rig and
DAPD our port of the authors' code at BL$=$64 (\Cref{app:results:dapd}), so those three columns are
paired; EB-Sampler and KLASS supply only the single cells published by \citet{kim2026dapd} at the
same four-block geometry, on their rig, so those columns are not paired, and the protocol offset of
\Cref{app:results:dapd} (our Fast-dLLM exceeds theirs by 4--14 points at equal TPF on three of four
benchmarks) accounts for a large part of them. Read on the paired columns, \PIGint{} reaches
DAPD-Direct's accuracy at $1.5$--$1.8\times$ its throughput on three of the four benchmarks and at
$0.8\times$ on MBPP, in every case at equal or higher accuracy.

\Cref{tab:matched-ext} compares the three lanes we sweep on both models in the other direction,
accuracy at fixed throughputs: $10$ and $15$ TPF on LLaDA-8B-Instruct, $3$ and $5$ TPF on
Nemotron-3B-Instruct, every row at BL$=$64 fixed, each lane's accuracy linearly interpolated between the two
measured cells that bracket the target. These are estimates read off the frontier rather than
measured operating points, and no lane is extrapolated: where a sweep ends before the target, the
table gives that lane's last measured cell, marked with its own throughput, and takes its gain
against Fast-dLLM at that same throughput. Four cells are of this kind (\PIGint{} on LLaDA
MATH-500 and MBPP at $15$ TPF; \PIGint{} on Nemotron MBPP and \PIGext{} on Nemotron HumanEval at
$5$ TPF), and on LLaDA MATH-500 at $15$ TPF Fast-dLLM's own sweep ends at $12.2$ TPF, so no gain
is given in that group. On LLaDA, \PIGint{} is the most accurate lane at $10$ TPF on every
benchmark ($+4$ to $+19$ points over Fast-dLLM) and \PIGext{} takes over at $15$ TPF on GSM8K and
MBPP ($+14$ and $+10$ points over Fast-dLLM); on MATH-500 it is the only lane still measured at
$15$ TPF. On Nemotron the three lanes are within a few points on GSM8K at both throughputs
(Fast-dLLM ahead at $3$ TPF), while on code \PIGext{} is the lane that stays accurate: $+9$/$+7$
points over Fast-dLLM at $3$ TPF, $36.9$ against $9.1$ points on MBPP at $5$ TPF, and $44.5$
against $19.5$ on HumanEval at $4.9$ TPF, the last cell of its sweep. Nemotron MATH-500 has
no BL$=$64 sweep and is not listed.    %
\clearpage

\section{Complete Frontier Grids and Tables}
\label{app:results}
\label{app:results:grids}

This appendix is the complete record of the frontier experiments: the protocol under which every
lane was run, the full frontier grids of both models, and the tables of every operating point
behind them, including the cells the figures crop. The analyses that read these results are in
\Cref{app:analyses}. Contents:
\begin{itemize}[leftmargin=1.4em, itemsep=1pt, topsep=2pt]
  \item \textbf{LLaDA-8B-Instruct} (\Cref{app:results:llada}): benchmarks, prompts, graders and decode regimes;
        \Cref{fig:llada-grid-full}, the five-regime $\times$ four-benchmark grid (flexible extraction);
        \Cref{fig:llada-grid-strict}, the same grid under strict extraction.
  \item \textbf{Nemotron-Labs-Diffusion-3B} (\Cref{app:results:nemotron}): benchmarks, prompts, graders and decode regimes;
        \Cref{fig:nemotron-grid-full}, the three-geometry grid at BL$=$32;
        \Cref{fig:nemotron-grid-bl16,fig:nemotron-grid-bl64}, the same grid at BL$=$16 and BL$=$64.
  \item \textbf{Tables} (\Cref{app:results:tables}), one per model and benchmark, every operating point with its Wilson 95\% interval:
        LLaDA-8B-Instruct GSM8K (\Cref{tab:frontier-llada-gsm8k}), MATH-500 (\Cref{tab:frontier-llada-math500}),
        MBPP (\Cref{tab:frontier-llada-mbpp}), HumanEval (\Cref{tab:frontier-llada-humaneval});
        Nemotron-Labs-Diffusion-3B GSM8K (\Cref{tab:frontier-nemotron-gsm8k}), MATH-500 (\Cref{tab:frontier-nemotron-math500}),
        MBPP (\Cref{tab:frontier-nemotron-mbpp}), HumanEval (\Cref{tab:frontier-nemotron-humaneval}).
\end{itemize}

\subsection{LLaDA-8B-Instruct}
\label{app:results:llada}

\paragraph{Benchmarks, prompts and graders.}
We evaluate LLaDA-8B-Instruct \citep{nie2025large} on GSM8K \citep{cobbe2021training}, MATH-500
\citep{hendrycks2021math}, MBPP \citep{austin2021program} and HumanEval \citep{chen2021evaluating}
under the evaluation protocol published with Fast-dLLM \citep{wu2025fastdllm}, which is the only
decoding paper in this space that states every knob: lm-eval task formats and graders, the shot
counts $5/4/3/0$, a $256$-token canvas, greedy commits, no KV cache. Prompts follow the lm-eval
task formats. GSM8K is five-shot in the \texttt{Question:}/\texttt{Answer:} format with the
\texttt{\#\#\#\#} answer line, the five exemplars fixed once from the GSM8K training split.
MATH-500 uses lm-eval's Minerva template and its four fixed exemplars. MBPP is three-shot with the
task's own \texttt{[BEGIN]}/\texttt{[DONE]} exemplar format, and HumanEval is zero-shot function
completion. Every prompt is wrapped in LLaDA-Instruct's own chat template, as the model's release
and Fast-dLLM's harness do for the Instruct model; this is our one stated deviation from stock
lm-eval, whose LLaDA runs are raw few-shot completions, and it matters (on a BL=64 probe the chat
template gave $0.72$ flexible-extraction accuracy on GSM8K against $0.31$ for the raw prompt).
Graders are lm-eval's: for GSM8K both the strict extractor (the \texttt{\#\#\#\#} line) and the
flexible one (last number in the generation), the latter being the main-text metric because the
instruct model omits the \texttt{\#\#\#\#} line in roughly a third of its answers; for MATH-500
the last \texttt{\textbackslash boxed\{\}} answer under Minerva normalisation and sympy
equivalence; for HumanEval execution-based pass@1 with the standard stop sequences; for MBPP
execution of the task's three tests in a sandbox. Problem counts are the first $500$ of GSM8K's
$1{,}319$ test problems, all $500$ of MATH-500, the $500$ problems of the classic MBPP test split,
and all $164$ of HumanEval; Wilson $95\%$ intervals are drawn on every accuracy. TPF is $256$
divided by the mean number of forwards per problem.

\paragraph{Decode regimes.}
Five candidate geometries are evaluated on every benchmark, each with its own baseline sweep:
fixed $32$- and $64$-token blocks, the semi-autoregressive block decoding of the LLaDA release; the
dynamic $32$- and $64$-candidate windows of \Cref{sec:exp-dynamic-blocks}, which advance as
commits land; and the free canvas, where all $256$ positions are eligible. On every geometry the
\PIGext{} lane is the operating rule stated in \Cref{sec:experiments}: two confidence-floored
cells at the slow end ($\theta{=}0.125$ with $\nu{\ge}0.9$, $\theta{=}0.25$ with $\nu{\ge}0.8$), then
the plain $\theta$ ladder from $0.5$; the tables mark the floored rows. The main text draws
the fixed $32$-token block and the free canvas (\Cref{fig:main-frontier}) and the two $32$-wide
geometries (\Cref{fig:dynamic-blocks}); the full grid is \Cref{fig:llada-grid-full}.

\paragraph{Baseline.}
The baseline on every panel is Fast-dLLM's confidence-threshold rule run by us: at each round,
every eligible masked position whose top-1 probability exceeds the threshold is committed, at
least one always, with the same eligible positions, prompts, canvas and grader as \PIGint{} and
\PIGext{}, and with the threshold swept from $0.95$ down to $0.1$ to trace a frontier rather than
the single default of $0.9$. No KV cache is used on any lane, so TPF counts forwards over the
full canvas for every method alike. Every number taken from another paper, and our own reproduction of DAPD, is documented in
\Cref{app:analyses:dapd}.

\paragraph{Frontier grids.}
\Cref{fig:llada-grid-full} is the complete LLaDA-8B-Instruct frontier: five decode regimes (rows)
$\times$ four benchmarks (columns), all three lanes at full range, flexible extraction. Dashed
segments continue a lane toward its next operating point beyond a panel edge. The main-text
figures are subsets of this grid; every operating point, including those beyond the crops, is
tabulated in \Cref{app:results:tables}.

\begin{figure}[t]
  \centering
  \includegraphics[width=\textwidth]{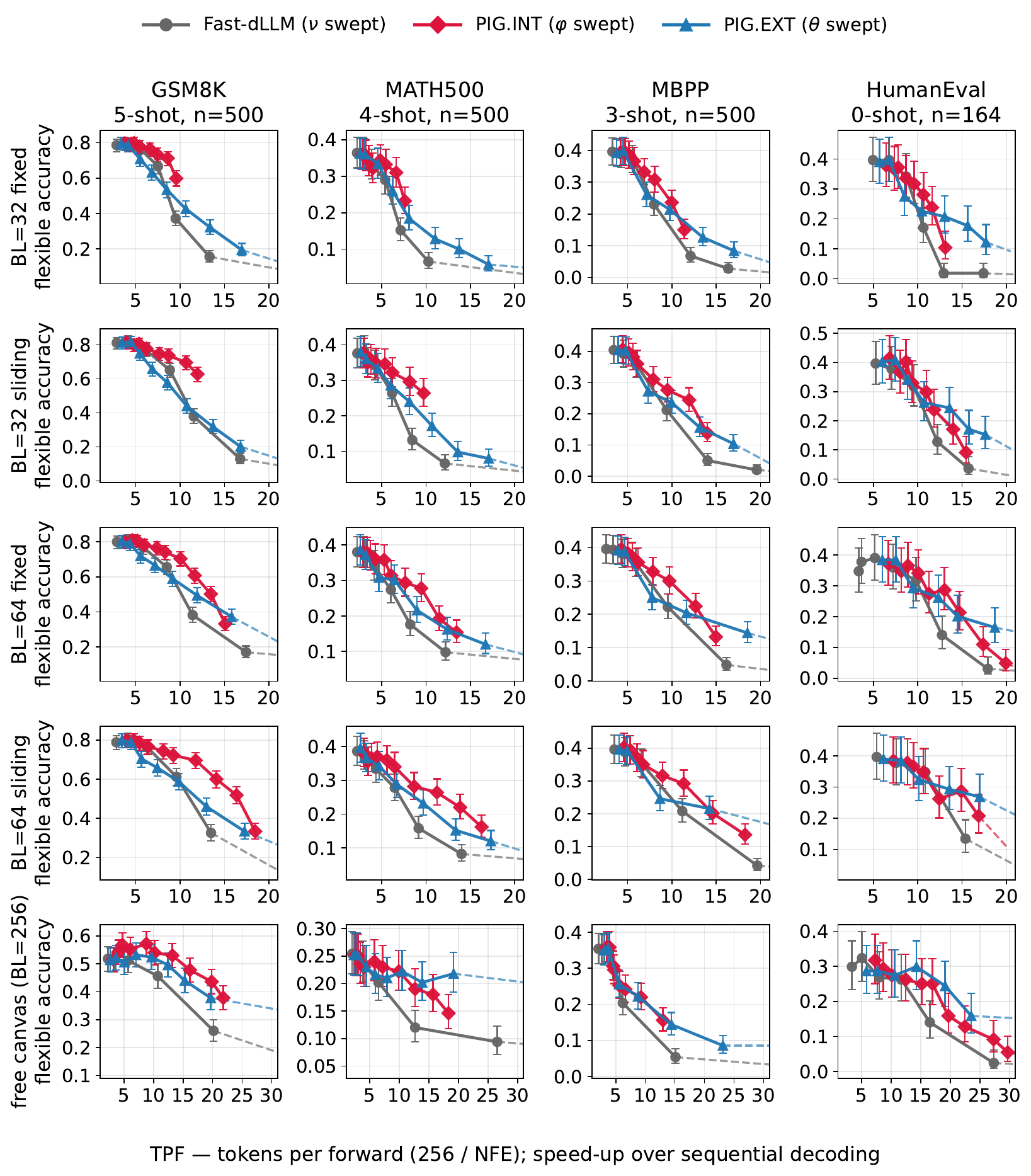}
  \caption{\textbf{The full LLaDA-8B-Instruct frontier} (canvas $256$, greedy, no KV cache;
  flexible extraction for GSM8K/MATH-500, pass@1 for MBPP/HumanEval; bars: Wilson 95\%). Rows:
  decode regime; columns: benchmark (shots and $n$ in the title). Grey: Fast-dLLM confidence
  threshold, swept $0.95\to0.1$; red: \PIGint{} ($\varphi$); blue: \PIGext{}, with its two slowest operating points additionally
  flooring admission at a confidence $\nu$ equal to $0.9$ and $0.8$ respectively.}
  \label{fig:llada-grid-full}
\end{figure}

\paragraph{Strict extraction.}
\Cref{fig:llada-grid-strict} repeats \Cref{fig:llada-grid-full} under the benchmarks' own answer
formats (\texttt{\#\#\#\#} for GSM8K, last \texttt{\textbackslash boxed\{\}} for MATH-500; code panels are
unchanged). Strict GSM8K accuracy is deflated for every lane alike by format non-compliance of the
instruct model (roughly a third of generations omit \texttt{\#\#\#\#}), which is why flexible
extraction is the main-text metric; the ordering of the lanes is the same under both.

\begin{figure}[t]
  \centering
  \includegraphics[width=\textwidth]{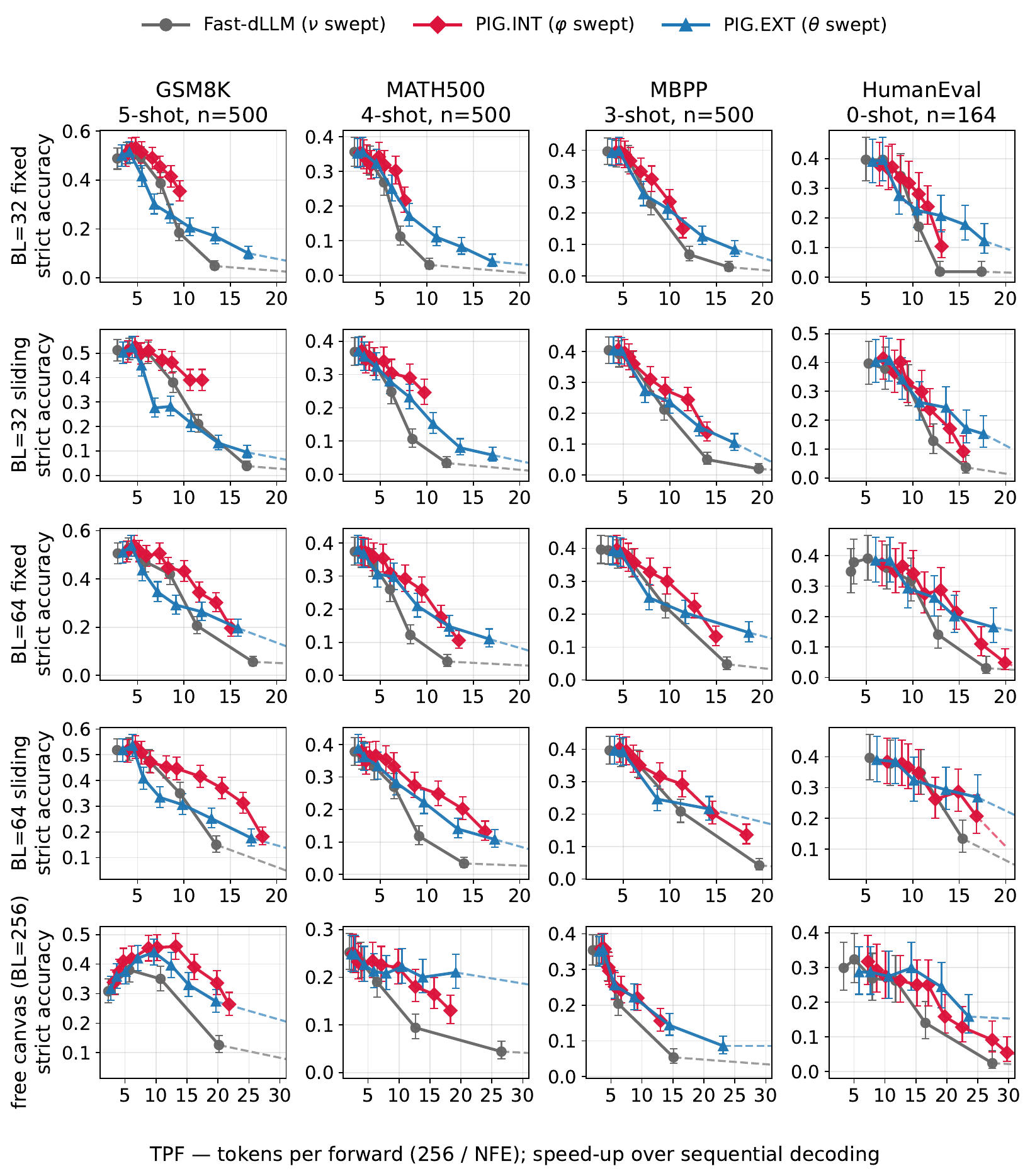}
  \caption{\Cref{fig:llada-grid-full} under strict extraction (GSM8K \texttt{\#\#\#\#}, MATH-500 last
  \texttt{\textbackslash boxed\{\}}); MBPP and HumanEval panels are identical to the main text.}
  \label{fig:llada-grid-strict}
\end{figure}

\subsection{Nemotron-Labs-Diffusion-3B}
\label{app:results:nemotron}
\label{app:results:nemotron:grids}

\paragraph{Benchmarks, prompts and graders.}
We evaluate Nemotron-Labs-Diffusion-3B \citep{fu2026nemotron} on the same four benchmarks,
zero-shot throughout, following the model's own evaluation practice: its released harness
(NeMo-Skills) prompts the instruct model without exemplars, and few-shot prompting makes it answer
the exemplars. Each task body is the benchmark's own text with a minimal instruction (the GSM8K
question as is; the MATH-500 problem followed by ``Please reason step by step''; the MBPP task
description with its three tests; ``Complete the following Python function'' before the HumanEval
prompt), wrapped in the model's chat template. The canvas is $2{,}048$ tokens, the model's
end-of-sequence handling stopping generation early, and the graders are the LLaDA ones above,
imported unchanged: flexible extraction for GSM8K (strict extraction is void at zero-shot, the
model answering in \texttt{\textbackslash boxed\{\}} rather than with a \texttt{\#\#\#\#} line),
last boxed answer under Minerva normalisation for MATH-500, execution for MBPP and HumanEval.
Problem counts are the first $500$ of GSM8K and of MATH-500, $499$ of MBPP, and all $164$ of
HumanEval. TPF is the mean number of generated tokens per forward; because the two models'
canvases differ, TPF is not comparable across models.

\paragraph{Decode regimes and baseline.}
Nemotron decodes in $32$-token blocks with a confidence threshold as its native decoder, and has
no free-canvas mode, so three geometries are evaluated at each block size: the fixed block, a
sliding window of the same width, and a commit band of that width nested in a window twice as
wide. The main text uses the model's own BL=32; \Cref{app:results:nemotron:grids} repeats the grid
at BL=16 and BL=64, and the HumanEval panel of \Cref{fig:teaser} uses the fixed $64$-token block.
The baseline on every panel is the model's own confidence-threshold decoder under the same
geometry, swept from $0.9$ (its default) down to $0.3$: on the fixed $32$-token block it is the
release's \texttt{generate} called as shipped, and on every other geometry the same rule
re-implemented over the window, with the release's end-of-sequence guard applied to every lane
alike. No numbers from other
papers appear on the Nemotron panels; the dashed line of \Cref{fig:teaser} is the accuracy the
model's release reports at its default setting.

\paragraph{Frontier grids.}
\Cref{fig:nemotron-grid-full} is the complete Nemotron-Diffusion-3B frontier at the model's own
block size: three geometries at BL=32 $\times$ four benchmarks. The main-text figure
(\Cref{fig:main-frontier}, bottom row) is a subset; every operating point, including those beyond the crops, is
tabulated in \Cref{app:results:nemotron:tables}. \Cref{fig:nemotron-grid-bl16,fig:nemotron-grid-bl64}
repeat the grid at BL=16 and BL=64, the three geometries scaled with the block (fixed block, sliding
window, and the block nested in a window twice as wide); MATH-500 was not run at BL=64, and at BL=16
its \PIGext{} lane has no confidence-floored cells. Here the baseline is the model's threshold rule re-implemented over
the window at every geometry, with the same end-of-sequence guard as our lanes, and \PIGext{} is
the composite operating rule of every other figure. The ordering of the lanes is the same at all
three block sizes.

\begin{figure}[t]
  \centering
  \includegraphics[width=\textwidth]{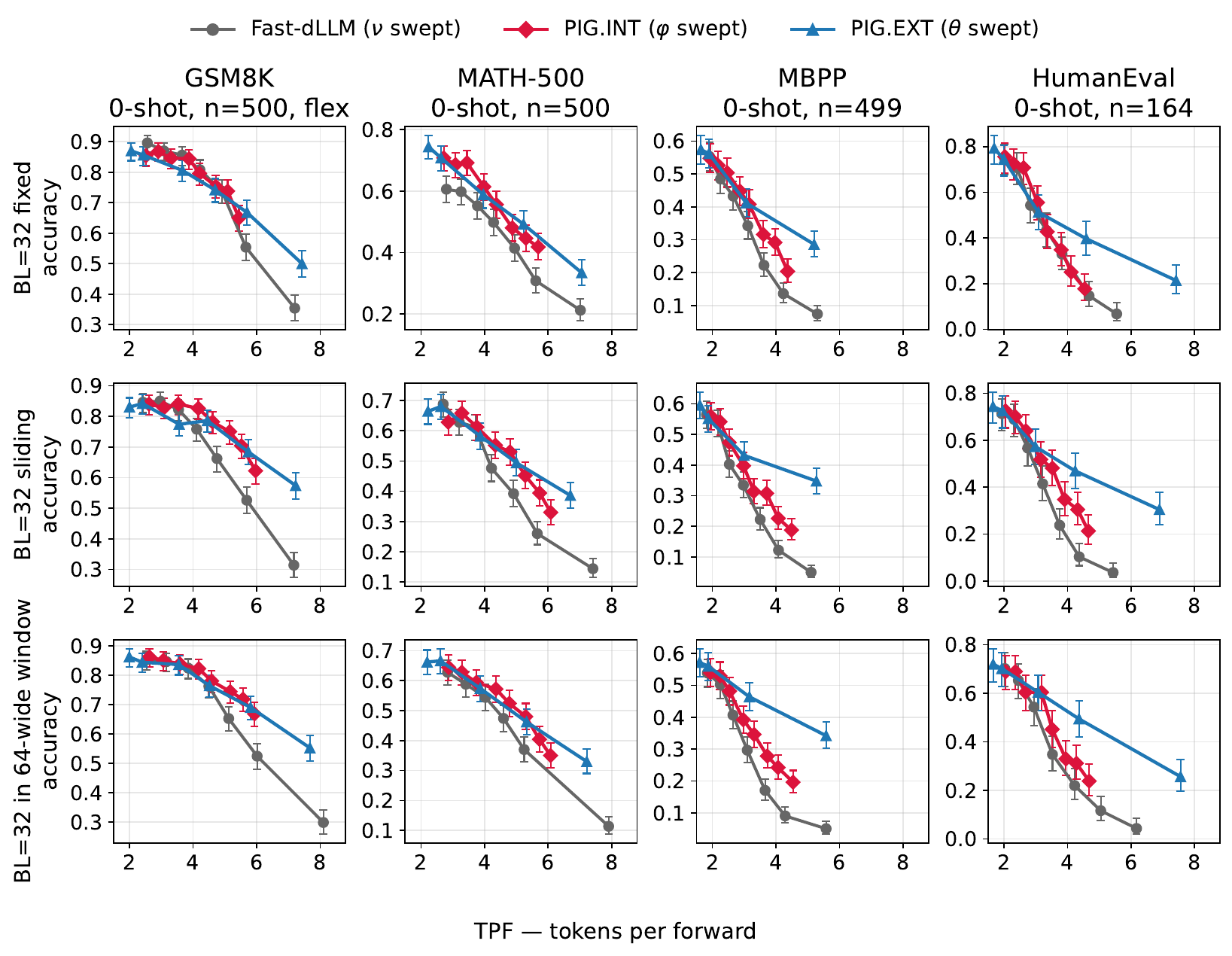}
  \caption{\textbf{The full Nemotron-Diffusion-3B frontier} ($0$-shot, canvas $2048$, greedy;
  flexible extraction for GSM8K, pass@1 for code; bars: Wilson 95\%). Rows: the model's own fixed
  $32$-token block, a sliding $32$-token window, and a $32$-token commit band inside a $64$-wide
  window; columns: benchmark. Grey: Fast-dLLM, the model's own confidence-threshold decoder swept
  $0.9\to0.3$ under the same geometry; red: \PIGint{}; blue: \PIGext{}, the composite operating
  rule of every figure. Cells beyond $8.5$ TPF are omitted; TPF is not comparable across the two
  models' canvases.}
  \label{fig:nemotron-grid-full}
\end{figure}

\begin{figure}[t]
  \centering
  \includegraphics[width=\textwidth]{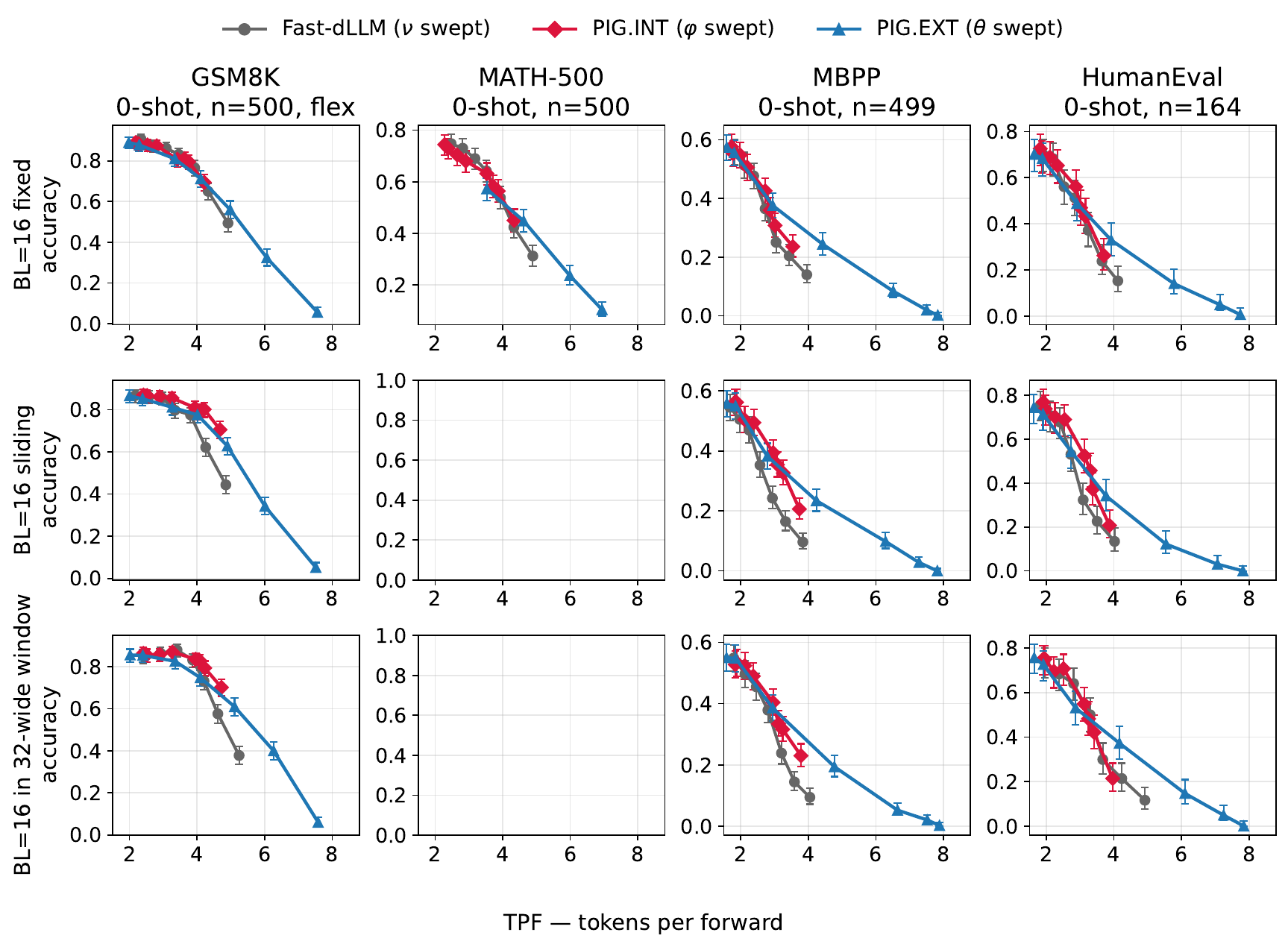}
  \caption{\textbf{The Nemotron-Diffusion-3B frontier at BL=16} ($0$-shot, canvas $2048$, greedy;
  flexible extraction for GSM8K, pass@1 for code; bars: Wilson 95\%). Rows: fixed $16$-token block,
  sliding $16$-token window, and a $16$-token commit band inside a $32$-wide window; columns:
  benchmark. Grey: the model's confidence-threshold decoder under the same geometry, swept
  $0.9\to0.3$; red: \PIGint{}; blue: \PIGext{}, the composite operating rule of
  \Cref{fig:main-frontier}. Cells beyond $8.5$ TPF are omitted.}
  \label{fig:nemotron-grid-bl16}
\end{figure}

\begin{figure}[t]
  \centering
  \includegraphics[width=\textwidth]{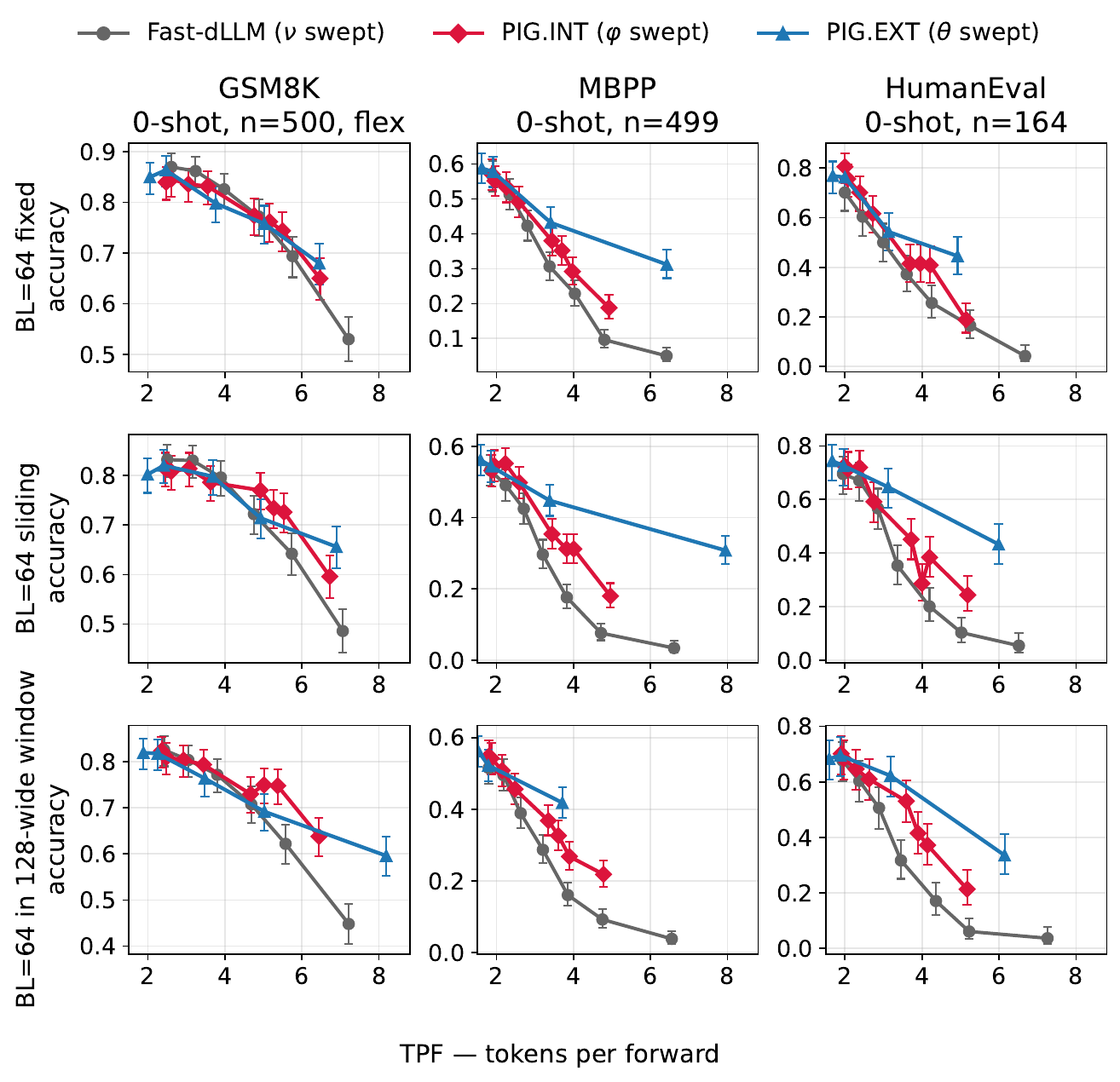}
  \caption{\textbf{The Nemotron-Diffusion-3B frontier at BL=64}, as \Cref{fig:nemotron-grid-bl16}
  with a fixed $64$-token block, a sliding $64$-token window, and a $64$-token commit band inside a
  $128$-wide window; MATH-500 was not run at this block size.}
  \label{fig:nemotron-grid-bl64}
\end{figure}

\clearpage
\subsection{Frontier tables}
\label{app:results:tables}
\label{app:results:nemotron:tables}

The eight tables that follow list every operating point behind the frontier figures of this
appendix, each accuracy with its Wilson 95\% interval, including the cells the figures crop.
\Cref{tab:frontier-llada-gsm8k,tab:frontier-llada-math500,tab:frontier-llada-mbpp,tab:frontier-llada-humaneval}
are the LLaDA-8B-Instruct grids of \Cref{fig:llada-grid-full} and \Cref{fig:llada-grid-strict}
(flexible and strict extraction);
\Cref{tab:frontier-nemotron-gsm8k,tab:frontier-nemotron-math500,tab:frontier-nemotron-mbpp,tab:frontier-nemotron-humaneval}
are the Nemotron-Labs-Diffusion-3B grids of \Cref{fig:nemotron-grid-full}. In every table the
Fast-dLLM lane is the confidence-threshold decoder swept over $\nu$, and PIG.EXT is the composite
operating rule of every figure.

\begin{table}[p]
\centering
\caption{\textbf{Frontier table, GSM8K, 5-shot.} LLaDA-8B-Instruct, canvas 256, greedy, no KV cache; problems $[0,500)$ (HumanEval all 164). Accuracy with its Wilson 95\% interval; flexible and strict extraction. Every operating point behind \Cref{fig:llada-grid-full} and \Cref{fig:llada-grid-strict}, including cells beyond the figures' crops; PIG.EXT is the composite operating rule of every figure: its two slowest rows additionally floor admission at a confidence $\nu$ (written $\theta$ ($\nu{\ge}c$)), then the plain $\theta$ ladder from $0.5$.}
\label{tab:frontier-llada-gsm8k}
\scriptsize
\setlength{\tabcolsep}{3pt}
\fitwidth{\begin{tabular}{lrrrrrrrrrrrr}
\toprule
regime & \multicolumn{4}{c}{Fast-dLLM} & \multicolumn{4}{c}{PIG.INT} & \multicolumn{4}{c}{PIG.EXT} \\
\cmidrule(lr){2-5} \cmidrule(lr){6-9} \cmidrule(lr){10-13}
 & $\nu$ & TPF & acc (flex) & acc (strict) & $\varphi$ & TPF & acc (flex) & acc (strict) & $\theta$ & TPF & acc (flex) & acc (strict) \\
\midrule
BL=32 fixed & 0.95 & 2.83 & $0.786^{+0.034}_{-0.038}$ & $0.488^{+0.044}_{-0.044}$ & 0.015 & 3.73 & $0.796^{+0.033}_{-0.038}$ & $0.500^{+0.044}_{-0.044}$ & 0.125 ($\nu{\ge}0.9$) & 3.33 & $0.796^{+0.033}_{-0.038}$ & $0.500^{+0.044}_{-0.044}$ \\
 & 0.85 & 3.85 & $0.794^{+0.033}_{-0.038}$ & $0.510^{+0.044}_{-0.044}$ & 0.03 & 4.36 & $0.788^{+0.034}_{-0.038}$ & $0.528^{+0.043}_{-0.044}$ & 0.25 ($\nu{\ge}0.8$) & 4.13 & $0.786^{+0.034}_{-0.038}$ & $0.514^{+0.044}_{-0.044}$ \\
 & 0.7 & 5.36 & $0.772^{+0.035}_{-0.039}$ & $0.492^{+0.044}_{-0.044}$ & 0.05 & 4.81 & $0.798^{+0.033}_{-0.037}$ & $0.532^{+0.043}_{-0.044}$ & 0.5 & 5.48 & $0.708^{+0.038}_{-0.041}$ & $0.416^{+0.044}_{-0.042}$ \\
 & 0.5 & 7.50 & $0.666^{+0.040}_{-0.042}$ & $0.386^{+0.043}_{-0.042}$ & 0.1 & 5.44 & $0.778^{+0.034}_{-0.038}$ & $0.514^{+0.044}_{-0.044}$ & 1 & 6.82 & $0.632^{+0.041}_{-0.043}$ & $0.300^{+0.042}_{-0.039}$ \\
 & 0.3 & 9.49 & $0.372^{+0.043}_{-0.041}$ & $0.184^{+0.036}_{-0.032}$ & 0.2 & 6.59 & $0.764^{+0.035}_{-0.039}$ & $0.490^{+0.044}_{-0.044}$ & 1.5 & 8.49 & $0.534^{+0.043}_{-0.044}$ & $0.260^{+0.040}_{-0.037}$ \\
 & 0.2 & 13.31 & $0.156^{+0.034}_{-0.029}$ & $0.048^{+0.022}_{-0.016}$ & 0.3 & 7.42 & $0.734^{+0.037}_{-0.040}$ & $0.454^{+0.044}_{-0.043}$ & 2 & 10.64 & $0.426^{+0.044}_{-0.043}$ & $0.206^{+0.038}_{-0.033}$ \\
 &  &  &  &  & 0.5 & 8.60 & $0.712^{+0.038}_{-0.041}$ & $0.414^{+0.044}_{-0.042}$ & 2.5 & 13.37 & $0.322^{+0.042}_{-0.039}$ & $0.170^{+0.035}_{-0.030}$ \\
 &  &  &  &  & 1 & 9.55 & $0.598^{+0.042}_{-0.044}$ & $0.354^{+0.043}_{-0.041}$ & 3 & 16.95 & $0.194^{+0.037}_{-0.032}$ & $0.100^{+0.029}_{-0.023}$ \\
\midrule
BL=64 fixed & 0.95 & 2.88 & $0.800^{+0.033}_{-0.037}$ & $0.506^{+0.044}_{-0.044}$ & 0.015 & 3.86 & $0.798^{+0.033}_{-0.037}$ & $0.514^{+0.044}_{-0.044}$ & 0.125 ($\nu{\ge}0.9$) & 3.43 & $0.804^{+0.032}_{-0.037}$ & $0.510^{+0.044}_{-0.044}$ \\
 & 0.85 & 4.03 & $0.810^{+0.032}_{-0.037}$ & $0.526^{+0.043}_{-0.044}$ & 0.03 & 4.62 & $0.810^{+0.032}_{-0.037}$ & $0.538^{+0.043}_{-0.044}$ & 0.25 ($\nu{\ge}0.8$) & 4.33 & $0.790^{+0.033}_{-0.038}$ & $0.538^{+0.043}_{-0.044}$ \\
 & 0.7 & 5.83 & $0.768^{+0.035}_{-0.039}$ & $0.472^{+0.044}_{-0.043}$ & 0.05 & 5.18 & $0.806^{+0.032}_{-0.037}$ & $0.516^{+0.044}_{-0.044}$ & 0.5 & 5.53 & $0.720^{+0.038}_{-0.041}$ & $0.436^{+0.044}_{-0.043}$ \\
 & 0.5 & 8.49 & $0.656^{+0.040}_{-0.043}$ & $0.420^{+0.044}_{-0.042}$ & 0.1 & 5.96 & $0.782^{+0.034}_{-0.038}$ & $0.496^{+0.044}_{-0.044}$ & 1 & 7.16 & $0.666^{+0.040}_{-0.042}$ & $0.346^{+0.043}_{-0.040}$ \\
 & 0.3 & 11.41 & $0.382^{+0.043}_{-0.042}$ & $0.206^{+0.038}_{-0.033}$ & 0.2 & 7.38 & $0.766^{+0.035}_{-0.039}$ & $0.506^{+0.044}_{-0.044}$ & 1.5 & 9.15 & $0.590^{+0.042}_{-0.044}$ & $0.292^{+0.041}_{-0.038}$ \\
 & 0.2 & 17.43 & $0.170^{+0.035}_{-0.030}$ & $0.056^{+0.024}_{-0.017}$ & 0.3 & 8.29 & $0.742^{+0.036}_{-0.040}$ & $0.446^{+0.044}_{-0.043}$ & 2 & 11.91 & $0.494^{+0.044}_{-0.044}$ & $0.264^{+0.040}_{-0.037}$ \\
 &  &  &  &  & 0.5 & 10.04 & $0.704^{+0.038}_{-0.041}$ & $0.432^{+0.044}_{-0.043}$ & 2.5 & 15.81 & $0.372^{+0.043}_{-0.041}$ & $0.196^{+0.037}_{-0.032}$ \\
 &  &  &  &  & 1 & 11.65 & $0.608^{+0.042}_{-0.043}$ & $0.344^{+0.043}_{-0.040}$ &  &  &  &  \\
 &  &  &  &  & 2 & 13.47 & $0.502^{+0.044}_{-0.044}$ & $0.302^{+0.042}_{-0.039}$ &  &  &  &  \\
 &  &  &  &  & 3 & 15.08 & $0.332^{+0.042}_{-0.040}$ & $0.196^{+0.037}_{-0.032}$ &  &  &  &  \\
\midrule
BL=32 sliding & 0.95 & 2.85 & $0.812^{+0.032}_{-0.037}$ & $0.512^{+0.044}_{-0.044}$ & 0.015 & 3.95 & $0.814^{+0.032}_{-0.036}$ & $0.508^{+0.044}_{-0.044}$ & 0.125 ($\nu{\ge}0.9$) & 3.45 & $0.814^{+0.032}_{-0.036}$ & $0.502^{+0.044}_{-0.044}$ \\
 & 0.85 & 4.12 & $0.822^{+0.031}_{-0.036}$ & $0.518^{+0.043}_{-0.044}$ & 0.03 & 4.75 & $0.800^{+0.033}_{-0.037}$ & $0.526^{+0.043}_{-0.044}$ & 0.25 ($\nu{\ge}0.8$) & 4.44 & $0.818^{+0.031}_{-0.036}$ & $0.524^{+0.043}_{-0.044}$ \\
 & 0.7 & 6.12 & $0.770^{+0.035}_{-0.039}$ & $0.502^{+0.044}_{-0.044}$ & 0.05 & 5.37 & $0.806^{+0.032}_{-0.037}$ & $0.498^{+0.044}_{-0.044}$ & 0.5 & 5.44 & $0.750^{+0.036}_{-0.040}$ & $0.450^{+0.044}_{-0.043}$ \\
 & 0.5 & 8.84 & $0.652^{+0.040}_{-0.043}$ & $0.380^{+0.043}_{-0.041}$ & 0.1 & 6.20 & $0.778^{+0.034}_{-0.038}$ & $0.510^{+0.044}_{-0.044}$ & 1 & 6.86 & $0.658^{+0.040}_{-0.043}$ & $0.276^{+0.041}_{-0.037}$ \\
 & 0.3 & 11.54 & $0.380^{+0.043}_{-0.041}$ & $0.210^{+0.038}_{-0.033}$ & 0.2 & 7.66 & $0.748^{+0.036}_{-0.040}$ & $0.472^{+0.044}_{-0.043}$ & 1.5 & 8.56 & $0.578^{+0.043}_{-0.044}$ & $0.282^{+0.041}_{-0.038}$ \\
 & 0.2 & 16.77 & $0.130^{+0.032}_{-0.027}$ & $0.038^{+0.021}_{-0.014}$ & 0.3 & 8.69 & $0.738^{+0.037}_{-0.040}$ & $0.462^{+0.044}_{-0.043}$ & 2 & 10.76 & $0.440^{+0.044}_{-0.043}$ & $0.214^{+0.038}_{-0.034}$ \\
 &  &  &  &  & 0.5 & 10.67 & $0.696^{+0.039}_{-0.042}$ & $0.390^{+0.043}_{-0.042}$ & 2.5 & 13.72 & $0.318^{+0.042}_{-0.039}$ & $0.132^{+0.032}_{-0.027}$ \\
 &  &  &  &  & 1 & 11.96 & $0.628^{+0.041}_{-0.043}$ & $0.390^{+0.043}_{-0.042}$ & 3 & 16.81 & $0.202^{+0.037}_{-0.033}$ & $0.094^{+0.029}_{-0.023}$ \\
\midrule
BL=64 sliding & 0.95 & 2.81 & $0.788^{+0.034}_{-0.038}$ & $0.518^{+0.043}_{-0.044}$ & 0.015 & 3.93 & $0.798^{+0.033}_{-0.037}$ & $0.516^{+0.044}_{-0.044}$ & 0.125 ($\nu{\ge}0.9$) & 3.43 & $0.798^{+0.033}_{-0.037}$ & $0.518^{+0.043}_{-0.044}$ \\
 & 0.85 & 4.12 & $0.802^{+0.033}_{-0.037}$ & $0.522^{+0.043}_{-0.044}$ & 0.03 & 4.77 & $0.798^{+0.033}_{-0.037}$ & $0.530^{+0.043}_{-0.044}$ & 0.25 ($\nu{\ge}0.8$) & 4.46 & $0.788^{+0.034}_{-0.038}$ & $0.536^{+0.043}_{-0.044}$ \\
 & 0.7 & 6.27 & $0.764^{+0.035}_{-0.039}$ & $0.476^{+0.044}_{-0.043}$ & 0.05 & 5.43 & $0.786^{+0.034}_{-0.038}$ & $0.508^{+0.044}_{-0.044}$ & 0.5 & 5.64 & $0.704^{+0.038}_{-0.041}$ & $0.408^{+0.044}_{-0.042}$ \\
 & 0.5 & 9.58 & $0.612^{+0.042}_{-0.043}$ & $0.350^{+0.043}_{-0.041}$ & 0.1 & 6.39 & $0.766^{+0.035}_{-0.039}$ & $0.472^{+0.044}_{-0.043}$ & 1 & 7.42 & $0.658^{+0.040}_{-0.043}$ & $0.334^{+0.042}_{-0.040}$ \\
 & 0.3 & 13.48 & $0.326^{+0.042}_{-0.040}$ & $0.150^{+0.034}_{-0.029}$ & 0.2 & 8.10 & $0.744^{+0.036}_{-0.040}$ & $0.452^{+0.044}_{-0.043}$ & 1.5 & 9.85 & $0.588^{+0.042}_{-0.044}$ & $0.306^{+0.042}_{-0.039}$ \\
 &  &  &  &  & 0.3 & 9.23 & $0.722^{+0.037}_{-0.041}$ & $0.446^{+0.044}_{-0.043}$ & 2 & 12.97 & $0.460^{+0.044}_{-0.043}$ & $0.252^{+0.040}_{-0.036}$ \\
 &  &  &  &  & 0.5 & 11.75 & $0.696^{+0.039}_{-0.042}$ & $0.416^{+0.044}_{-0.042}$ & 2.5 & 17.26 & $0.334^{+0.042}_{-0.040}$ & $0.176^{+0.036}_{-0.031}$ \\
 &  &  &  &  & 1 & 14.09 & $0.600^{+0.042}_{-0.044}$ & $0.370^{+0.043}_{-0.041}$ &  &  &  &  \\
 &  &  &  &  & 2 & 16.37 & $0.518^{+0.043}_{-0.044}$ & $0.312^{+0.042}_{-0.039}$ &  &  &  &  \\
 &  &  &  &  & 3 & 18.47 & $0.334^{+0.042}_{-0.040}$ & $0.182^{+0.036}_{-0.031}$ &  &  &  &  \\
\midrule
free canvas & 0.95 & 2.29 & $0.518^{+0.043}_{-0.044}$ & $0.308^{+0.042}_{-0.039}$ & 0.015 & 3.19 & $0.516^{+0.044}_{-0.044}$ & $0.338^{+0.043}_{-0.040}$ & 0.125 ($\nu{\ge}0.9$) & 2.68 & $0.514^{+0.044}_{-0.044}$ & $0.316^{+0.042}_{-0.039}$ \\
 & 0.85 & 3.36 & $0.514^{+0.044}_{-0.044}$ & $0.338^{+0.043}_{-0.040}$ & 0.03 & 4.04 & $0.542^{+0.043}_{-0.044}$ & $0.374^{+0.043}_{-0.041}$ & 0.25 ($\nu{\ge}0.8$) & 3.65 & $0.522^{+0.043}_{-0.044}$ & $0.358^{+0.043}_{-0.041}$ \\
 & 0.7 & 5.69 & $0.512^{+0.044}_{-0.044}$ & $0.380^{+0.043}_{-0.041}$ & 0.05 & 4.73 & $0.568^{+0.043}_{-0.044}$ & $0.410^{+0.044}_{-0.042}$ & 0.5 & 5.06 & $0.506^{+0.044}_{-0.044}$ & $0.374^{+0.043}_{-0.041}$ \\
 & 0.5 & 10.72 & $0.456^{+0.044}_{-0.043}$ & $0.350^{+0.043}_{-0.041}$ & 0.1 & 6.07 & $0.552^{+0.043}_{-0.044}$ & $0.418^{+0.044}_{-0.042}$ & 1 & 7.05 & $0.532^{+0.043}_{-0.044}$ & $0.420^{+0.044}_{-0.042}$ \\
 & 0.3 & 20.16 & $0.260^{+0.040}_{-0.037}$ & $0.126^{+0.032}_{-0.026}$ & 0.2 & 8.78 & $0.572^{+0.043}_{-0.044}$ & $0.454^{+0.044}_{-0.043}$ & 1.5 & 9.67 & $0.524^{+0.043}_{-0.044}$ & $0.440^{+0.044}_{-0.043}$ \\
 &  &  &  &  & 0.3 & 10.18 & $0.540^{+0.043}_{-0.044}$ & $0.456^{+0.044}_{-0.043}$ & 2 & 12.44 & $0.496^{+0.044}_{-0.044}$ & $0.396^{+0.044}_{-0.042}$ \\
 &  &  &  &  & 0.5 & 13.19 & $0.530^{+0.043}_{-0.044}$ & $0.460^{+0.044}_{-0.043}$ & 2.5 & 15.27 & $0.440^{+0.044}_{-0.043}$ & $0.330^{+0.042}_{-0.040}$ \\
 &  &  &  &  & 1 & 16.16 & $0.478^{+0.044}_{-0.043}$ & $0.390^{+0.043}_{-0.042}$ & 3 & 19.68 & $0.378^{+0.043}_{-0.041}$ & $0.274^{+0.041}_{-0.037}$ \\
 &  &  &  &  & 2 & 19.85 & $0.436^{+0.044}_{-0.043}$ & $0.336^{+0.043}_{-0.040}$ &  &  &  &  \\
 &  &  &  &  & 3 & 21.79 & $0.378^{+0.043}_{-0.041}$ & $0.264^{+0.040}_{-0.037}$ &  &  &  &  \\
\bottomrule
\end{tabular}}
\end{table}

\begin{table}[p]
\centering
\caption{\textbf{Frontier table, MATH-500, 4-shot.} LLaDA-8B-Instruct, canvas 256, greedy, no KV cache; problems $[0,500)$ (HumanEval all 164). Accuracy with its Wilson 95\% interval; flexible and strict extraction. Every operating point behind \Cref{fig:llada-grid-full} and \Cref{fig:llada-grid-strict}, including cells beyond the figures' crops; PIG.EXT is the composite operating rule of every figure: its two slowest rows additionally floor admission at a confidence $\nu$ (written $\theta$ ($\nu{\ge}c$)), then the plain $\theta$ ladder from $0.5$.}
\label{tab:frontier-llada-math500}
\scriptsize
\setlength{\tabcolsep}{3pt}
\fitwidth{\begin{tabular}{lrrrrrrrrrrrr}
\toprule
regime & \multicolumn{4}{c}{Fast-dLLM} & \multicolumn{4}{c}{PIG.INT} & \multicolumn{4}{c}{PIG.EXT} \\
\cmidrule(lr){2-5} \cmidrule(lr){6-9} \cmidrule(lr){10-13}
 & $\nu$ & TPF & acc (flex) & acc (strict) & $\varphi$ & TPF & acc (flex) & acc (strict) & $\theta$ & TPF & acc (flex) & acc (strict) \\
\midrule
BL=32 fixed & 0.95 & 2.21 & $0.364^{+0.043}_{-0.041}$ & $0.356^{+0.043}_{-0.041}$ & 0.015 & 2.83 & $0.360^{+0.043}_{-0.041}$ & $0.354^{+0.043}_{-0.041}$ & 0.125 ($\nu{\ge}0.9$) & 2.54 & $0.362^{+0.043}_{-0.041}$ & $0.352^{+0.043}_{-0.041}$ \\
 & 0.85 & 2.88 & $0.364^{+0.043}_{-0.041}$ & $0.354^{+0.043}_{-0.041}$ & 0.03 & 3.23 & $0.356^{+0.043}_{-0.041}$ & $0.348^{+0.043}_{-0.040}$ & 0.25 ($\nu{\ge}0.8$) & 3.09 & $0.360^{+0.043}_{-0.041}$ & $0.356^{+0.043}_{-0.041}$ \\
 & 0.7 & 3.87 & $0.332^{+0.042}_{-0.040}$ & $0.330^{+0.042}_{-0.040}$ & 0.05 & 3.52 & $0.342^{+0.043}_{-0.040}$ & $0.334^{+0.042}_{-0.040}$ & 0.5 & 4.57 & $0.334^{+0.042}_{-0.040}$ & $0.322^{+0.042}_{-0.039}$ \\
 & 0.5 & 5.39 & $0.290^{+0.041}_{-0.038}$ & $0.268^{+0.040}_{-0.037}$ & 0.1 & 4.00 & $0.322^{+0.042}_{-0.039}$ & $0.318^{+0.042}_{-0.039}$ & 1 & 6.23 & $0.260^{+0.040}_{-0.037}$ & $0.250^{+0.040}_{-0.036}$ \\
 & 0.3 & 7.14 & $0.152^{+0.034}_{-0.029}$ & $0.112^{+0.031}_{-0.025}$ & 0.2 & 4.84 & $0.344^{+0.043}_{-0.040}$ & $0.342^{+0.043}_{-0.040}$ & 1.5 & 8.10 & $0.184^{+0.036}_{-0.032}$ & $0.172^{+0.036}_{-0.031}$ \\
 & 0.2 & 10.27 & $0.066^{+0.025}_{-0.019}$ & $0.030^{+0.019}_{-0.012}$ & 0.3 & 5.44 & $0.332^{+0.042}_{-0.040}$ & $0.318^{+0.042}_{-0.039}$ & 2 & 11.03 & $0.128^{+0.032}_{-0.026}$ & $0.110^{+0.030}_{-0.025}$ \\
 &  &  &  &  & 0.5 & 6.64 & $0.310^{+0.042}_{-0.039}$ & $0.302^{+0.042}_{-0.039}$ & 2.5 & 13.74 & $0.100^{+0.029}_{-0.023}$ & $0.082^{+0.027}_{-0.021}$ \\
 &  &  &  &  & 1 & 7.61 & $0.232^{+0.039}_{-0.035}$ & $0.216^{+0.038}_{-0.034}$ & 3 & 17.05 & $0.058^{+0.024}_{-0.017}$ & $0.040^{+0.021}_{-0.014}$ \\
\midrule
BL=64 fixed & 0.95 & 2.24 & $0.380^{+0.043}_{-0.041}$ & $0.374^{+0.043}_{-0.041}$ & 0.015 & 2.91 & $0.378^{+0.043}_{-0.041}$ & $0.374^{+0.043}_{-0.041}$ & 0.125 ($\nu{\ge}0.9$) & 2.59 & $0.386^{+0.043}_{-0.042}$ & $0.380^{+0.043}_{-0.041}$ \\
 & 0.85 & 2.98 & $0.380^{+0.043}_{-0.041}$ & $0.368^{+0.043}_{-0.041}$ & 0.03 & 3.36 & $0.378^{+0.043}_{-0.041}$ & $0.372^{+0.043}_{-0.041}$ & 0.25 ($\nu{\ge}0.8$) & 3.19 & $0.372^{+0.043}_{-0.041}$ & $0.368^{+0.043}_{-0.041}$ \\
 & 0.7 & 4.13 & $0.342^{+0.043}_{-0.040}$ & $0.340^{+0.043}_{-0.040}$ & 0.05 & 3.70 & $0.370^{+0.043}_{-0.041}$ & $0.368^{+0.043}_{-0.041}$ & 0.5 & 4.68 & $0.308^{+0.042}_{-0.039}$ & $0.306^{+0.042}_{-0.039}$ \\
 & 0.5 & 6.04 & $0.274^{+0.041}_{-0.037}$ & $0.260^{+0.040}_{-0.037}$ & 0.1 & 4.28 & $0.360^{+0.043}_{-0.041}$ & $0.358^{+0.043}_{-0.041}$ & 1 & 6.42 & $0.302^{+0.042}_{-0.039}$ & $0.298^{+0.042}_{-0.038}$ \\
 & 0.3 & 8.22 & $0.176^{+0.036}_{-0.031}$ & $0.122^{+0.032}_{-0.026}$ & 0.2 & 5.27 & $0.358^{+0.043}_{-0.041}$ & $0.354^{+0.043}_{-0.041}$ & 1.5 & 8.97 & $0.216^{+0.038}_{-0.034}$ & $0.210^{+0.038}_{-0.033}$ \\
 & 0.2 & 12.20 & $0.098^{+0.029}_{-0.023}$ & $0.042^{+0.021}_{-0.014}$ & 0.3 & 6.07 & $0.314^{+0.042}_{-0.039}$ & $0.308^{+0.042}_{-0.039}$ & 2 & 12.40 & $0.162^{+0.035}_{-0.030}$ & $0.148^{+0.034}_{-0.028}$ \\
 &  &  &  &  & 0.5 & 7.65 & $0.294^{+0.041}_{-0.038}$ & $0.292^{+0.041}_{-0.038}$ & 2.5 & 16.72 & $0.120^{+0.031}_{-0.026}$ & $0.110^{+0.030}_{-0.025}$ \\
 &  &  &  &  & 1 & 9.46 & $0.278^{+0.041}_{-0.037}$ & $0.258^{+0.040}_{-0.036}$ &  &  &  &  \\
 &  &  &  &  & 2 & 11.51 & $0.192^{+0.037}_{-0.032}$ & $0.176^{+0.036}_{-0.031}$ &  &  &  &  \\
 &  &  &  &  & 3 & 13.44 & $0.154^{+0.034}_{-0.029}$ & $0.106^{+0.030}_{-0.024}$ &  &  &  &  \\
\midrule
BL=32 sliding & 0.95 & 2.25 & $0.376^{+0.043}_{-0.041}$ & $0.368^{+0.043}_{-0.041}$ & 0.015 & 2.96 & $0.376^{+0.043}_{-0.041}$ & $0.370^{+0.043}_{-0.041}$ & 0.125 ($\nu{\ge}0.9$) & 2.63 & $0.380^{+0.043}_{-0.041}$ & $0.370^{+0.043}_{-0.041}$ \\
 & 0.85 & 3.03 & $0.382^{+0.043}_{-0.042}$ & $0.372^{+0.043}_{-0.041}$ & 0.03 & 3.44 & $0.350^{+0.043}_{-0.041}$ & $0.346^{+0.043}_{-0.040}$ & 0.25 ($\nu{\ge}0.8$) & 3.25 & $0.360^{+0.043}_{-0.041}$ & $0.354^{+0.043}_{-0.041}$ \\
 & 0.7 & 4.25 & $0.342^{+0.043}_{-0.040}$ & $0.338^{+0.043}_{-0.040}$ & 0.05 & 3.78 & $0.358^{+0.043}_{-0.041}$ & $0.354^{+0.043}_{-0.041}$ & 0.5 & 4.56 & $0.334^{+0.042}_{-0.040}$ & $0.324^{+0.042}_{-0.040}$ \\
 & 0.5 & 6.18 & $0.264^{+0.040}_{-0.037}$ & $0.248^{+0.040}_{-0.036}$ & 0.1 & 4.38 & $0.348^{+0.043}_{-0.040}$ & $0.338^{+0.043}_{-0.040}$ & 1 & 6.02 & $0.286^{+0.041}_{-0.038}$ & $0.280^{+0.041}_{-0.038}$ \\
 & 0.3 & 8.45 & $0.132^{+0.032}_{-0.027}$ & $0.106^{+0.030}_{-0.024}$ & 0.2 & 5.37 & $0.346^{+0.043}_{-0.040}$ & $0.340^{+0.043}_{-0.040}$ & 1.5 & 8.13 & $0.240^{+0.039}_{-0.035}$ & $0.232^{+0.039}_{-0.035}$ \\
 & 0.2 & 12.12 & $0.066^{+0.025}_{-0.019}$ & $0.034^{+0.020}_{-0.013}$ & 0.3 & 6.20 & $0.322^{+0.042}_{-0.039}$ & $0.304^{+0.042}_{-0.039}$ & 2 & 10.68 & $0.172^{+0.036}_{-0.031}$ & $0.152^{+0.034}_{-0.029}$ \\
 &  &  &  &  & 0.5 & 8.17 & $0.296^{+0.041}_{-0.038}$ & $0.290^{+0.041}_{-0.038}$ & 2.5 & 13.59 & $0.098^{+0.029}_{-0.023}$ & $0.080^{+0.027}_{-0.021}$ \\
 &  &  &  &  & 1 & 9.74 & $0.264^{+0.040}_{-0.037}$ & $0.246^{+0.040}_{-0.036}$ & 3 & 17.08 & $0.080^{+0.027}_{-0.021}$ & $0.058^{+0.024}_{-0.017}$ \\
\midrule
BL=64 sliding & 0.95 & 2.24 & $0.386^{+0.043}_{-0.042}$ & $0.378^{+0.043}_{-0.041}$ & 0.015 & 2.96 & $0.382^{+0.043}_{-0.042}$ & $0.378^{+0.043}_{-0.041}$ & 0.125 ($\nu{\ge}0.9$) & 2.61 & $0.396^{+0.044}_{-0.042}$ & $0.388^{+0.043}_{-0.042}$ \\
 & 0.85 & 3.05 & $0.374^{+0.043}_{-0.041}$ & $0.366^{+0.043}_{-0.041}$ & 0.03 & 3.45 & $0.356^{+0.043}_{-0.041}$ & $0.350^{+0.043}_{-0.041}$ & 0.25 ($\nu{\ge}0.8$) & 3.26 & $0.366^{+0.043}_{-0.041}$ & $0.362^{+0.043}_{-0.041}$ \\
 & 0.7 & 4.35 & $0.334^{+0.042}_{-0.040}$ & $0.332^{+0.042}_{-0.040}$ & 0.05 & 3.84 & $0.366^{+0.043}_{-0.041}$ & $0.364^{+0.043}_{-0.041}$ & 0.5 & 4.72 & $0.342^{+0.043}_{-0.040}$ & $0.334^{+0.042}_{-0.040}$ \\
 & 0.5 & 6.47 & $0.278^{+0.041}_{-0.037}$ & $0.270^{+0.041}_{-0.037}$ & 0.1 & 4.50 & $0.370^{+0.043}_{-0.041}$ & $0.366^{+0.043}_{-0.041}$ & 1 & 6.71 & $0.288^{+0.041}_{-0.038}$ & $0.282^{+0.041}_{-0.038}$ \\
 & 0.3 & 9.16 & $0.158^{+0.035}_{-0.029}$ & $0.118^{+0.031}_{-0.025}$ & 0.2 & 5.58 & $0.360^{+0.043}_{-0.041}$ & $0.354^{+0.043}_{-0.041}$ & 1.5 & 9.69 & $0.232^{+0.039}_{-0.035}$ & $0.222^{+0.038}_{-0.034}$ \\
 & 0.2 & 14.00 & $0.082^{+0.027}_{-0.021}$ & $0.034^{+0.020}_{-0.013}$ & 0.3 & 6.45 & $0.340^{+0.043}_{-0.040}$ & $0.332^{+0.042}_{-0.040}$ & 2 & 13.34 & $0.152^{+0.034}_{-0.029}$ & $0.140^{+0.033}_{-0.028}$ \\
 &  &  &  &  & 0.5 & 8.66 & $0.282^{+0.041}_{-0.038}$ & $0.274^{+0.041}_{-0.037}$ & 2.5 & 17.31 & $0.120^{+0.031}_{-0.026}$ & $0.108^{+0.030}_{-0.024}$ \\
 &  &  &  &  & 1 & 11.23 & $0.264^{+0.040}_{-0.037}$ & $0.248^{+0.040}_{-0.036}$ &  &  &  &  \\
 &  &  &  &  & 2 & 13.85 & $0.220^{+0.038}_{-0.034}$ & $0.202^{+0.037}_{-0.033}$ &  &  &  &  \\
 &  &  &  &  & 3 & 16.27 & $0.162^{+0.035}_{-0.030}$ & $0.132^{+0.032}_{-0.027}$ &  &  &  &  \\
\midrule
free canvas & 0.95 & 2.01 & $0.254^{+0.040}_{-0.036}$ & $0.252^{+0.040}_{-0.036}$ & 0.015 & 2.58 & $0.252^{+0.040}_{-0.036}$ & $0.250^{+0.040}_{-0.036}$ & 0.125 ($\nu{\ge}0.9$) & 2.31 & $0.252^{+0.040}_{-0.036}$ & $0.248^{+0.040}_{-0.036}$ \\
 & 0.85 & 2.64 & $0.250^{+0.040}_{-0.036}$ & $0.246^{+0.040}_{-0.036}$ & 0.03 & 3.02 & $0.250^{+0.040}_{-0.036}$ & $0.246^{+0.040}_{-0.036}$ & 0.25 ($\nu{\ge}0.8$) & 2.87 & $0.252^{+0.040}_{-0.036}$ & $0.248^{+0.040}_{-0.036}$ \\
 & 0.7 & 3.80 & $0.238^{+0.039}_{-0.035}$ & $0.234^{+0.039}_{-0.035}$ & 0.05 & 3.39 & $0.236^{+0.039}_{-0.035}$ & $0.232^{+0.039}_{-0.035}$ & 0.5 & 4.37 & $0.230^{+0.039}_{-0.035}$ & $0.226^{+0.039}_{-0.034}$ \\
 & 0.5 & 6.47 & $0.202^{+0.037}_{-0.033}$ & $0.190^{+0.037}_{-0.032}$ & 0.1 & 4.01 & $0.230^{+0.039}_{-0.035}$ & $0.226^{+0.039}_{-0.034}$ & 1 & 5.95 & $0.216^{+0.038}_{-0.034}$ & $0.212^{+0.038}_{-0.034}$ \\
 & 0.3 & 12.66 & $0.120^{+0.031}_{-0.026}$ & $0.094^{+0.029}_{-0.023}$ & 0.2 & 5.76 & $0.240^{+0.039}_{-0.035}$ & $0.234^{+0.039}_{-0.035}$ & 1.5 & 7.92 & $0.210^{+0.038}_{-0.033}$ & $0.208^{+0.038}_{-0.033}$ \\
 & 0.2 & 26.54 & $0.094^{+0.029}_{-0.023}$ & $0.044^{+0.022}_{-0.015}$ & 0.3 & 7.14 & $0.230^{+0.039}_{-0.035}$ & $0.226^{+0.039}_{-0.034}$ & 2 & 10.47 & $0.222^{+0.038}_{-0.034}$ & $0.222^{+0.038}_{-0.034}$ \\
 &  &  &  &  & 0.5 & 9.89 & $0.222^{+0.038}_{-0.034}$ & $0.220^{+0.038}_{-0.034}$ & 2.5 & 13.85 & $0.202^{+0.037}_{-0.033}$ & $0.200^{+0.037}_{-0.033}$ \\
 &  &  &  &  & 1 & 12.65 & $0.190^{+0.037}_{-0.032}$ & $0.180^{+0.036}_{-0.031}$ & 3 & 19.16 & $0.218^{+0.038}_{-0.034}$ & $0.210^{+0.038}_{-0.033}$ \\
 &  &  &  &  & 2 & 15.68 & $0.180^{+0.036}_{-0.031}$ & $0.164^{+0.035}_{-0.030}$ &  &  &  &  \\
 &  &  &  &  & 3 & 18.31 & $0.146^{+0.034}_{-0.028}$ & $0.130^{+0.032}_{-0.027}$ &  &  &  &  \\
\bottomrule
\end{tabular}}
\end{table}

\begin{table}[p]
\centering
\caption{\textbf{Frontier table, MBPP, 3-shot.} LLaDA-8B-Instruct, canvas 256, greedy, no KV cache; problems $[0,500)$ (HumanEval all 164). Accuracy with its Wilson 95\% interval; pass@1. Every operating point behind \Cref{fig:llada-grid-full} and \Cref{fig:llada-grid-strict}, including cells beyond the figures' crops; PIG.EXT is the composite operating rule of every figure: its two slowest rows additionally floor admission at a confidence $\nu$ (written $\theta$ ($\nu{\ge}c$)), then the plain $\theta$ ladder from $0.5$.}
\label{tab:frontier-llada-mbpp}
\scriptsize
\setlength{\tabcolsep}{3pt}
\fitwidth{\begin{tabular}{lrrrrrrrrr}
\toprule
regime & \multicolumn{3}{c}{Fast-dLLM} & \multicolumn{3}{c}{PIG.INT} & \multicolumn{3}{c}{PIG.EXT} \\
\cmidrule(lr){2-4} \cmidrule(lr){5-7} \cmidrule(lr){8-10}
 & $\nu$ & TPF & acc & $\varphi$ & TPF & acc & $\theta$ & TPF & acc \\
\midrule
BL=32 fixed & 0.95 & 3.29 & $0.396^{+0.044}_{-0.042}$ & 0.015 & 4.21 & $0.394^{+0.043}_{-0.042}$ & 0.125 ($\nu{\ge}0.9$) & 3.78 & $0.392^{+0.043}_{-0.042}$ \\
 & 0.85 & 4.26 & $0.388^{+0.043}_{-0.042}$ & 0.03 & 4.73 & $0.394^{+0.043}_{-0.042}$ & 0.25 ($\nu{\ge}0.8$) & 4.50 & $0.398^{+0.044}_{-0.042}$ \\
 & 0.7 & 5.53 & $0.372^{+0.043}_{-0.041}$ & 0.05 & 5.16 & $0.386^{+0.043}_{-0.042}$ & 0.5 & 7.15 & $0.260^{+0.040}_{-0.037}$ \\
 & 0.5 & 7.99 & $0.230^{+0.039}_{-0.035}$ & 0.1 & 5.73 & $0.366^{+0.043}_{-0.041}$ & 1 & 9.76 & $0.214^{+0.038}_{-0.034}$ \\
 & 0.3 & 12.10 & $0.068^{+0.026}_{-0.019}$ & 0.2 & 6.87 & $0.332^{+0.042}_{-0.040}$ & 1.5 & 13.50 & $0.126^{+0.032}_{-0.026}$ \\
 & 0.2 & 16.35 & $0.028^{+0.018}_{-0.011}$ & 0.3 & 8.05 & $0.308^{+0.042}_{-0.039}$ & 2 & 17.00 & $0.084^{+0.028}_{-0.021}$ \\
 &  &  &  & 0.5 & 9.98 & $0.236^{+0.039}_{-0.035}$ &  &  &  \\
 &  &  &  & 1 & 11.41 & $0.150^{+0.034}_{-0.029}$ &  &  &  \\
\midrule
BL=64 fixed & 0.99 & 2.59 & $0.396^{+0.044}_{-0.042}$ & 0.015 & 4.29 & $0.396^{+0.044}_{-0.042}$ & 0.125 ($\nu{\ge}0.9$) & 3.89 & $0.392^{+0.043}_{-0.042}$ \\
 & 0.95 & 3.38 & $0.394^{+0.043}_{-0.042}$ & 0.03 & 5.00 & $0.384^{+0.043}_{-0.042}$ & 0.25 ($\nu{\ge}0.8$) & 4.67 & $0.388^{+0.043}_{-0.042}$ \\
 & 0.85 & 4.42 & $0.382^{+0.043}_{-0.042}$ & 0.05 & 5.53 & $0.372^{+0.043}_{-0.041}$ & 0.5 & 7.78 & $0.250^{+0.040}_{-0.036}$ \\
 & 0.7 & 5.94 & $0.354^{+0.043}_{-0.041}$ & 0.1 & 6.19 & $0.356^{+0.043}_{-0.041}$ & 1 & 11.66 & $0.204^{+0.038}_{-0.033}$ \\
 & 0.5 & 9.54 & $0.222^{+0.038}_{-0.034}$ & 0.2 & 7.86 & $0.328^{+0.042}_{-0.040}$ & 1.5 & 18.52 & $0.144^{+0.033}_{-0.028}$ \\
 & 0.3 & 16.15 & $0.048^{+0.022}_{-0.016}$ & 0.3 & 9.72 & $0.300^{+0.042}_{-0.039}$ &  &  &  \\
 &  &  &  & 0.5 & 12.66 & $0.224^{+0.039}_{-0.034}$ &  &  &  \\
 &  &  &  & 1 & 14.96 & $0.132^{+0.032}_{-0.027}$ &  &  &  \\
\midrule
BL=32 sliding & 0.95 & 3.44 & $0.404^{+0.044}_{-0.042}$ & 0.015 & 4.45 & $0.406^{+0.044}_{-0.042}$ & 0.125 ($\nu{\ge}0.9$) & 3.95 & $0.404^{+0.044}_{-0.042}$ \\
 & 0.85 & 4.55 & $0.392^{+0.043}_{-0.042}$ & 0.03 & 5.13 & $0.396^{+0.044}_{-0.042}$ & 0.25 ($\nu{\ge}0.8$) & 4.78 & $0.402^{+0.044}_{-0.042}$ \\
 & 0.7 & 5.99 & $0.358^{+0.043}_{-0.041}$ & 0.05 & 5.61 & $0.382^{+0.043}_{-0.042}$ & 0.5 & 7.33 & $0.272^{+0.041}_{-0.037}$ \\
 & 0.5 & 9.41 & $0.212^{+0.038}_{-0.034}$ & 0.1 & 6.08 & $0.360^{+0.043}_{-0.041}$ & 1 & 9.97 & $0.236^{+0.039}_{-0.035}$ \\
 & 0.3 & 14.02 & $0.050^{+0.023}_{-0.016}$ & 0.2 & 7.90 & $0.310^{+0.042}_{-0.039}$ & 1.5 & 13.15 & $0.156^{+0.034}_{-0.029}$ \\
 & 0.2 & 19.57 & $0.020^{+0.016}_{-0.009}$ & 0.3 & 9.50 & $0.276^{+0.041}_{-0.037}$ & 2 & 16.94 & $0.104^{+0.030}_{-0.024}$ \\
 &  &  &  & 0.5 & 11.94 & $0.244^{+0.040}_{-0.036}$ &  &  &  \\
 &  &  &  & 1 & 14.00 & $0.138^{+0.033}_{-0.027}$ &  &  &  \\
\midrule
BL=64 sliding & 0.95 & 3.52 & $0.396^{+0.044}_{-0.042}$ & 0.015 & 4.58 & $0.402^{+0.044}_{-0.042}$ & 0.125 ($\nu{\ge}0.9$) & 4.03 & $0.396^{+0.044}_{-0.042}$ \\
 & 0.85 & 4.70 & $0.388^{+0.043}_{-0.042}$ & 0.03 & 5.30 & $0.394^{+0.043}_{-0.042}$ & 0.25 ($\nu{\ge}0.8$) & 4.94 & $0.392^{+0.043}_{-0.042}$ \\
 & 0.7 & 6.56 & $0.354^{+0.043}_{-0.041}$ & 0.05 & 6.06 & $0.370^{+0.043}_{-0.041}$ & 0.5 & 8.65 & $0.246^{+0.040}_{-0.036}$ \\
 & 0.5 & 11.17 & $0.208^{+0.038}_{-0.033}$ & 0.1 & 6.74 & $0.350^{+0.043}_{-0.041}$ & 1 & 14.23 & $0.216^{+0.038}_{-0.034}$ \\
 & 0.3 & 19.63 & $0.042^{+0.021}_{-0.014}$ & 0.2 & 8.93 & $0.316^{+0.042}_{-0.039}$ &  &  &  \\
 &  &  &  & 0.3 & 11.34 & $0.292^{+0.041}_{-0.038}$ &  &  &  \\
 &  &  &  & 0.5 & 14.58 & $0.202^{+0.037}_{-0.033}$ &  &  &  \\
 &  &  &  & 1 & 18.25 & $0.136^{+0.033}_{-0.027}$ &  &  &  \\
\midrule
free canvas & 0.99 & 2.08 & $0.354^{+0.043}_{-0.041}$ & 0.015 & 3.44 & $0.358^{+0.043}_{-0.041}$ & 0.125 ($\nu{\ge}0.9$) & 3.05 & $0.350^{+0.043}_{-0.041}$ \\
 & 0.95 & 2.66 & $0.354^{+0.043}_{-0.041}$ & 0.03 & 3.96 & $0.358^{+0.043}_{-0.041}$ & 0.25 ($\nu{\ge}0.8$) & 3.70 & $0.356^{+0.043}_{-0.041}$ \\
 & 0.85 & 3.54 & $0.356^{+0.043}_{-0.041}$ & 0.05 & 4.58 & $0.306^{+0.042}_{-0.039}$ & 0.5 & 5.57 & $0.256^{+0.040}_{-0.036}$ \\
 & 0.7 & 4.48 & $0.308^{+0.042}_{-0.039}$ & 0.1 & 4.72 & $0.296^{+0.041}_{-0.038}$ & 1 & 8.75 & $0.222^{+0.038}_{-0.034}$ \\
 & 0.5 & 6.19 & $0.204^{+0.038}_{-0.033}$ & 0.2 & 5.63 & $0.252^{+0.040}_{-0.036}$ & 1.5 & 14.43 & $0.144^{+0.033}_{-0.028}$ \\
 & 0.3 & 15.11 & $0.054^{+0.023}_{-0.017}$ & 0.3 & 6.62 & $0.242^{+0.039}_{-0.035}$ & 2 & 23.12 & $0.086^{+0.028}_{-0.022}$ \\
 &  &  &  & 0.5 & 9.27 & $0.220^{+0.038}_{-0.034}$ &  &  &  \\
 &  &  &  & 1 & 13.00 & $0.156^{+0.034}_{-0.029}$ &  &  &  \\
\bottomrule
\end{tabular}}
\end{table}

\begin{table}[p]
\centering
\caption{\textbf{Frontier table, HumanEval, 0-shot.} LLaDA-8B-Instruct, canvas 256, greedy, no KV cache; problems $[0,500)$ (HumanEval all 164). Accuracy with its Wilson 95\% interval; pass@1. Every operating point behind \Cref{fig:llada-grid-full} and \Cref{fig:llada-grid-strict}, including cells beyond the figures' crops; PIG.EXT is the composite operating rule of every figure: its two slowest rows additionally floor admission at a confidence $\nu$ (written $\theta$ ($\nu{\ge}c$)), then the plain $\theta$ ladder from $0.5$.}
\label{tab:frontier-llada-humaneval}
\scriptsize
\setlength{\tabcolsep}{3pt}
\fitwidth{\begin{tabular}{lrrrrrrrrr}
\toprule
regime & \multicolumn{3}{c}{Fast-dLLM} & \multicolumn{3}{c}{PIG.INT} & \multicolumn{3}{c}{PIG.EXT} \\
\cmidrule(lr){2-4} \cmidrule(lr){5-7} \cmidrule(lr){8-10}
 & $\nu$ & TPF & acc & $\varphi$ & TPF & acc & $\theta$ & TPF & acc \\
\midrule
BL=32 fixed & 0.95 & 4.95 & $0.396^{+0.076}_{-0.072}$ & 0.015 & 6.42 & $0.378^{+0.076}_{-0.071}$ & 0.125 ($\nu{\ge}0.9$) & 5.69 & $0.390^{+0.076}_{-0.071}$ \\
 & 0.85 & 6.76 & $0.396^{+0.076}_{-0.072}$ & 0.03 & 7.37 & $0.366^{+0.076}_{-0.070}$ & 0.25 ($\nu{\ge}0.8$) & 6.77 & $0.396^{+0.076}_{-0.072}$ \\
 & 0.7 & 8.69 & $0.342^{+0.075}_{-0.068}$ & 0.05 & 7.79 & $0.372^{+0.076}_{-0.070}$ & 0.5 & 8.50 & $0.274^{+0.073}_{-0.063}$ \\
 & 0.5 & 10.64 & $0.171^{+0.065}_{-0.050}$ & 0.1 & 8.67 & $0.335^{+0.075}_{-0.068}$ & 1 & 10.46 & $0.226^{+0.070}_{-0.057}$ \\
 & 0.3 & 12.92 & $0.018^{+0.034}_{-0.012}$ & 0.2 & 9.59 & $0.317^{+0.075}_{-0.066}$ & 1.5 & 13.03 & $0.207^{+0.068}_{-0.055}$ \\
 & 0.2 & 17.41 & $0.018^{+0.034}_{-0.012}$ & 0.3 & 10.64 & $0.281^{+0.073}_{-0.063}$ & 2 & 15.64 & $0.177^{+0.066}_{-0.051}$ \\
 &  &  &  & 0.5 & 11.60 & $0.238^{+0.071}_{-0.059}$ & 2.5 & 17.69 & $0.122^{+0.059}_{-0.042}$ \\
 &  &  &  & 1 & 13.10 & $0.104^{+0.056}_{-0.038}$ &  &  &  \\
\midrule
BL=64 fixed & 0.9 & 3.34 & $0.348^{+0.076}_{-0.069}$ & 0.015 & 6.68 & $0.372^{+0.076}_{-0.070}$ & 0.125 ($\nu{\ge}0.9$) & 6.00 & $0.384^{+0.076}_{-0.071}$ \\
 & 0.99 & 3.66 & $0.378^{+0.076}_{-0.071}$ & 0.03 & 8.14 & $0.348^{+0.076}_{-0.069}$ & 0.25 ($\nu{\ge}0.8$) & 7.54 & $0.384^{+0.076}_{-0.071}$ \\
 & 0.95 & 5.16 & $0.390^{+0.076}_{-0.071}$ & 0.05 & 8.88 & $0.366^{+0.076}_{-0.070}$ & 0.5 & 9.47 & $0.293^{+0.074}_{-0.064}$ \\
 & 0.85 & 7.21 & $0.360^{+0.076}_{-0.069}$ & 0.1 & 10.06 & $0.342^{+0.075}_{-0.068}$ & 1 & 12.32 & $0.262^{+0.072}_{-0.061}$ \\
 & 0.7 & 9.79 & $0.317^{+0.075}_{-0.066}$ & 0.2 & 11.27 & $0.274^{+0.073}_{-0.063}$ & 1.5 & 14.48 & $0.201^{+0.068}_{-0.054}$ \\
 & 0.5 & 12.73 & $0.140^{+0.061}_{-0.045}$ & 0.3 & 13.01 & $0.287^{+0.073}_{-0.064}$ & 2 & 18.69 & $0.165^{+0.064}_{-0.049}$ \\
 & 0.3 & 17.90 & $0.030^{+0.039}_{-0.017}$ & 0.5 & 14.70 & $0.213^{+0.069}_{-0.056}$ &  &  &  \\
 &  &  &  & 1 & 17.37 & $0.110^{+0.057}_{-0.039}$ &  &  &  \\
 &  &  &  & 2 & 19.93 & $0.049^{+0.045}_{-0.024}$ &  &  &  \\
\midrule
BL=32 sliding & 0.95 & 5.25 & $0.396^{+0.076}_{-0.072}$ & 0.015 & 6.82 & $0.415^{+0.077}_{-0.073}$ & 0.125 ($\nu{\ge}0.9$) & 6.00 & $0.402^{+0.076}_{-0.072}$ \\
 & 0.85 & 7.04 & $0.378^{+0.076}_{-0.071}$ & 0.03 & 8.04 & $0.366^{+0.076}_{-0.070}$ & 0.25 ($\nu{\ge}0.8$) & 7.45 & $0.408^{+0.076}_{-0.072}$ \\
 & 0.7 & 9.54 & $0.317^{+0.075}_{-0.066}$ & 0.05 & 8.68 & $0.402^{+0.076}_{-0.072}$ & 0.5 & 8.85 & $0.342^{+0.075}_{-0.068}$ \\
 & 0.5 & 12.21 & $0.128^{+0.060}_{-0.043}$ & 0.1 & 9.39 & $0.329^{+0.075}_{-0.067}$ & 1 & 10.72 & $0.262^{+0.072}_{-0.061}$ \\
 & 0.3 & 15.72 & $0.037^{+0.041}_{-0.020}$ & 0.2 & 10.96 & $0.299^{+0.074}_{-0.065}$ & 1.5 & 13.58 & $0.244^{+0.071}_{-0.059}$ \\
 &  &  &  & 0.3 & 11.83 & $0.238^{+0.071}_{-0.059}$ & 2 & 15.76 & $0.171^{+0.065}_{-0.050}$ \\
 &  &  &  & 0.5 & 13.99 & $0.171^{+0.065}_{-0.050}$ & 2.5 & 17.62 & $0.152^{+0.063}_{-0.047}$ \\
 &  &  &  & 1 & 15.44 & $0.091^{+0.054}_{-0.035}$ &  &  &  \\
\midrule
BL=64 sliding & 0.95 & 5.36 & $0.396^{+0.076}_{-0.072}$ & 0.015 & 7.25 & $0.384^{+0.076}_{-0.071}$ & 0.125 ($\nu{\ge}0.9$) & 6.13 & $0.390^{+0.076}_{-0.071}$ \\
 & 0.85 & 7.58 & $0.378^{+0.076}_{-0.071}$ & 0.03 & 8.83 & $0.378^{+0.076}_{-0.071}$ & 0.25 ($\nu{\ge}0.8$) & 8.11 & $0.384^{+0.076}_{-0.071}$ \\
 & 0.7 & 10.85 & $0.348^{+0.076}_{-0.069}$ & 0.05 & 9.50 & $0.366^{+0.076}_{-0.070}$ & 0.5 & 10.12 & $0.323^{+0.075}_{-0.067}$ \\
 & 0.5 & 15.37 & $0.134^{+0.061}_{-0.044}$ & 0.1 & 10.63 & $0.348^{+0.076}_{-0.069}$ & 1 & 13.57 & $0.293^{+0.074}_{-0.064}$ \\
 &  &  &  & 0.2 & 12.42 & $0.262^{+0.072}_{-0.061}$ & 1.5 & 16.98 & $0.268^{+0.073}_{-0.062}$ \\
 &  &  &  & 0.3 & 14.90 & $0.287^{+0.073}_{-0.064}$ &  &  &  \\
 &  &  &  & 0.5 & 16.86 & $0.207^{+0.068}_{-0.055}$ &  &  &  \\
\midrule
free canvas & 0.99 & 3.35 & $0.299^{+0.074}_{-0.065}$ & 0.015 & 7.27 & $0.317^{+0.075}_{-0.066}$ & 0.125 ($\nu{\ge}0.9$) & 5.82 & $0.287^{+0.073}_{-0.064}$ \\
 & 0.95 & 5.06 & $0.323^{+0.075}_{-0.067}$ & 0.03 & 8.68 & $0.293^{+0.074}_{-0.064}$ & 0.25 ($\nu{\ge}0.8$) & 7.65 & $0.287^{+0.073}_{-0.064}$ \\
 & 0.85 & 7.96 & $0.268^{+0.073}_{-0.062}$ & 0.05 & 10.10 & $0.274^{+0.073}_{-0.063}$ & 0.5 & 10.60 & $0.274^{+0.073}_{-0.063}$ \\
 & 0.7 & 11.71 & $0.262^{+0.072}_{-0.061}$ & 0.1 & 12.47 & $0.262^{+0.072}_{-0.061}$ & 1 & 14.27 & $0.299^{+0.074}_{-0.065}$ \\
 & 0.5 & 16.54 & $0.140^{+0.061}_{-0.045}$ & 0.2 & 15.18 & $0.250^{+0.071}_{-0.060}$ & 1.5 & 19.13 & $0.244^{+0.071}_{-0.059}$ \\
 & 0.3 & 27.35 & $0.024^{+0.037}_{-0.015}$ & 0.3 & 16.98 & $0.250^{+0.071}_{-0.060}$ & 2 & 23.53 & $0.159^{+0.064}_{-0.048}$ \\
 &  &  &  & 0.5 & 19.68 & $0.159^{+0.064}_{-0.048}$ &  &  &  \\
 &  &  &  & 1 & 22.48 & $0.128^{+0.060}_{-0.043}$ &  &  &  \\
 &  &  &  & 2 & 27.33 & $0.091^{+0.054}_{-0.035}$ &  &  &  \\
 &  &  &  & 3 & 29.71 & $0.055^{+0.046}_{-0.026}$ &  &  &  \\
\bottomrule
\end{tabular}}
\end{table}

\begin{table}[p]
\centering
\caption{\textbf{Frontier table, Nemotron-3B, GSM8K, 0-shot (flex).} Nemotron-Labs-Diffusion-3B, canvas 2048, 0-shot, greedy; the Fast-dLLM lane is the model's own confidence-threshold decoder at the same geometry; PIG.EXT is the composite operating rule (the two slowest points additionally floor admission at a confidence $\nu$, written $\theta$ ($\nu{\ge}c$)). Accuracy with its Wilson 95\% interval; the full three-geometry grid is \Cref{fig:nemotron-grid-full}.}
\label{tab:frontier-nemotron-gsm8k}
\scriptsize
\setlength{\tabcolsep}{3pt}
\fitwidth{\begin{tabular}{lrrrrrrrrr}
\toprule
regime & \multicolumn{3}{c}{Fast-dLLM} & \multicolumn{3}{c}{PIG.INT} & \multicolumn{3}{c}{PIG.EXT} \\
\cmidrule(lr){2-4} \cmidrule(lr){5-7} \cmidrule(lr){8-10}
 & $\nu$ & TPF & acc & $\varphi$ & TPF & acc & $\theta$ & TPF & acc \\
\midrule
BL=32 fixed & 0.9 & 2.56 & $0.896^{+0.024}_{-0.030}$ & 0.03 & 2.52 & $0.854^{+0.028}_{-0.034}$ & 0.125 & 2.05 & $0.870^{+0.027}_{-0.032}$ \\
 & 0.8 & 3.08 & $0.870^{+0.027}_{-0.032}$ & 0.05 & 2.91 & $0.868^{+0.027}_{-0.032}$ & 0.25 & 2.44 & $0.856^{+0.028}_{-0.033}$ \\
 & 0.7 & 3.66 & $0.856^{+0.028}_{-0.033}$ & 0.1 & 3.31 & $0.846^{+0.029}_{-0.034}$ & 0.5 & 3.66 & $0.806^{+0.032}_{-0.037}$ \\
 & 0.6 & 4.22 & $0.808^{+0.032}_{-0.037}$ & 0.2 & 3.87 & $0.844^{+0.029}_{-0.034}$ & 1 & 4.69 & $0.742^{+0.036}_{-0.040}$ \\
 & 0.5 & 4.91 & $0.738^{+0.037}_{-0.040}$ & 0.3 & 4.21 & $0.796^{+0.033}_{-0.038}$ & 1.5 & 5.70 & $0.668^{+0.040}_{-0.042}$ \\
 & 0.4 & 5.66 & $0.554^{+0.043}_{-0.044}$ & 0.5 & 4.71 & $0.754^{+0.036}_{-0.040}$ &  &  &  \\
 &  &  &  & 0.7 & 5.09 & $0.738^{+0.037}_{-0.040}$ &  &  &  \\
 &  &  &  & 1 & 5.44 & $0.650^{+0.041}_{-0.043}$ &  &  &  \\
\midrule
BL=32 sliding & 0.9 & 2.42 & $0.846^{+0.029}_{-0.034}$ & 0.03 & 2.62 & $0.840^{+0.030}_{-0.035}$ & 0.125 & 2.00 & $0.830^{+0.030}_{-0.035}$ \\
 & 0.8 & 2.97 & $0.850^{+0.029}_{-0.034}$ & 0.05 & 3.08 & $0.828^{+0.031}_{-0.036}$ & 0.25 & 2.41 & $0.842^{+0.029}_{-0.035}$ \\
 & 0.7 & 3.54 & $0.822^{+0.031}_{-0.036}$ & 0.1 & 3.54 & $0.840^{+0.030}_{-0.035}$ & 0.5 & 3.56 & $0.774^{+0.034}_{-0.039}$ \\
 & 0.6 & 4.11 & $0.758^{+0.035}_{-0.039}$ & 0.2 & 4.17 & $0.826^{+0.031}_{-0.036}$ & 1 & 4.45 & $0.786^{+0.034}_{-0.038}$ \\
 & 0.5 & 4.75 & $0.662^{+0.040}_{-0.043}$ & 0.3 & 4.61 & $0.782^{+0.034}_{-0.038}$ & 1.5 & 5.74 & $0.684^{+0.039}_{-0.042}$ \\
 & 0.4 & 5.69 & $0.526^{+0.043}_{-0.044}$ & 0.5 & 5.15 & $0.750^{+0.036}_{-0.040}$ &  &  &  \\
 &  &  &  & 0.7 & 5.52 & $0.704^{+0.038}_{-0.041}$ &  &  &  \\
 &  &  &  & 1 & 5.97 & $0.622^{+0.041}_{-0.043}$ &  &  &  \\
\midrule
BL=32 in 64-wide window & 0.9 & 2.53 & $0.852^{+0.028}_{-0.034}$ & 0.03 & 2.63 & $0.862^{+0.027}_{-0.033}$ & 0.125 & 1.99 & $0.862^{+0.027}_{-0.033}$ \\
 & 0.8 & 3.16 & $0.846^{+0.029}_{-0.034}$ & 0.05 & 3.07 & $0.850^{+0.029}_{-0.034}$ & 0.25 & 2.40 & $0.844^{+0.029}_{-0.034}$ \\
 & 0.7 & 3.85 & $0.824^{+0.031}_{-0.036}$ & 0.1 & 3.59 & $0.838^{+0.030}_{-0.035}$ & 0.5 & 3.55 & $0.836^{+0.030}_{-0.035}$ \\
 & 0.6 & 4.50 & $0.762^{+0.035}_{-0.039}$ & 0.2 & 4.18 & $0.822^{+0.031}_{-0.036}$ & 1 & 4.52 & $0.764^{+0.035}_{-0.039}$ \\
 & 0.5 & 5.12 & $0.652^{+0.040}_{-0.043}$ & 0.3 & 4.58 & $0.780^{+0.034}_{-0.038}$ & 1.5 & 5.81 & $0.690^{+0.039}_{-0.042}$ \\
 &  &  &  & 0.5 & 5.17 & $0.744^{+0.036}_{-0.040}$ &  &  &  \\
 &  &  &  & 0.7 & 5.57 & $0.718^{+0.038}_{-0.041}$ &  &  &  \\
 &  &  &  & 1 & 5.93 & $0.668^{+0.040}_{-0.042}$ &  &  &  \\
\bottomrule
\end{tabular}}
\end{table}

\begin{table}[p]
\centering
\caption{\textbf{Frontier table, Nemotron-3B, MATH-500, 0-shot.} Nemotron-Labs-Diffusion-3B, canvas 2048, 0-shot, greedy; the Fast-dLLM lane is the model's own confidence-threshold decoder at the same geometry; PIG.EXT is the composite operating rule (the two slowest points additionally floor admission at a confidence $\nu$, written $\theta$ ($\nu{\ge}c$)). Accuracy with its Wilson 95\% interval; the full three-geometry grid is \Cref{fig:nemotron-grid-full}.}
\label{tab:frontier-nemotron-math500}
\scriptsize
\setlength{\tabcolsep}{3pt}
\fitwidth{\begin{tabular}{lrrrrrrrrr}
\toprule
regime & \multicolumn{3}{c}{Fast-dLLM} & \multicolumn{3}{c}{PIG.INT} & \multicolumn{3}{c}{PIG.EXT} \\
\cmidrule(lr){2-4} \cmidrule(lr){5-7} \cmidrule(lr){8-10}
 & $\nu$ & TPF & acc & $\varphi$ & TPF & acc & $\theta$ & TPF & acc \\
\midrule
BL=32 fixed & 0.9 & 2.80 & $0.606^{+0.042}_{-0.043}$ & 0.03 & 2.72 & $0.708^{+0.038}_{-0.041}$ & 0.125 & 2.24 & $0.744^{+0.036}_{-0.040}$ \\
 & 0.8 & 3.26 & $0.598^{+0.042}_{-0.044}$ & 0.05 & 3.09 & $0.686^{+0.039}_{-0.042}$ & 0.25 & 2.63 & $0.708^{+0.038}_{-0.041}$ \\
 & 0.7 & 3.78 & $0.552^{+0.043}_{-0.044}$ & 0.1 & 3.45 & $0.692^{+0.039}_{-0.042}$ & 0.5 & 3.96 & $0.588^{+0.042}_{-0.044}$ \\
 & 0.6 & 4.28 & $0.498^{+0.044}_{-0.044}$ & 0.2 & 3.98 & $0.614^{+0.042}_{-0.043}$ & 1 & 5.23 & $0.492^{+0.044}_{-0.044}$ \\
 & 0.5 & 4.95 & $0.414^{+0.044}_{-0.042}$ & 0.3 & 4.37 & $0.556^{+0.043}_{-0.044}$ &  &  &  \\
 & 0.4 & 5.61 & $0.308^{+0.042}_{-0.039}$ & 0.5 & 4.87 & $0.480^{+0.044}_{-0.043}$ &  &  &  \\
 &  &  &  & 0.7 & 5.30 & $0.446^{+0.044}_{-0.043}$ &  &  &  \\
 &  &  &  & 1 & 5.69 & $0.418^{+0.044}_{-0.042}$ &  &  &  \\
\midrule
BL=32 sliding & 0.9 & 2.70 & $0.688^{+0.039}_{-0.042}$ & 0.03 & 2.86 & $0.628^{+0.041}_{-0.043}$ & 0.125 & 2.22 & $0.664^{+0.040}_{-0.043}$ \\
 & 0.8 & 3.20 & $0.628^{+0.041}_{-0.043}$ & 0.05 & 3.28 & $0.658^{+0.040}_{-0.043}$ & 0.25 & 2.63 & $0.680^{+0.039}_{-0.042}$ \\
 & 0.7 & 3.71 & $0.612^{+0.042}_{-0.043}$ & 0.1 & 3.76 & $0.612^{+0.042}_{-0.043}$ & 0.5 & 3.87 & $0.582^{+0.042}_{-0.044}$ \\
 & 0.6 & 4.22 & $0.476^{+0.044}_{-0.043}$ & 0.2 & 4.34 & $0.552^{+0.043}_{-0.044}$ & 1 & 5.00 & $0.494^{+0.044}_{-0.044}$ \\
 & 0.5 & 4.91 & $0.392^{+0.043}_{-0.042}$ & 0.3 & 4.80 & $0.530^{+0.043}_{-0.044}$ &  &  &  \\
 & 0.4 & 5.66 & $0.260^{+0.040}_{-0.037}$ & 0.5 & 5.28 & $0.452^{+0.044}_{-0.043}$ &  &  &  \\
 &  &  &  & 0.7 & 5.74 & $0.394^{+0.043}_{-0.042}$ &  &  &  \\
\midrule
BL=32 in 64-wide window & 0.9 & 2.84 & $0.628^{+0.041}_{-0.043}$ & 0.03 & 2.87 & $0.644^{+0.041}_{-0.043}$ & 0.125 & 2.19 & $0.662^{+0.040}_{-0.043}$ \\
 & 0.8 & 3.41 & $0.588^{+0.042}_{-0.044}$ & 0.05 & 3.29 & $0.628^{+0.041}_{-0.043}$ & 0.25 & 2.61 & $0.666^{+0.040}_{-0.042}$ \\
 & 0.7 & 4.03 & $0.544^{+0.043}_{-0.044}$ & 0.1 & 3.74 & $0.594^{+0.042}_{-0.044}$ & 0.5 & 3.86 & $0.574^{+0.043}_{-0.044}$ \\
 & 0.6 & 4.60 & $0.474^{+0.044}_{-0.043}$ & 0.2 & 4.36 & $0.572^{+0.043}_{-0.044}$ & 1 & 5.32 & $0.462^{+0.044}_{-0.043}$ \\
 & 0.5 & 5.25 & $0.370^{+0.043}_{-0.041}$ & 0.3 & 4.78 & $0.524^{+0.043}_{-0.044}$ &  &  &  \\
 &  &  &  & 0.5 & 5.30 & $0.480^{+0.044}_{-0.043}$ &  &  &  \\
 &  &  &  & 0.7 & 5.73 & $0.404^{+0.044}_{-0.042}$ &  &  &  \\
\bottomrule
\end{tabular}}
\end{table}

\begin{table}[p]
\centering
\caption{\textbf{Frontier table, Nemotron-3B, MBPP, 0-shot.} Nemotron-Labs-Diffusion-3B, canvas 2048, 0-shot, greedy; the Fast-dLLM lane is the model's own confidence-threshold decoder at the same geometry; PIG.EXT is the composite operating rule (the two slowest points additionally floor admission at a confidence $\nu$, written $\theta$ ($\nu{\ge}c$)). Accuracy with its Wilson 95\% interval; the full three-geometry grid is \Cref{fig:nemotron-grid-full}.}
\label{tab:frontier-nemotron-mbpp}
\scriptsize
\setlength{\tabcolsep}{3pt}
\fitwidth{\begin{tabular}{lrrrrrrrrr}
\toprule
regime & \multicolumn{3}{c}{Fast-dLLM} & \multicolumn{3}{c}{PIG.INT} & \multicolumn{3}{c}{PIG.EXT} \\
\cmidrule(lr){2-4} \cmidrule(lr){5-7} \cmidrule(lr){8-10}
 & $\nu$ & TPF & acc & $\varphi$ & TPF & acc & $\theta$ & TPF & acc \\
\midrule
BL=32 fixed & 0.9 & 1.91 & $0.553^{+0.043}_{-0.044}$ & 0.03 & 1.92 & $0.548^{+0.043}_{-0.044}$ & 0.125 & 1.64 & $0.574^{+0.043}_{-0.044}$ \\
 & 0.8 & 2.26 & $0.485^{+0.044}_{-0.044}$ & 0.05 & 2.18 & $0.526^{+0.043}_{-0.044}$ & 0.25 & 1.89 & $0.562^{+0.043}_{-0.044}$ \\
 & 0.7 & 2.65 & $0.433^{+0.044}_{-0.043}$ & 0.1 & 2.47 & $0.504^{+0.044}_{-0.044}$ & 0.5 & 3.06 & $0.412^{+0.044}_{-0.042}$ \\
 & 0.6 & 3.11 & $0.343^{+0.043}_{-0.040}$ & 0.2 & 2.87 & $0.448^{+0.044}_{-0.043}$ & 1 & 5.20 & $0.286^{+0.041}_{-0.038}$ \\
 & 0.5 & 3.62 & $0.222^{+0.039}_{-0.034}$ & 0.3 & 3.16 & $0.408^{+0.044}_{-0.042}$ &  &  &  \\
 & 0.4 & 4.23 & $0.136^{+0.033}_{-0.027}$ & 0.5 & 3.60 & $0.316^{+0.042}_{-0.039}$ &  &  &  \\
 & 0.3 & 5.30 & $0.074^{+0.026}_{-0.020}$ & 0.7 & 3.97 & $0.292^{+0.041}_{-0.038}$ &  &  &  \\
 &  &  &  & 1 & 4.36 & $0.204^{+0.038}_{-0.033}$ &  &  &  \\
\midrule
BL=32 sliding & 0.9 & 1.81 & $0.565^{+0.043}_{-0.044}$ & 0.03 & 1.95 & $0.560^{+0.043}_{-0.044}$ & 0.125 & 1.60 & $0.596^{+0.042}_{-0.044}$ \\
 & 0.8 & 2.17 & $0.537^{+0.043}_{-0.044}$ & 0.05 & 2.24 & $0.542^{+0.043}_{-0.044}$ & 0.25 & 1.87 & $0.552^{+0.043}_{-0.044}$ \\
 & 0.7 & 2.53 & $0.403^{+0.044}_{-0.042}$ & 0.1 & 2.53 & $0.474^{+0.044}_{-0.043}$ & 0.5 & 2.99 & $0.432^{+0.044}_{-0.043}$ \\
 & 0.6 & 2.97 & $0.335^{+0.043}_{-0.040}$ & 0.2 & 2.98 & $0.398^{+0.044}_{-0.042}$ & 1 & 5.27 & $0.348^{+0.043}_{-0.040}$ \\
 & 0.5 & 3.50 & $0.222^{+0.039}_{-0.034}$ & 0.3 & 3.30 & $0.314^{+0.042}_{-0.039}$ &  &  &  \\
 & 0.4 & 4.08 & $0.122^{+0.032}_{-0.026}$ & 0.5 & 3.71 & $0.308^{+0.042}_{-0.039}$ &  &  &  \\
 & 0.3 & 5.10 & $0.050^{+0.023}_{-0.016}$ & 0.7 & 4.07 & $0.226^{+0.039}_{-0.034}$ &  &  &  \\
 &  &  &  & 1 & 4.49 & $0.188^{+0.037}_{-0.032}$ &  &  &  \\
\midrule
BL=32 in 64-wide window & 0.9 & 1.84 & $0.539^{+0.043}_{-0.044}$ & 0.03 & 1.94 & $0.540^{+0.043}_{-0.044}$ & 0.125 & 1.60 & $0.572^{+0.043}_{-0.044}$ \\
 & 0.8 & 2.25 & $0.503^{+0.044}_{-0.044}$ & 0.05 & 2.24 & $0.530^{+0.043}_{-0.044}$ & 0.25 & 1.87 & $0.560^{+0.043}_{-0.044}$ \\
 & 0.7 & 2.65 & $0.407^{+0.044}_{-0.042}$ & 0.1 & 2.54 & $0.482^{+0.044}_{-0.043}$ & 0.5 & 3.16 & $0.464^{+0.044}_{-0.043}$ \\
 & 0.6 & 3.09 & $0.297^{+0.042}_{-0.038}$ & 0.2 & 2.97 & $0.392^{+0.043}_{-0.042}$ & 1 & 5.58 & $0.342^{+0.043}_{-0.040}$ \\
 & 0.5 & 3.66 & $0.170^{+0.035}_{-0.030}$ & 0.3 & 3.31 & $0.346^{+0.043}_{-0.040}$ &  &  &  \\
 & 0.4 & 4.29 & $0.090^{+0.028}_{-0.022}$ & 0.5 & 3.73 & $0.278^{+0.041}_{-0.037}$ &  &  &  \\
 & 0.3 & 5.58 & $0.050^{+0.023}_{-0.016}$ & 0.7 & 4.07 & $0.242^{+0.039}_{-0.035}$ &  &  &  \\
 &  &  &  & 1 & 4.54 & $0.196^{+0.037}_{-0.032}$ &  &  &  \\
\bottomrule
\end{tabular}}
\end{table}

\begin{table}[p]
\centering
\caption{\textbf{Frontier table, Nemotron-3B, HumanEval, 0-shot.} Nemotron-Labs-Diffusion-3B, canvas 2048, 0-shot, greedy; the Fast-dLLM lane is the model's own confidence-threshold decoder at the same geometry; PIG.EXT is the composite operating rule (the two slowest points additionally floor admission at a confidence $\nu$, written $\theta$ ($\nu{\ge}c$)). Accuracy with its Wilson 95\% interval; the full three-geometry grid is \Cref{fig:nemotron-grid-full}.}
\label{tab:frontier-nemotron-humaneval}
\scriptsize
\setlength{\tabcolsep}{3pt}
\fitwidth{\begin{tabular}{lrrrrrrrrr}
\toprule
regime & \multicolumn{3}{c}{Fast-dLLM} & \multicolumn{3}{c}{PIG.INT} & \multicolumn{3}{c}{PIG.EXT} \\
\cmidrule(lr){2-4} \cmidrule(lr){5-7} \cmidrule(lr){8-10}
 & $\nu$ & TPF & acc & $\varphi$ & TPF & acc & $\theta$ & TPF & acc \\
\midrule
BL=32 fixed & 0.9 & 2.02 & $0.750^{+0.060}_{-0.071}$ & 0.03 & 2.03 & $0.756^{+0.059}_{-0.071}$ & 0.125 & 1.69 & $0.793^{+0.055}_{-0.068}$ \\
 & 0.8 & 2.41 & $0.707^{+0.064}_{-0.074}$ & 0.05 & 2.30 & $0.726^{+0.063}_{-0.073}$ & 0.25 & 2.00 & $0.744^{+0.061}_{-0.072}$ \\
 & 0.7 & 2.84 & $0.543^{+0.074}_{-0.076}$ & 0.1 & 2.62 & $0.707^{+0.064}_{-0.074}$ & 0.5 & 3.09 & $0.512^{+0.075}_{-0.076}$ \\
 & 0.6 & 3.34 & $0.433^{+0.077}_{-0.073}$ & 0.2 & 3.06 & $0.555^{+0.074}_{-0.076}$ & 1 & 4.59 & $0.396^{+0.076}_{-0.072}$ \\
 & 0.5 & 3.83 & $0.329^{+0.075}_{-0.067}$ & 0.3 & 3.36 & $0.427^{+0.077}_{-0.073}$ &  &  &  \\
 & 0.4 & 4.68 & $0.146^{+0.062}_{-0.046}$ & 0.5 & 3.80 & $0.348^{+0.076}_{-0.069}$ &  &  &  \\
 & 0.3 & 5.55 & $0.067^{+0.049}_{-0.029}$ & 0.7 & 4.12 & $0.250^{+0.071}_{-0.060}$ &  &  &  \\
 &  &  &  & 1 & 4.55 & $0.177^{+0.066}_{-0.051}$ &  &  &  \\
\midrule
BL=32 sliding & 0.9 & 1.94 & $0.713^{+0.064}_{-0.073}$ & 0.03 & 2.05 & $0.726^{+0.063}_{-0.073}$ & 0.125 & 1.66 & $0.744^{+0.061}_{-0.072}$ \\
 & 0.8 & 2.32 & $0.689^{+0.066}_{-0.074}$ & 0.05 & 2.36 & $0.701^{+0.065}_{-0.074}$ & 0.25 & 1.96 & $0.726^{+0.063}_{-0.073}$ \\
 & 0.7 & 2.75 & $0.567^{+0.073}_{-0.077}$ & 0.1 & 2.69 & $0.640^{+0.069}_{-0.076}$ & 0.5 & 2.99 & $0.573^{+0.073}_{-0.077}$ \\
 & 0.6 & 3.22 & $0.415^{+0.077}_{-0.073}$ & 0.2 & 3.16 & $0.518^{+0.075}_{-0.076}$ & 1 & 4.24 & $0.469^{+0.076}_{-0.075}$ \\
 & 0.5 & 3.75 & $0.238^{+0.071}_{-0.059}$ & 0.3 & 3.52 & $0.482^{+0.076}_{-0.075}$ &  &  &  \\
 & 0.4 & 4.37 & $0.104^{+0.056}_{-0.038}$ & 0.5 & 3.91 & $0.348^{+0.076}_{-0.069}$ &  &  &  \\
 & 0.3 & 5.44 & $0.037^{+0.041}_{-0.020}$ & 0.7 & 4.32 & $0.305^{+0.074}_{-0.065}$ &  &  &  \\
 &  &  &  & 1 & 4.66 & $0.213^{+0.069}_{-0.056}$ &  &  &  \\
\midrule
BL=32 in 64-wide window & 0.9 & 2.03 & $0.701^{+0.065}_{-0.074}$ & 0.03 & 2.06 & $0.689^{+0.066}_{-0.074}$ & 0.125 & 1.66 & $0.720^{+0.063}_{-0.073}$ \\
 & 0.8 & 2.46 & $0.652^{+0.069}_{-0.076}$ & 0.05 & 2.36 & $0.689^{+0.066}_{-0.074}$ & 0.25 & 1.94 & $0.701^{+0.065}_{-0.074}$ \\
 & 0.7 & 2.94 & $0.543^{+0.074}_{-0.076}$ & 0.1 & 2.68 & $0.604^{+0.072}_{-0.076}$ & 0.5 & 3.08 & $0.604^{+0.072}_{-0.076}$ \\
 & 0.6 & 3.53 & $0.348^{+0.076}_{-0.069}$ & 0.2 & 3.18 & $0.604^{+0.072}_{-0.076}$ & 1 & 4.36 & $0.494^{+0.076}_{-0.075}$ \\
 & 0.5 & 4.23 & $0.220^{+0.069}_{-0.057}$ & 0.3 & 3.52 & $0.451^{+0.076}_{-0.074}$ &  &  &  \\
 & 0.4 & 5.05 & $0.116^{+0.058}_{-0.040}$ & 0.5 & 3.95 & $0.329^{+0.075}_{-0.067}$ &  &  &  \\
 &  &  &  & 0.7 & 4.28 & $0.311^{+0.074}_{-0.066}$ &  &  &  \\
 &  &  &  & 1 & 4.68 & $0.238^{+0.071}_{-0.059}$ &  &  &  \\
\bottomrule
\end{tabular}}
\end{table}

\clearpage

\end{document}